\documentclass{article}

\usepackage[final]{corl_2023} 
\usepackage{multicol}
\usepackage{graphicx}
\usepackage{subcaption}
\usepackage{amsthm}
\newtheorem{definition}{Definition}

\newtheorem{theorem}{Theorem}
\newtheorem{lemma}{Lemma}
\usepackage{mathtools}
\usepackage{booktabs}
\usepackage{multirow}
\usepackage{pythonhighlight}
\usepackage{wrapfig}
\usepackage{soul}

\usepackage[capitalise]{cleveref}
\usepackage[all]{abaisero}

\newcommand{\edit}[1]{{\color{black}#1}}

\title{Equivariant Reinforcement Learning under Partial Observability}

\author{
  Hai Nguyen, Andrea Baisero, David Klee, Dian Wang, Robert Platt, Christopher Amato \\
  Khoury College of Computer Sciences, 
  Northeastern University, Boston, MA, United States\\
  \texttt{nguyen.hai1@northeastern.edu}\\
  \url{https://sites.google.com/view/equi-rl-pomdp}
}

\begin{document}
\maketitle


\begin{abstract}
Incorporating inductive biases is a promising approach for tackling challenging robot learning domains with sample-efficient solutions. This paper identifies partially observable domains where symmetries can be a useful inductive bias for efficient learning. Specifically, by encoding the equivariance regarding specific group symmetries into the neural networks, our actor-critic reinforcement learning agents can reuse solutions in the past for related scenarios. Consequently, our equivariant agents outperform non-equivariant approaches significantly in terms of sample efficiency and final performance, demonstrated through experiments on a range of robotic tasks in simulation and real hardware. 
\end{abstract}

\keywords{Partial Observability, Equivariant Learning, Symmetry}


\section{Introduction}

\begin{wrapfigure}[18]{R}{0.45\linewidth}
  \centering
  \includegraphics[width=0.9\linewidth]{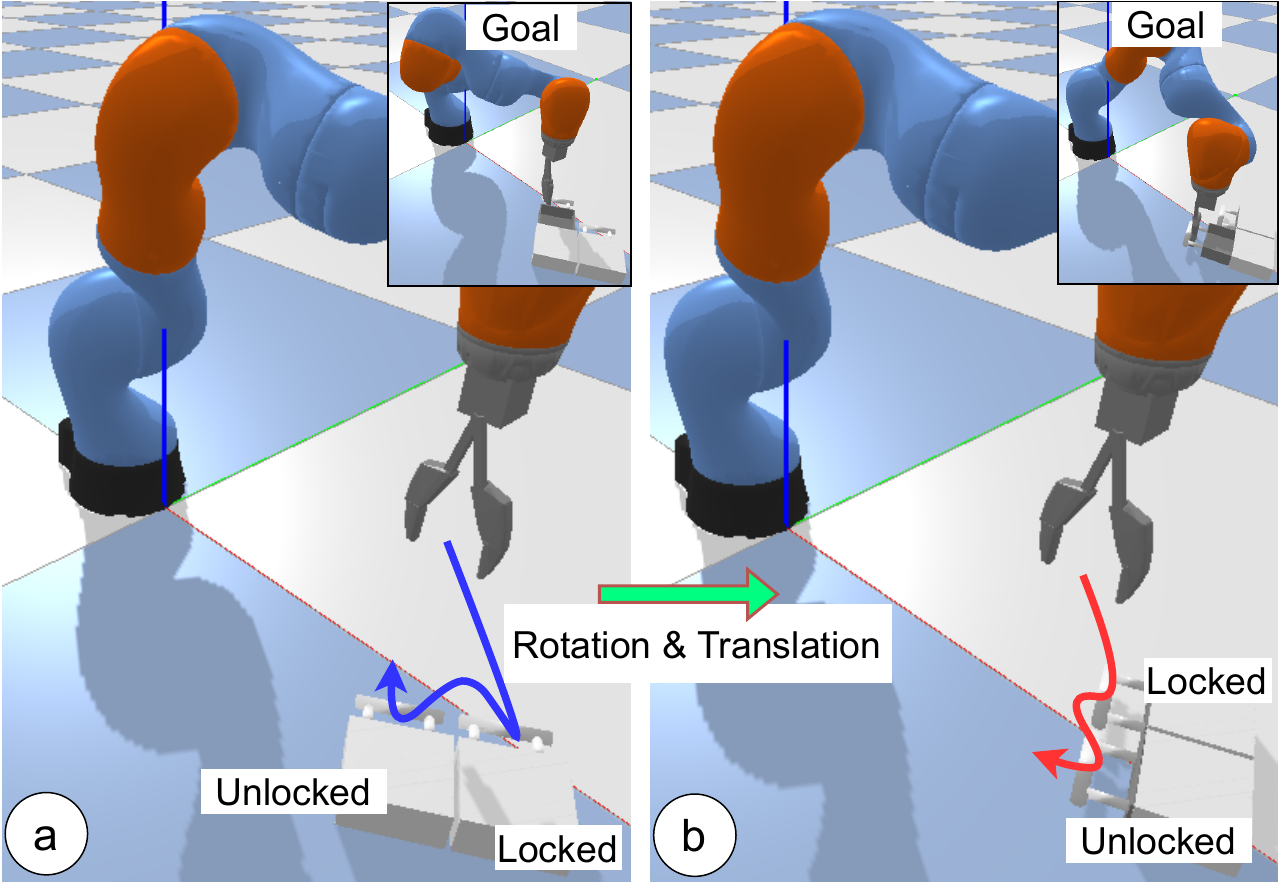}
  \caption{\texttt{Drawer-Opening}: This POMDP is rotationally symmetric in the sense that an optimal solution to the problem on the left (in {\color{blue}{blue}}) can be rotated to obtain an optimal solution to a rotated version of the problem on the right (in {\color{red}{red}}). }
  \label{fig:motivating_example}
\end{wrapfigure}

A key challenge in robot learning is to improve sample efficiency, i.e., to reduce the number of experiences or demonstrations needed to learn a good policy. One way to do this is to identify domain symmetries that can structure the policy space. Recent works have demonstrated that symmetry-preserving (equivariant) neural network models are a particularly effective way of accomplishing this~\citep{wang2022so, wang2022equivariant, wang2022robot, huang2022equivariant, zhu2022sample}. However, these works have focused primarily on fully observable Markov decision processes (MDPs) rather than partially observable systems encoded as partially observable MDPs (POMDPs)~\citep{astrom1965optimal}. The question arises whether symmetric neural models can also be used to solve Partially Observable Reinforcement Learning (PORL) problems. This paper identifies the theoretical conditions under which this is indeed the case and describes an equivariant recurrent model that works well in practice.

To motivate, \cref{fig:motivating_example} illustrates the \texttt{Drawer-Opening} problem where a robot is presented with a chest containing two drawers, one locked and one unlocked. To solve this task, the robot must determine which drawer is unlocked and then open that drawer, relying only on top-down image observations. This task reflects a common POMDP when physical properties (whether a drawer is unlocked) are hidden from the visual input. The only way for the robot to distinguish between the two drawers is to attempt to open one of them. This is a classic feature of a POMDP -- that the agent must perform \emph{information gathering actions} to obtain information needed to solve the task. Notice that this problem is rotationally symmetric in the sense that its optimal solution (the {\color{blue}{blue}} end-effector trajectory in \cref{fig:motivating_example}a) rotates (the {\color{red}{red}} trajectory in \cref{fig:motivating_example}b) when the scene itself rotates and is an example of the type of symmetry that we want our agents to embed in their architectures.

We make three contributions in this work. First, we extend the framework of group-invariant Markov decision processes~\citep{wang2022so} to the partially observable setting, resulting in a new theory and solution method. Specifically, we prove the optimal policy and the value function must be equivariant and invariant in this new setting. Second, backed by the proof, we introduce equivariant actor-critic agents that inherently embed the domain symmetry in their architectures. Finally, we apply the agents in realistic robot manipulation tasks with sparse rewards, where our agents are shown to significantly outperform non-equivariant approaches in both sample efficiency and final performance. Our approach's effectiveness is shown through simulated and real-robot experiments with equivariant and recurrent versions of Advantage Actor-Critic (A2C)~\citep{mnih2016asynchronous} and Soft Actor-Critic (SAC)~\citep{haarnoja2018soft}.

\section{Related Works}
\noindent \textbf{Learning under Partial Observability }Unlike classical planning-based methods~\citep{kurniawati2008sarsop, somani2013despot, pineau2003point} that impractically require the complete dynamics of the environment, learning-based methods~\citep{hausknecht2015deep, ni2021recurrent, ma2020discriminative, han2019variational, yang2021recurrent, heess2015memory, igl2018deep, meng2021memory} utilize recurrent versions of common reinforcement learning (RL) algorithms for policy learning by directly interacting with the environment. To speed up learning, some methods leverage privileged information assumed available during training, such as the states, the belief about the environment states, or the fully observable policy~\citep{nguyen2020belief, nguyen2022leveraging, baisero2022asymmetric, baisero2022unbiased}\edit{, which are orthogonal to our approach}. Only a few prior works exploited domain symmetries under partial observability.~\citet{kang2012exploiting, doshi2008permutable} leveraged the invariance of the value function of some POMDPs given a state permutation and experimented on a classical planning-based method~\citep{pineau2003point} with the above limitations. Recently,~\citet{muglich2022equivariant} used equivariant networks to enforce symmetry when multiple agents coordinate. In contrast, we use model-free RL agents in a single-agent setting.

\noindent \textbf{Equivariant Learning }Equivariant networks have been successfully applied to a range of tasks such as point cloud analysis~\citep{thomas2018tensor} and molecular dynamics~\citep{satorras2021n, batzner20223}. A common approach is to build networks with group equivariant convolutions~\citep{cohen2016group} which are equivariant to arbitrary symmetry groups, such as 2D~\citep{lecun1995convolutional, weiler2019general} and 3D~\citep{thomas2018tensor, chen2021equivariant, cohen2018spherical, deng2021vector} transformations. Recently, for MDPs, equivariant networks have been applied to robotics~\citep{wang2022equivariant, wang2022so, zhu2022sample} and reinforcement learning~\citep{mondal2020group, van2020mdp} to improve sample efficiency. 
Closest to our work is~\citep{wang2022so}, which formalized group-invariant MDPs and used equivariant networks to perform robotic manipulation tasks. In contrast, this work extends equivariant reinforcement learning to partially observable environments, resulting in a new theory and method.

\noindent \textbf{Equivariance v.s.~Data Augmentation }Both methods leverage known domain symmetry to improve learning, but in different ways. On the one hand, data augmentation artificially expands the training data distribution with transformed versions of the data using the symmetry (e.g., rotating, cropping, or translating images~\citep{laskin2020curl, zhan2020framework}); then training a non-equivariant model. On the other hand, an equivariant approach bakes the domain symmetry directly into the model's weights, so an equivariant model can automatically generalize across input transformations even before training. Compared to an equivariant approach, a model trained using data augmentation alone is often less sample efficient~\citep{weiler2019general, cohen2018spherical}, generalizes worse~\citep{wang2020incorporating}, and requires a bigger architecture and longer training time for the same performance due to the extra work of learning symmetry injected in the data.

\section{Background}
Here, we review some background about POMDPs, some specific group theories used in our work, and finally, the basis of our approach --- the framework of group-invariant MDPs~\citep{wang2022so}.

\subsection{Partially Observable Markov Decision Processes}

A POMDP is defined by a tuple $(\mathcal{S}, \mathcal{A}, \Omega, b_0, T, R, O)$, where $\mathcal{S}$, $\mathcal{A}$, and $\Omega$ are the state space, the action space, and the observation space, respectively. $b_0\in\Delta\sset$ is the starting state distribution (a.k.a. the initial belief), states change and observations are emitted according to the stochastic dynamics function $T(s, a, s')$ and the stochastic observation function $O(a, s', o)$, respectively. Generally, an optimal agent may need to choose actions based on the entire observable action-observation history $h_t = (o_0, a_0, \ldots, a_{t-1}, o_t)$~\citep{singh1994learning}. Denoting the space of all histories as $\hset$, the goal is to find a history-policy $\pi\colon\hset\to\Delta\aset$ which maximizes the expected discounted return $J = \Exp\left[ \sum_{t=0}^\infty \gamma^t R(s_t, a_t) \right]$, where $\gamma \in [0, 1)$ is a discounting factor. An important concept in POMDPs is the belief $b(s) = \Pr (s \mid h)$, which is the probability that the true state is $s$ given an observed history $h$. The belief state is a sufficient statistic of the history, sufficient for optimal control. However, updating the belief state requires complete knowledge of the POMDP dynamic models, which are often hard to obtain.

\begin{figure}[htbp]
    \centering
    \includegraphics[width=0.8\linewidth]{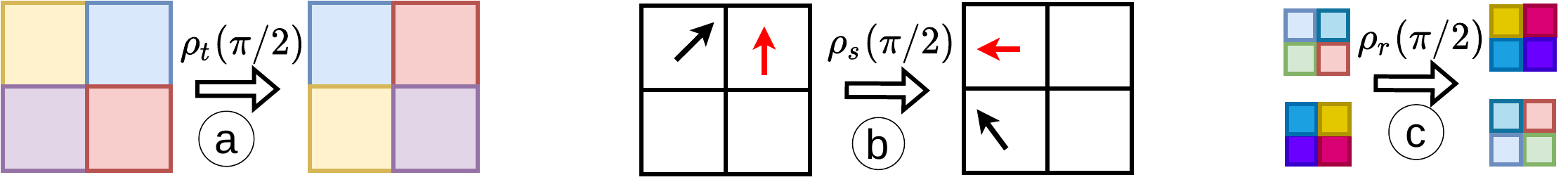}
    \caption{Illustration of a \edit{pixel-wise rotation} (characterized by a \emph{fixed} representation $\rho_f$) and a channel-wise rotation (characterized by the representation $\rho$). When $g$ is a $\pi/2$ CCW rotation, $\rho_f$ \emph{always} rotates the pixels while the effect of $\rho$ varies, e.g., the effect when $\rho$: \textbf{(a)} being a trivial representation ($\rho_t$) acting on a 1-channel feature map, \textbf{(b)} being a standard representation ($\rho_s$) acting on a vector field, and \textbf{(c)} being a regular representation ($\rho_r$) acting on a 2-channel feature map.}
    \label{fig:repr_viz}
    \vspace{-5pt}
\end{figure}

\subsection{$C_n$ and SO(2) Symmetry Groups}
In this work, we are mainly concerned with the symmetry group $G=\text{SO(2)}$ of continuous planar rotation, defined as $\text{SO(2)}=\text{Rot}_\theta: \{ 0 \leq \theta < 2\pi\}$. For a reduced computation complexity, we use the cyclic subgroup $C_n \leq \text{SO(2)}$ to approximate SO(2), which is defined as $C_n = \{ \text{Rot}_\theta: \theta \in \{ \frac{2\pi i }{n} \mid 0 \leq i < n\} \} $. In other words, $C_n$ defines $n$ rotations (i.e., group elements), which are multiples of $\frac{2\pi}{n}$. For instance, $C_4 = \{0, \pi/2, 2\pi/2, 3\pi/2\}$ and $C_8 = \{0, \pi/8, \dots, 6\pi/8, 7\pi/8\}$.

\subsection{Group Representations}

A \emph{group representation} is a mapping from a group $G$ to a $d$-dimensional general linear (GL) group, i.e., $\rho: G \rightarrow \text{GL}_d$ by assigning each group element $g \in G$ with an invertible matrix $\rho(g) \in \mathbb{R}^{d \times d}$. 

When $G = C_n$, the effect of a rotation $g \in C_n$ on a signal $x$ (i.e., $gx$) \edit{starts with a \emph{pixel-wise} rotation $\rho_f(g)^{-1} x$} (with a \emph{fixed} group representation $\rho_f$), followed by a \emph{channel-wise} rotation, i.e., $g x = \rho(g) (\rho_f(g)^{-1} x)$ (with the choice of group representation $\rho$). In this work, we consider three choices of the channel-wise representation $\rho$:

\noindent \textbf{Trivial Representation ($\rho=\rho_t$): }For $\forall g \in G$, $\rho_t$ associates $g$ with an identity matrix. For example, in~\cref{fig:repr_viz}a when $g$ is a $\pi/2$ counter-clockwise (CCW) rotation, and $x$ is a 1-channel feature map, \edit{$\rho_f$} rotates the pixels of $x$ while $\rho_t$ does not change the pixel values (i.e., the colors are unchanged). 

\noindent \textbf{Standard Representation ($\rho=\rho_s$): }For $\forall g \in G$, $\rho_s$ associates $g$ with a rotational matrix, i.e., $\rho_s(g)=g$. As in~\cref{fig:repr_viz}b, when $g$ is a $\pi/2$ CCW rotation and $x$ is a vector field input, $\rho_f$ rotates the positions of vectors (denoted as colored arrows), and $\rho_s$ rotates their orientations.

\noindent \textbf{Regular Representation ($\rho=\rho_r$): } For each $g \in G$, when acting on an input $x$, $p_r$ will cyclically permute the coordinates of $x$. \cref{fig:repr_viz}c illustrates when $g$ is a $\pi/2$ CCW rotation and $x$ is a 2-channel feature map, $\rho_f$ rotates each channel's pixels and $\rho_r$ permutes the orders of the channels.

\textbf{An Illustrative Example }Combining the group and the group representation fully characterizes how a signal will be transformed. For an illustrative example in a grid-world domain, see~\cref{app:simple_example}.

\subsection{Equivariance, Invariance, and Group-invariant MDPs}

Given $\phi\colon \mathcal{X} \to\mathcal{Y}$ and a symmetric group $G$ that acts on $\mathcal{X}$ and $\mathcal{Y}$, we say that $\phi$ is \emph{$G$-equivariant} if $\phi(g x) = g \phi(x)$, and \emph{$G$-invariant} if $\phi(g x) = \phi(x)$.  For the remainder of this document, we drop the prefix $G$ and simply refer to these properties as invariance and equivariance.

These notions have been adopted in the framework of group-invariant MDPs~\citep{wang2022so}. Specifically, an MDP $M_G = (\mathcal{S}, \mathcal{A}, T, R)$ is invariant if the transition and the reward function are invariant, i.e., 
$T(g s, g a, g s') = T(s, a, s')$ and $R(g s, g a) = R(s,a)$.  Group-invariant MDPs are associated with an invariant optimal Q-function, i.e.,
$Q^*(g s, g a) = Q^*(s, a)$, and at least one equivariant deterministic optimal policy, i.e., $\pi^*(g s) = g \pi^*(s)$.
These properties were exploited to build very sample-efficient equivariant agents under full observability~\citep{wang2022equivariant, wang2022so, zhu2022sample, wang2022robot}.

\begin{figure}[htbp]
    \centering
    \includegraphics[width=0.85\linewidth]{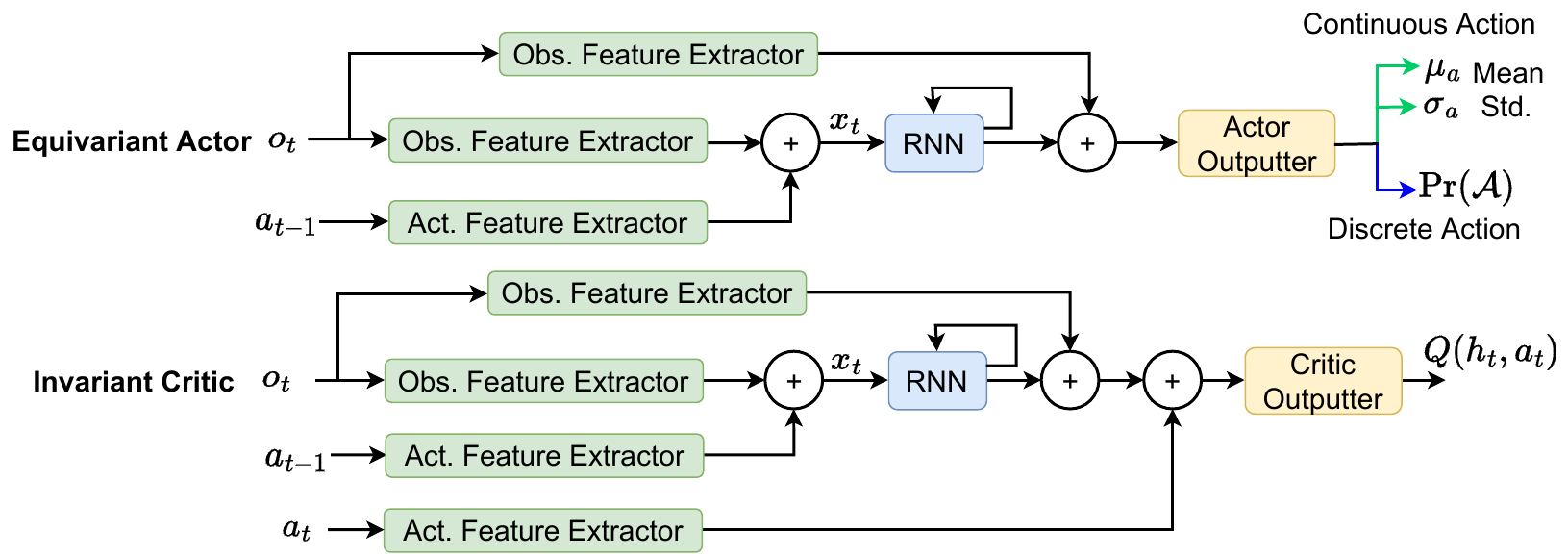}
    \caption{Our equivariant agent takes the \edit{commonly used} structure of a memory-based actor-critic agent~\citep{ni2021recurrent, ha2018world, zintgraf2019varibad, hung2019optimizing} but consists of an equivariant actor and an invariant critic, each constructed by equivariant modules. The actor's output can be learned means and standard deviations (for continuous action spaces) or a categorical distribution over the action space (for discrete action spaces).}
    \label{fig:agent}
    \vspace{-10pt}
\end{figure}

\section{Group-Invariant POMDPs}
\label{sect:analysis}

In this section, we extend the ideas from~\citep{wang2022so} to POMDPs and identify the basic set of assumptions that a POMDP needs to satisfy to have analogous invariance properties. We also note that while other assumptions might also lead to an invariant POMDP, ours are probably the most natural.  
\begin{definition}
\label{def:group-invariant-pomdp}
  We say a POMDP $P_G = (\mathcal{S}, \mathcal{A}, \Omega, b_0, T, R, O)$ is group-invariant with respect to group $G$ if it satisfies the following invariant properties for all $g \in G$:
  \begin{equation}\label{eq:requirements}
  \begin{aligned}
      T(g s, g a, g s') &= T(s, a, s') \quad
      R(g s, g a) = R(s, a) \\
      O(g a, g s', g o) &=  O(a, s', o) \quad
      b_0(g s) = b_0(s) \,.
  \end{aligned}
  \end{equation}
  \end{definition}
  \noindent 
  This extends the definition of the group invariant MDP from~\citep{wang2022so} by incorporating additional constraints on the observation function and the initial belief distribution.
  We also extend the group operations on histories.

\begin{definition}
    \label{def:history-action}
    Group operation $g$ acts on history $h_t$ according to $g h_t \coloneqq (g o_0, g a_0, \ldots, g a_{t-1}, g o_t)$.
\end{definition}

Finally, we show that group-invariant POMDPs exhibit similar properties and benefits as group-invariant MDPs.

\begin{theorem}
\label{thm:group-invariant-values-and-policies}
A group-invariant POMDP has an invariant optimal Q-function $Q^*(g h, g a)=Q^*(h, a)$, an invariant optimal value function $V^*(g h) = V^*(h)$, and at least one equivariant deterministic optimal policy $\pi^*(g h) = g\pi^*(h)$.
\end{theorem}

\begin{proof}
See~\cref{app:proof}.
\end{proof}

The above analysis allows us to constrain the value function and policy for a $G$-invariant POMDP to be invariant and equivariant, respectively, without eliminating optimal solutions.

\section{Equivariant Actor-Critic RL for POMDPs}

In this section, we introduce an equivariant agent that directly exhibits the desired properties of the optimal value function and policy, backed by the analysis in~\cref{thm:group-invariant-values-and-policies}. \cref{fig:agent} shows the agent, which takes a typical memory-based agent~\citep{ni2021recurrent, ha2018world, zintgraf2019varibad, hung2019optimizing} but has an equivariant actor and an invariant critic, each consisting of equivariant models. We later show that this very \edit{generic} agent, when embedded with the domain symmetry, can outperform significantly strong POMDP methods.

\subsection{Equivariant Modules}

We describe the details of the equivariant modules within our agent below, with the core components being \emph{equivariant CNNs}~\citep{e2cnn, weiler2019general}. For the implementation details, please see~\cref{app:implementation_details}.

\begin{figure}[htbp]
    \centering
    \includegraphics[width=0.9\linewidth]{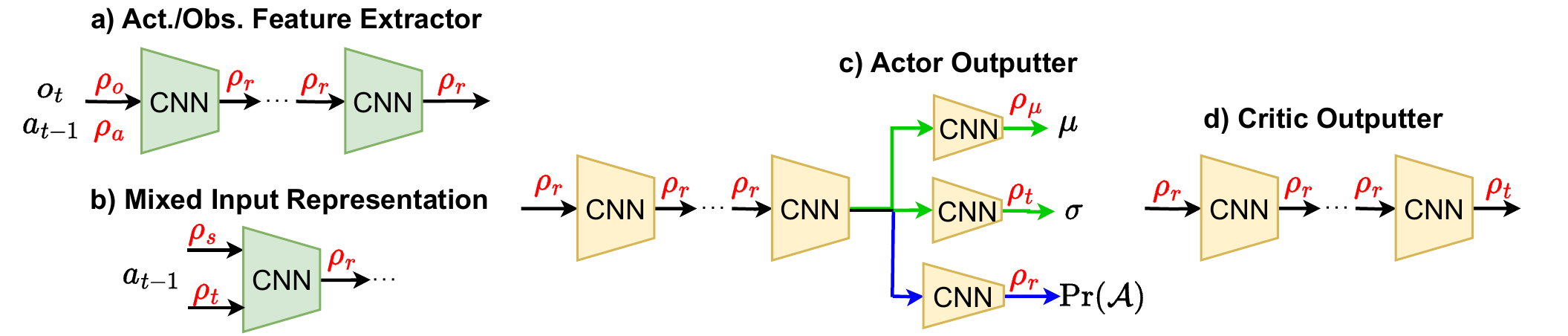}
    \caption{Equivariant Feature Extractor and Actor/Critic Outputter modules.}
    \label{fig:agent_details_2}
    \vspace{-10pt}
\end{figure}

\noindent \textbf{Equivariant Feature Extractor }This module takes observations or actions and outputs intermediate features for further processing. It comprises multiple equivariant CNN components chained sequentially as shown in~\cref{fig:agent_details_2}a. Its input representation is the observation representation $\rho_o$ or the action representation $\rho_a$. The intermediate and output representations are chosen to be the regular representation $\rho_r$, which empirically outperforms other representations~\citep{weiler2019general}. The input representation can be a single representation type or \emph{mixed}, i.e., a sum of different representations. A mixed representation is necessary when the input has different components that transform differently under a group transformation, e.g., one component rotates with the transform, and one component is unchanged. Such a case can be seen in~\cref{sect:explain_mixed} and is simplified in~\cref{fig:agent_details_2}b, where $\rho_a = \rho_s + \rho_t$.

\begin{wrapfigure}[10]{R}{0.3\linewidth}
  \centering
    \includegraphics[width=1.0\linewidth]{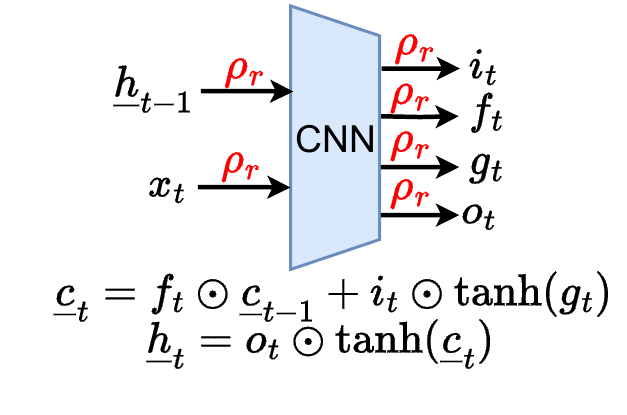}
    \vspace{-15pt}
    \caption{Equi. LSTM cell.}
    \label{fig:agent_details_1}
\end{wrapfigure}
\noindent \textbf{Equivariant Actor Outputter }\cref{fig:agent_details_2}c shows the input representation is regular because the signal coming in (the output of the RNN and the observation feature extractor modules) uses a regular representation. The output representation varies depending on the action type (discrete/continuous) and how a group transformation will affect an action. For \emph{discrete actions}, the module produces a categorical distribution over the action space. In this case, the output representation is the regular one $\rho_r$ as we want to change the (discrete) action when the history is transformed (see~\cref{app:simple_example} for an illustration). For \emph{continuous actions}, this module outputs the means with some representation $\rho_\mu$ and the standard deviations of actions with some representation $\rho_\sigma$ (as in A2C~\citep{mnih2016asynchronous}, PPO~\citep{schulman2017proximal}, or SAC~\citep{haarnoja2018soft}). The representations used for $\rho_\mu$ and $\rho_\sigma$ are mixed, as each action component might change differently under a group transformation (see~\cref{sect:explain_mixed}).

\noindent \textbf{Equivariant Critic Outputter }Because the optimal critic is invariant, this module (\cref{fig:agent_details_2}d) uses the trivial representation $\rho_t$ at the output to keep the output the same under a group transformation. Its input representation is regular, enforced by the output of the RNN and the action feature extractor.

\noindent \textbf{Equivariant Recurrent Neural Network }This is our contribution needed for constructing a POMDP equivariant agent (there is another similar component in~\citep{muglich2022equivariant}, but only model weights are released). We utilize an LSTM~\citep{hochreiter1997long} for this module, but the approach can also be modified for other types of RNNs. Specifically, given an input $x_t$ (e.g., the concatenated obs-action feature) and the previous hidden state $\underline{h}_{t-1}$, the input gate $i_t$, the forget gate $f_t$, the memory cell candidate $g_t$, and the output $o_t$ are computed as follows with $W$s and $b$s being learnable weights and biases:
\begin{equation}
\begin{aligned}
    i_t &= \text{sigmoid} (W_{xi}x_t + W_{hi}\underline{h}_{t-1} + b_i) \quad
    f_t = \text{sigmoid}(W_{xf}x_t + W_{hf}\underline{h}_{t-1} + b_f) \\ 
    o_t &= \text{sigmoid} (W_{xo}x_t + W_{ho}\underline{h}_{t-1} + b_o) \quad
    g_t = \text{tanh}(W_{xg}x_t + W_{hg}\underline{h}_{t-1} + b_g) \,.
\end{aligned}  
\end{equation}

\edit{The above equations do not make an equivariant RNN module. To enforce the equivariance, we compute all equations at once using an equivariant CNN module (\cref{fig:agent_details_1}), similar to the ConvLSTM network~\citep{shi2015convolutional}}. The input representation $\rho_r$ is determined by the output of the feature extractors ($x_t$) and the previous hidden state ($\underline{h}_{t-1}$), and the output representation is also regular. Next, we compute the next hidden state $\underline{h}_t$ and cell state $\underline{c}_t$ using the common LSTM equations with $\odot$ denoting the Hadamard product:
\begin{align}
    \underline{c}_t = f_t \odot \underline{c}_{t-1} + i_t \odot \text{tanh}(g_t) \quad 
    \underline{h}_t = o_t \odot \text{tanh}(\underline{c}_t) \,.
\end{align}
Finally, as the output of the RNN is an approximation of the belief state, to satisfy the condition of an invariant initial belief distribution, we set $\underline{c}_0$ and $\underline{h}_0$ with zero vectors.

\section{Experiments}
We compare the performance of learning agents on two grid-world domains (discrete actions and feature-based observations) and four robot domains (continuous actions and pixel observations).
\subsection{Domains}
We briefly describe our domains below. Please refer to~\cref{app:domain_details} for more specific details.

\begin{figure}[htbp]
  \centering
  \begin{subfigure}[t]{0.22\linewidth}
    \includegraphics[width=\linewidth]{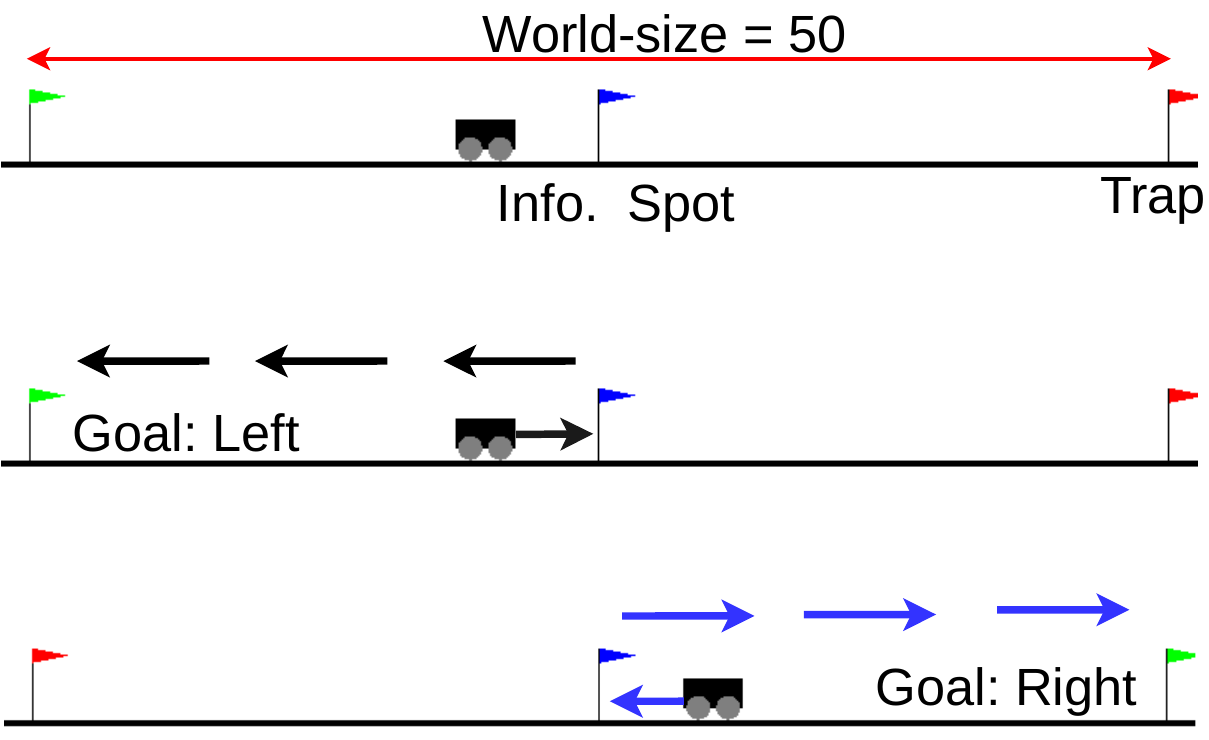}
    \caption*{\texttt{CarFlag-1D}} \label{fig:carflag-1d}
  \end{subfigure}
  \begin{subfigure}[t]{0.14\linewidth}
    \includegraphics[width=\linewidth]{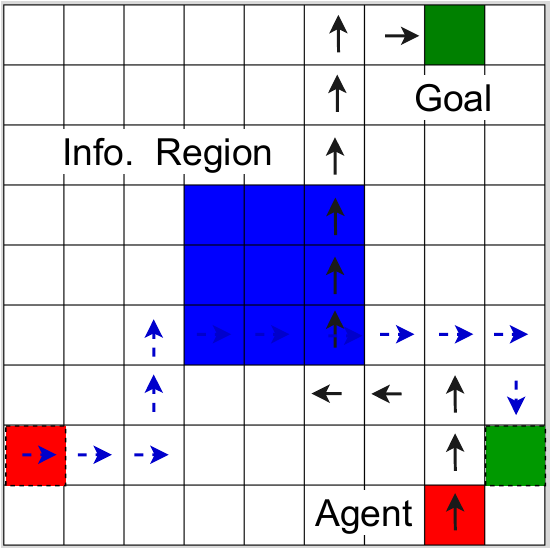}
    \caption*{\texttt{CarFlag-2D}} \label{fig:carflag-2d}
  \end{subfigure}
  \begin{subfigure}[t]{0.16\linewidth}
    \includegraphics[width=\linewidth]{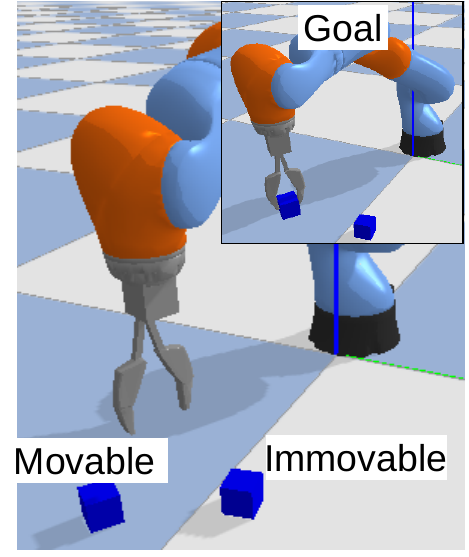}
    \caption*{\texttt{Block-Picking}} \label{fig:block-picking}
  \end{subfigure}
  \begin{subfigure}[t]{0.16\linewidth}
    \includegraphics[width=\linewidth]{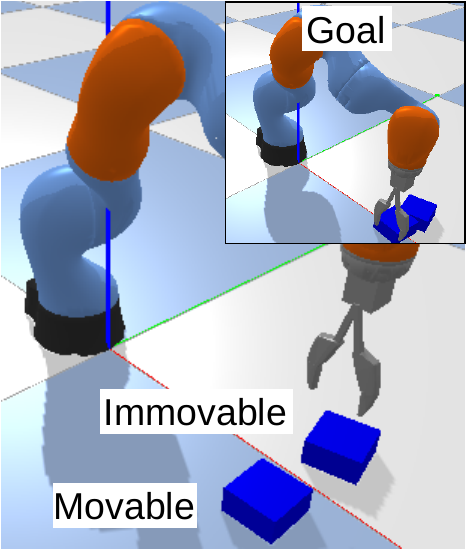}
    \caption*{\texttt{Block-Pulling}} \label{fig:block-pulling}
  \end{subfigure}
  \begin{subfigure}[t]{0.15\linewidth}
    \includegraphics[width=\linewidth]{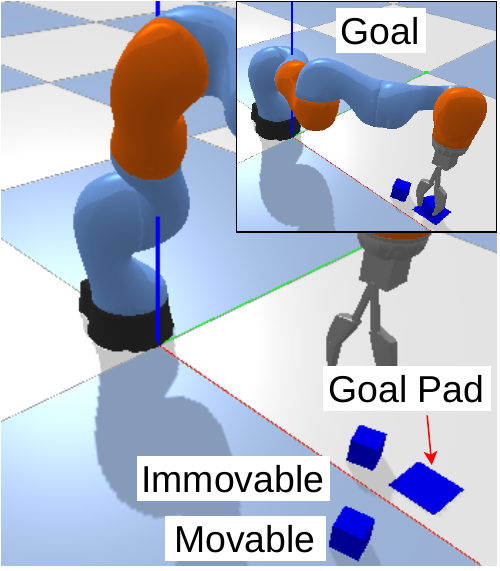}
    \caption*{\texttt{Block-Pushing}} \label{fig:block-pushing}
  \end{subfigure}
  \begin{subfigure}[t]{0.14\linewidth}
    \includegraphics[width=\linewidth]{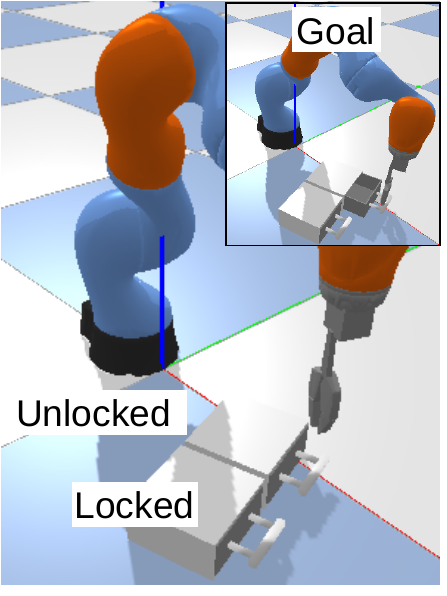}
    \caption*{\texttt{Drawer-Opening}}\label{fig:drawer-opening}
  \end{subfigure}
\caption{Our domains. The first two domains have feature-based observation and discrete action spaces. The last four domains have pixel-based observations and continuous action spaces.}
\label{fig:domains}
\vspace{-10pt}
\end{figure}

\subsubsection{Grid-World Domains}

There are two versions of \texttt{CarFlag}~\citep{nguyen2021penvs} in~\cref{fig:domains}, where an agent must reach a goal (green), whose position is visible \emph{only} when the agent visits an unknown information region. For instance, in \texttt{CarFlag-1D}, the agent must visit the central blue flag to get the side (left/right) of the goal; or in \texttt{CarFlag-2D}, the agent must visit the central blue region to see the coordinate of the goal cell. We also illustrate the domain symmetry in the figure: in these domains, when the starting position and the goal location are transformed (flipped in \texttt{CarFlag-1D} or rotated by $\pi/2$ radians clockwise in \texttt{CarFlag-2D}, the optimal trajectories will be transformed similarly, i.e., black $\rightarrow$ blue trajectories.

\subsubsection{Robot Manipulation Domains}
\cref{fig:domains} shows our robot manipulation domains (extended from the BulletArm suite~\citep{wang2022bulletarm}), where a robot arm must perform individual manipulation tasks (i.e., picking, pulling, pushing, and opening) using \emph{top-down} depth images to win a sparse reward. In these domains, only one object is manipulable, but both objects are the same if only relying on the current image. Therefore, the agent must actively check the objects' mobility and remember past interactions with the objects to determine the next action. Specifically, in \texttt{Block-Picking}, the agent needs to pick the movable block up. In \texttt{Block-Pulling}, the agent needs to pull the movable block to be in contact with the other block. In \texttt{Block-Pushing}, the goal is to push the movable block to a goal pad. In \texttt{Drawer-Opening}, the agent is tasked to open an unlocked drawer between a locked and an unlocked one.

In these domains, the transition function is invariant because the Newtonian physics applied to the interaction is invariant to the location of the reference frame. The reward function is invariant by definition. Using top-down depth images makes the observation function invariant. If the initial belief is assumed invariant, then according to~\cref{def:group-invariant-pomdp}, these domains are group-invariant POMDPs.
 
\subsection{Agents}
We compare our proposed agents (instances of the structure in~\cref{fig:agent} applied to A2C~\citep{mnih2016asynchronous} and SAC~\citep{haarnoja2018soft}) against a diverse set of baselines, including on-policy/off-policy, model-based/model-free, and \edit{generic}/specialized POMDP methods (see~\cref{app:implementation_details} and~\cref{app:baseline_details} for more details).
\subsubsection{Grid-world Domains}
\textbf{RA2C}~\citep{pytorchrl} is a recurrent version of A2C~\citep{mnih2016asynchronous}. \textbf{Equi-RA2C} is our proposed architecture applied to A2C. \textbf{DPFRL}~\citep{ma2020discriminative} is a state-of-the-art model-based POMDP baseline where an A2C agent is given features produced by a differentiable particle filter. \textbf{DreamerV2}~\citep{hafner2020mastering} and \textbf{DreamerV3}~\citep{hafner2023mastering} are strong model-based methods that learn a recurrent world model, thus, can work with POMDPs.

\noindent \textbf{No Data Augmentation for All }Since all methods are on-policy algorithms \edit{, augmented data using the domain symmetry only becomes on-policy \emph{only} for \textbf{Equi-RA2C} dues to its unique symmetry-awareness. Therefore, for a fair comparison, we do not perform any data augmentation.}

\subsubsection{Robot Manipulation Domains}
While on-policy RA2C or Dreamer-v2 can handle continuous action spaces in these domains, there is no clear way to leverage expert demonstrations necessary to efficiently solve the robot manipulation tasks with sparse rewards in~\cref{fig:domains}. Thus, we switch to SAC~\citep{haarnoja2018soft} as the base RL algorithm, where we can pre-populate its replay buffer with demonstration episodes. In our experiments, \textbf{RSAC}~\citep{ni2021recurrent} is a non-equivariant recurrent SAC agent. \textbf{Equi-RSAC} is our proposed method applied to SAC. We also compare with \emph{recurrent} versions of two strong data augmentation baselines: \textbf{RAD-Crop-RSAC}~\citep{laskin2020reinforcement} and \textbf{DrQ-Shift-RSAC}~\citep{kostrikov2020image}. These specific data augmentation techniques, i.e., random cropping and shifting (see~\cref{app:viz_data_aug} for visualizations), are chosen among others because they were reported to perform best~\citep{kostrikov2020image}. To train RAD-Crop-RSAC, for each training episode, an auxiliary episode is created by using the same random cropping for every depth image inside the original episode. DrQ-Shift-RSAC applies two random shifts to each depth image in a training episode to create two. The Q-target and the Q-values are then computed by averaging the values computed on the two episodes. Finally, \textbf{SLAC}~\citep{lee2020stochastic} learns a latent model from pixels and then uses SAC on the latent space by using the observation-action history (instead of the latent state) for the actor and the latent state samples for the critic. This enables SLAC to scale to more difficult tasks.

\noindent \textbf{Demonstrations + Rotational Data Augmentation for All} All replay buffers are pre-populated with 80 expert episodes to overcome the reward sparsity. Moreover, we augment the training data by applying the same random rotation for every action and observation inside a training episode (see~\cref{app:viz_data_aug} for visualizations). Note that these rotational data augmentations are applied in addition to the existing data augmentation techniques in RAD-Crop-RSAC and DrQ-Shift-RSAC.

\subsection{Results}

\begin{figure}
    \centering
    \includegraphics[width=\linewidth]{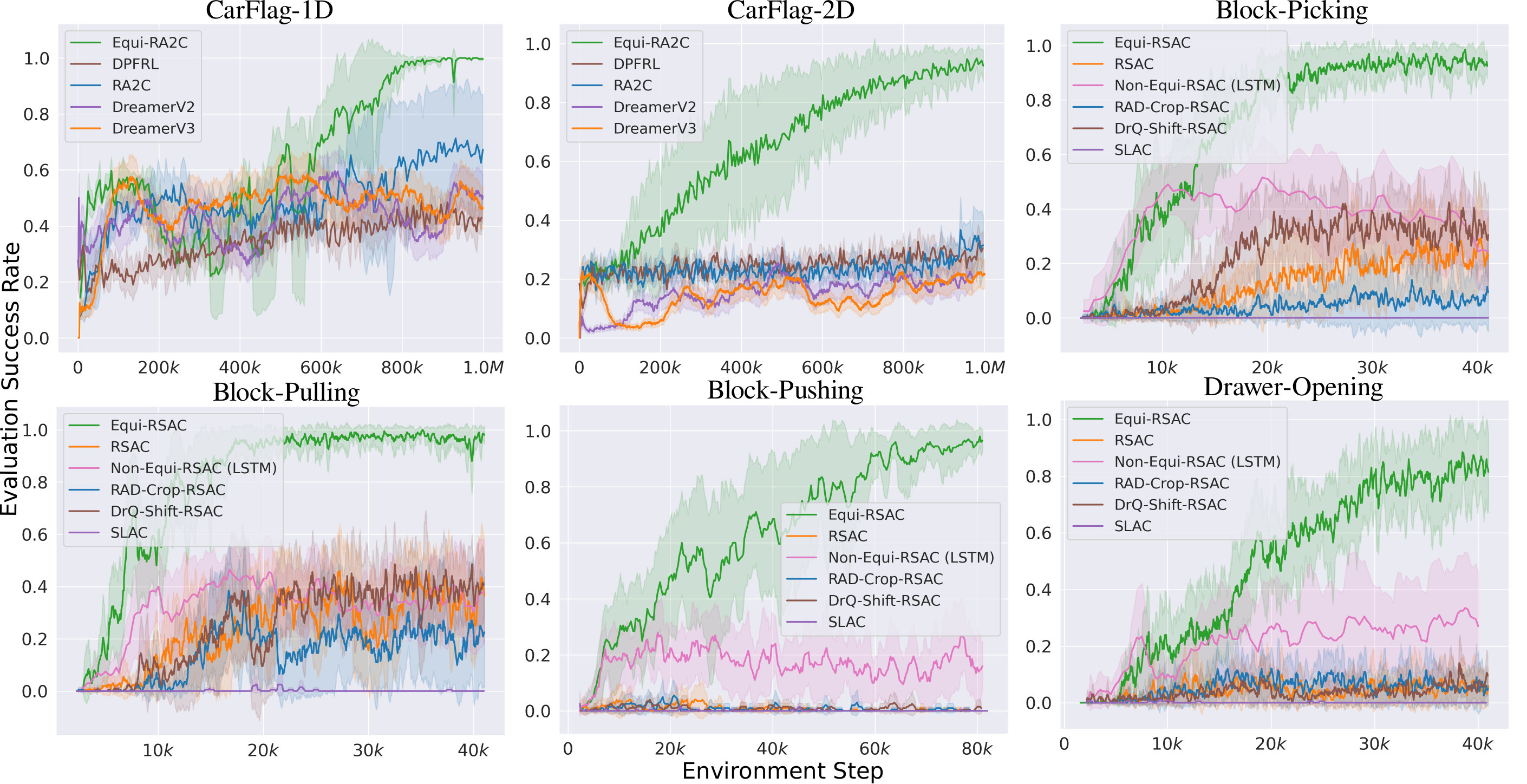}
    \caption{Evaluated success rates (four seeds, shaded areas denote one std.). No data augmentation is used in \texttt{CarFlag} domains. Rotational data augmentation is used for \emph{all} agents in robot domains.}
    \label{fig:all_res}
    \vspace{-15pt}
\end{figure}

\noindent \textbf{Grid-world Domains}
\cref{fig:all_res} shows that Equi-RA2C is significantly more sample efficient than the baselines. Moreover, the dominance of our method is also seen with variants of these domains with different sizes (see~\cref{app:extra_results}). DPFRL did not perform well potentially because of the reward sparsity, which was also previously reported in~\citep{nguyen2022hierarchical}. DreamerV2 and DreamerV3 also perform poorly, even with many more learnable parameters of the models, potentially indicating that learning a good model under partial observability and sparse rewards might be more challenging than in the domains originally tested. For instance, most Atari games and locomotion tasks in the DeepMind Control suite~\citep{tassa2018deepmind} have low levels of partial observability and provide dense rewards.

\begin{table}[htbp]
\centering
\begin{tabular}{rcc|rcc}
\textbf{$d_{1D}$} & \textbf{Equi-RA2C 1M ($\uparrow$)} & \textbf{RA2C 1M ($\uparrow$)} & \textbf{$d_{2D}$} & \textbf{Equi-RA2C 2M ($\uparrow$)} & \textbf{RA2C 2M ($\uparrow$)} \\ \midrule
         -10 & 0.51 $\pm$ 0.06 & \textbf{0.78} $\pm$ 0.08 & -2 & \textbf{0.38} $\pm$ 0.11 & \textbf{0.31} $\pm$ 0.04  \\
         10 & 0.44 $\pm$ 0.07 & \textbf{0.70} $\pm$ 0.22 & 2 & \textbf{0.41} $\pm$ 0.12 & 0.28 $\pm$ 0.02 \\
         -5 & \textbf{0.95} $\pm$ 0.03 & 0.72 $\pm$ 0.33 & -1 & \textbf{0.58} $\pm$ 0.24 & 0.38 $\pm$ 0.13 \\
         5 & \textbf{0.99} $\pm$ 0.03 & 0.76 $\pm$ 0.24 & 1 & \textbf{0.71} $\pm$ 0.14 & 0.30 $\pm$ 0.03 
\end{tabular}
    \caption{The convergent success rates (mean $\pm$ one standard deviation) of Equi-RA2C and RA2C agents in asymmetric variants of \texttt{CarFlag} (after 1M and 2M training steps). $d_{1D}$ and $d_{2D}$ refer to the distance from the information region to the world center (see~\cref{app:domain_details} for illustrations).}
    \label{tab:asym_res}
    \vspace{-20pt}
\end{table}

\noindent \textbf{Robot Manipulation Domains}
Clearly from \cref{fig:all_res}, Equi-RSAC strongly outperforms other baselines in all domains, with itself being the only agent that can reach a satisfactory performance. Across all domains, without the equivariant LSTM module (denoted as Non-Equi-RSAC (LSTM), which can be considered as a naive extension of~\citep{wang2022so}), the performance degrades significantly, even though it starts pretty well. SLAC, surprisingly, performs the worst. A possible reason is that SLAC was originally only tested on domains with dense rewards and low levels of partial observability (e.g., locomotion domains in DeepMind Control Suite~\citep{tassa2018deepmind} and OpenAI Gym~\citep{brockman2016openai}). Another potential reason is the usage of concatenated feature vectors across an episode for the actor, which can be very high-dimensional for a long episode. Moreover, we also found that the trained latent model failed to sufficiently reconstruct the observation in \texttt{Block-Pulling} (see~\cref{app:slac_viz} for more details).

\noindent \textbf{Additional Results} See \cref{app:ablation_studies} for the performance when using a different group symmetry ($C_8$ instead of $C_4$), utilizing symmetry partially (for either actor or critic only), and $\underline{c}_0$ and $\underline{h}_0$ being random instead of zero vectors. Other additional results are shown in~\cref{app:extra_results}. 

\subsection{Using Equivariant Models on Domains with Imperfect Symmetry}
\label{sect:asym}
We investigate the performance when the perfect symmetry does not hold in asymmetric variants of \texttt{CarFlag}, created by offsetting the information region a distance $d$ from the world center (see~\cref{app:domain_details}). From the final success rates shown in~\cref{tab:asym_res} (see~\cref{app:extra_asym_results} for learning curves), we can see that equivariant Equi-RA2C still outperforms non-equivariant RA2C when the domains are close to perfect symmetry, i.e., when $d = d_{1D}=5$ in \texttt{CarFlag-1D} or $d = d_{2D}=1, 2$ in \texttt{CarFlag-2D}. However, a bigger symmetry gap might lead to the sub-optimality of equivariant agents. As evidence, Equi-RA2C performs worse than RA2C when $d=10, -10$ in \texttt{CarFlag-1D}.

\subsection{Zero-shot Transfers to Real Hardware}

\begin{minipage}{\textwidth}
\begin{minipage}[b]{0.49\textwidth}
  \centering
    \includegraphics[width=0.9\linewidth]{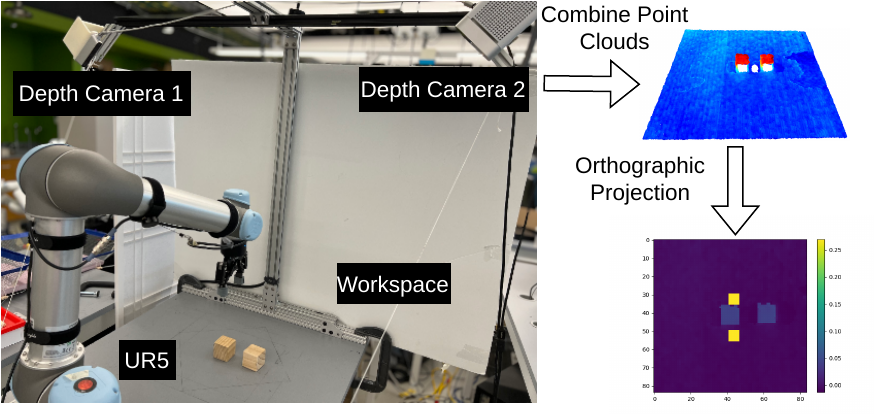}
    \captionof{figure}{Experimental robot setup.}
    \label{fig:workspace}
\end{minipage}
\hfill
\begin{minipage}[b]{0.49\textwidth}
\centering
\begin{tabular}{lc}
\textbf{Domain} & \textbf{Success Sim / Real} ($\uparrow$)\\ \midrule
\texttt{Block-Picking} & 1.00 / 0.90 \\
\texttt{Block-Pulling} & 1.00 / 0.88 \\
\texttt{Block-Pushing} & 0.96 / 0.92 \\
\texttt{Drawer-Opening} & 0.95 / 0.80 \\
\end{tabular}
\captionof{table}{Average success rates of sim2real transfers over 50 episodes.}
\label{tab:sim2real}
\end{minipage}
\end{minipage}

Because only our agents can perform well in simulation, we transfer their best policies in simulation to a UR5 robot (see~\cref{fig:workspace}). We combine the point clouds from two side-view cameras to create a top-down depth image using a projection at the gripper's position. We roll out 50 episodes, divided equally into test cases when the agents first manipulate the immovable or movable objects. \cref{tab:sim2real} shows that the learned policies can be zero-shot transferred well in the real world regardless of small performance drops in all domains (see our supplementary video for policy visualizations). The biggest performance drop is in \texttt{Draw-Opening}, in which the transferred policies sometimes clumsily move one drawer far away from the other, creating a novel scene never seen in simulation.

\section{Conclusion and Limitations}

\textbf{Conclusion }In this work, we introduced group-invariant POMDPs and proposed equivariant actor-critic RL agents as an effective solution method. Through extensive experiments, our proposed equivariant agents can tackle realistic and challenging robotic manipulation domains much better than non-equivariant approaches with learned policies zero-shot transferable to a real robot.

\textbf{Limitations }A limitation of most equivariant approaches, including ours, is the requirement of imperfect symmetry, which might be present when images are affected by non-symmetric factors, e.g., side view instead of top-down view or asymmetric noises. Fortunately, under full observability, recent empirical~\citep{wang2022surprising, yang2023neural} and theory work~\citep{wang2023general} show that an equivariant model can still outperform non-equivariant approaches in many such cases. Together with the results in~\cref{sect:asym}, our approach might still perform better than unstructured agents even under imperfect symmetry.



\clearpage

\acknowledgments{We are grateful to Elise van der Pol for her early contributions to this project, including her suggestions on incorporating symmetries in the history space and her ideas for implementing an equivariant LSTM. This material is supported by the Army Research Office under award number W911NF20-1-0265; the U.S. Office of Naval Research under award number N00014-19-1-2131; NSF grants 1816382, 1830425, 1724257, and 1724191.
}


\bibliography{refs}  

\clearpage
\appendix

\section{Illustration of Group and Group Representation in CarFlag-2D}
\label{app:simple_example}
\paragraph{Domain} We consider a small version of \texttt{CarFlag-2D} (see~\cref{fig:domain_illustration}) with a grid size of 3x3, where the agent (red) must navigate to an unknown target cell (green) in a grid world. The agent can always observe its current location but only observe the target cell when it visits the information cell (blue), which is also unknown to the agent.

\paragraph{Observation} The observation is a two-channel image size 2x3x3, where the first channel encodes the agent's location and the second encodes the target location. The values of the second channel are non-zero only when the agent is at the information cell (\cref{fig:domain_illustration}).

\paragraph{Actions} Movements in four directions (the location does not change if going out of the world).

\begin{figure}[htbp]
    \centering
    \includegraphics[width=0.9\linewidth]{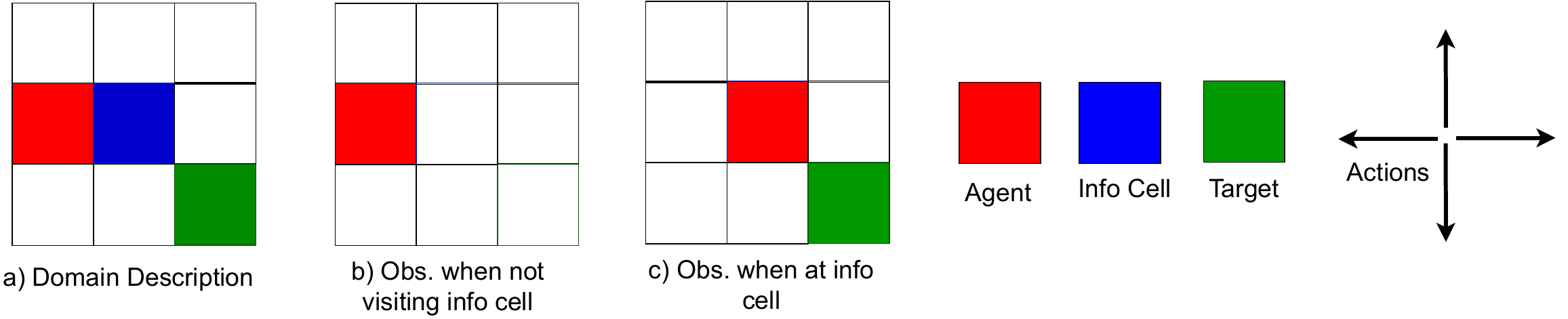}
    \caption{Illustration of the domain, action, and observation.}
    \label{fig:domain_illustration}
    \vspace{-10pt}
\end{figure}

\paragraph{Domain Symmetry} Consider Scenario 1 and Scenario 2 in~\cref{fig:domain_symmetry}: Scenario 2 is the rotated version of Scenario 1 after a $90^\circ$ counter-clockwise (CCW) rotation. Therefore, an optimal path (denoted with colored arrows) in Scenario 1 is equally optimal in Scenario 2 if we rotate the path similarly. The same happens if the rotation angle is $180^\circ$ or $270^\circ$. We can capture the rotational symmetry using group $C_4 = \{0^\circ, 90^\circ, 180^\circ, 270^\circ\}$.
\begin{figure}[htbp]
    \centering
    \includegraphics[width=0.4\linewidth]{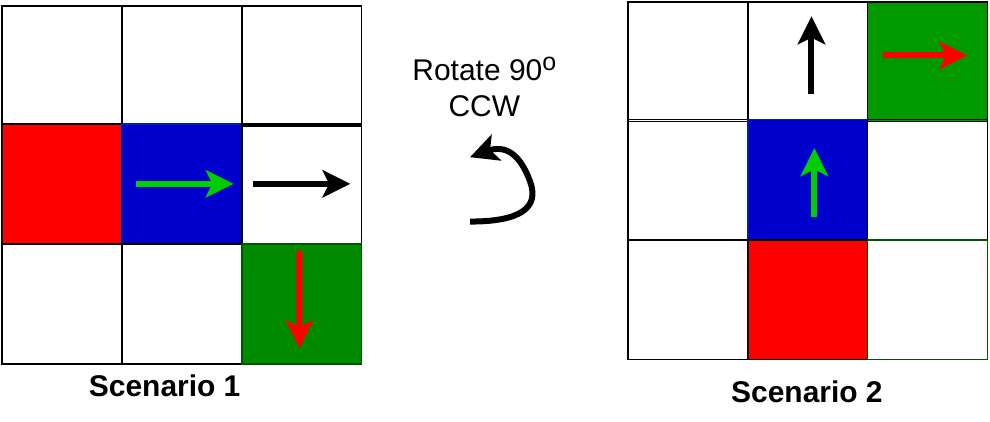}
    \caption{Illustration of domain symmetry. Scenario 2 is the rotated version of Scenario 1 after a $90^\circ$ counter-clockwise rotation. An optimal path in Scenario 1 can be rotated similarly to become optimal in Scenario 2.}
    \label{fig:domain_symmetry}
    \vspace{-10pt}
\end{figure}

\begin{figure}[htbp]
    \centering
    \includegraphics[width=0.4\linewidth]{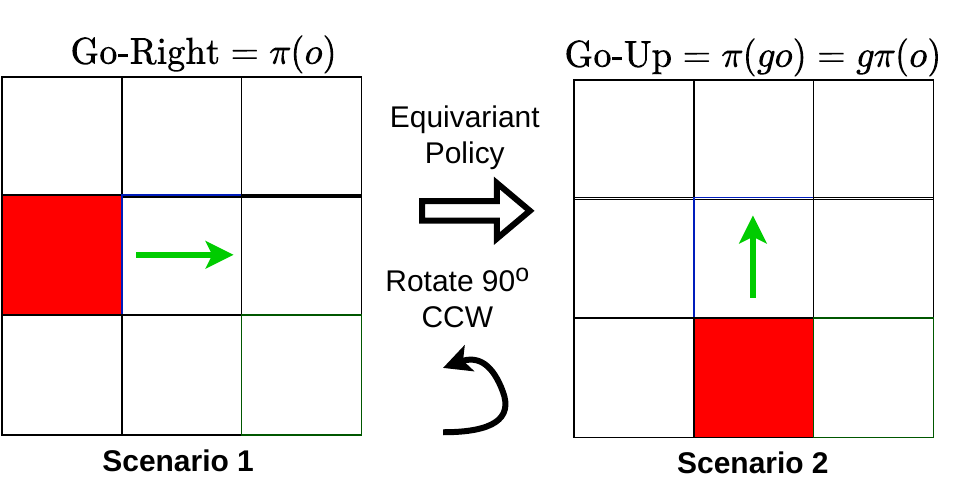}
    \caption{Illustration of the effect of an equivariant policy: the action is automatically rotated when the input observation is rotated.}
    \label{fig:equi_policy}
\end{figure}

\paragraph{Equivariant Policy} We want our policy to automatically capture the domain symmetry above by making it equivariant. In~\cref{fig:equi_policy}, we illustrate the property of an equivariant policy $\pi$ with $g$ being a $90^\circ$ CCW rotation. In Scenario 1, given the first observation $o$, we assume that $\pi$ already knows it should go to the right $\texttt{Go-Right} = \pi(o)$ towards the information cell. Now, moving to Scenario 2, when the first observation is the rotated version of $o$, denoted as $go$. An equivariant policy automatically calculates the next action in Scenario 2 as:
\begin{equation} \label{eq:0}
    \pi(go) = g\pi(o) = g(\texttt{Go-Right}) = \texttt{Go-Up}
\end{equation}

\paragraph{Group Representation}
From~\cref{eq:0}, to construct an equivariant policy, we need to define how a rotation $g$ acts on an observation at the input (i.e., define $go$) and on an action at the output (i.e., define $g\pi(o)$). For that purpose, besides defining the group, we need to specify the group representation, i.e., defining an \emph{observation group representation} $\rho_o$ and an \emph{action group representation} $\rho_a$ for $\pi$ (see below).

\paragraph{Example of Group Acting on Observation and Action}
The effect of a $90^0$ CCW rotation $g$ on the observation via a \emph{trivial} representation $\rho_o = \rho_t$ and the action via a \emph{regular} representation $\rho_a = \rho_r$ is illustrated in~\cref{fig:equi_policy_demon}. A trivial representation $\rho_t$ rotates the observation (like rotating the normal image) while keeping the pixel values unchanged (the value in cell 0 is still $(0_0, 0_1)$). In contrast, a regular representation $\rho_r$ permutes the action distribution output, resulting in a different action, i.e., \texttt{Go-Right} $\rightarrow$ \texttt{Go-Up}). This automatic change only happens if the policy is equivariant.  
\begin{figure}[htbp]
    \centering
    \includegraphics[width=0.8\linewidth]{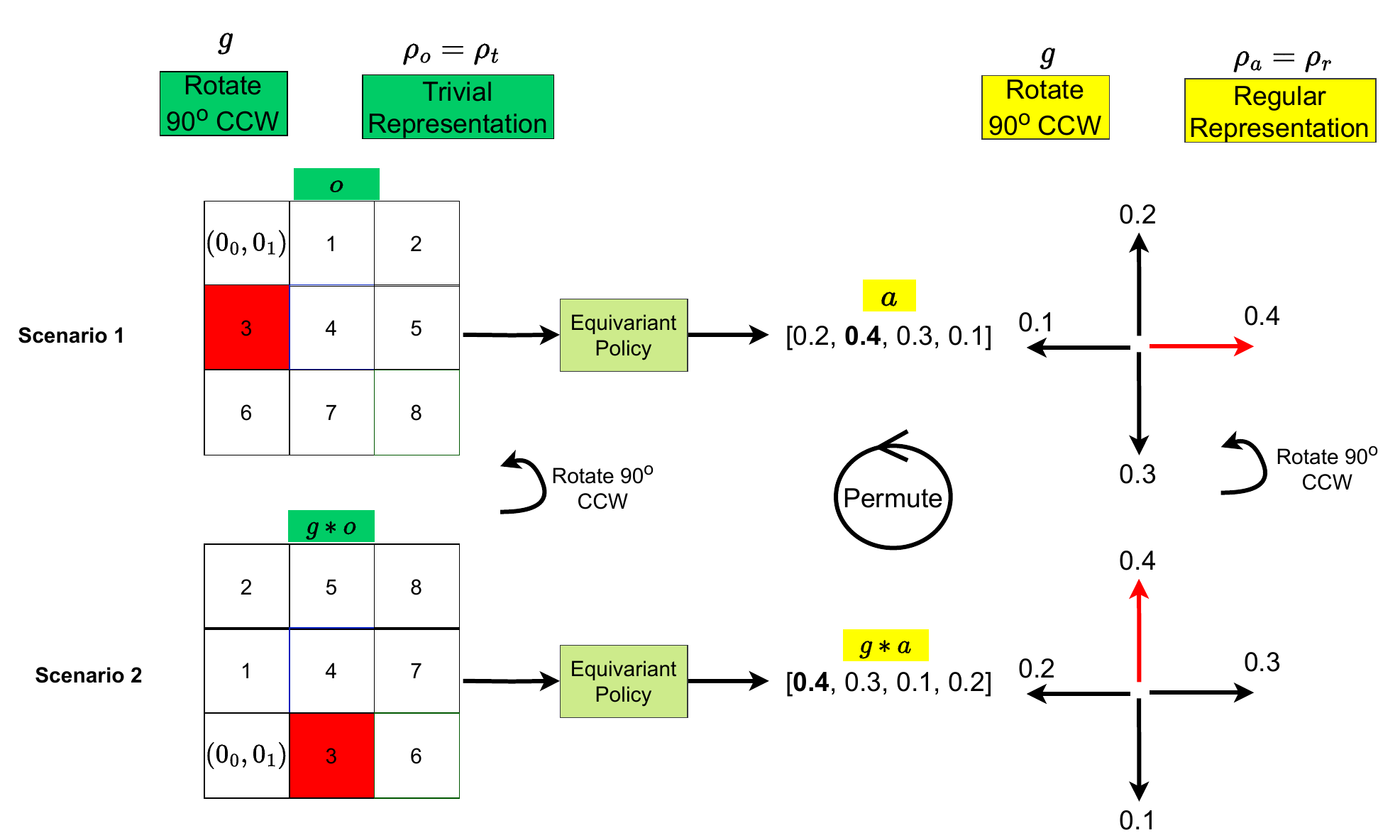}
    \caption{Illustration of the effect of a trivial representation acting on the observation ($\rho_o = \rho_t$) and a regular representation acting on the action ($\rho_a = \rho_r$) with $g$ being a $90^0$ CCW rotation. $\rho_o$ rotates the observation and keeps the pixel values unchanged. $\rho_a$ permutes the action distribution output, resulting in a different action (\texttt{Go-Right} $\rightarrow$ \texttt{Go-Up}).}
    \label{fig:equi_policy_demon}
\end{figure}

\clearpage
\section{Proof of~\cref{thm:group-invariant-values-and-policies}}
\label{app:proof}
In this section, we introduce the framework of history representation MDP~\citep{nguyen2021converting} and prove a supporting lemma before arriving at the proof.
\subsection{History Representation MDP}
\label{sect:hr-mdp}

A POMDP can be converted into a \emph{history representation MDP} (HR-MDP)~\citep{nguyen2021converting} whose state is a sufficient statistic of the POMDP history for control, e.g., the well-known \emph{Belief-MDP}~\citep{kaelbling1998planning} construct is a special case of an HR-MDP based on the belief representation.   Useful representations such as the belief might require a known POMDP model;  however, we adopt a model-free approach with no such knowledge and, therefore, use the trivial identity representation whereby the history is represented by itself.  This effectively converts the POMDP into an equivalent \emph{History-MDP}, which is defined by the tuple $(\mathcal{H}, \mathcal{A}, \bar{T}, \bar{R})$, where:
\begin{align}
    \bar{T}(h, a, h') = \Exp_{o\mid h, a} \left[ \Ind \{ h' = hao \} \right] \quad 
    \bar{R}(h, a) = \Exp_{s\mid h} \left[ R(s, a) \right] \,, \label{eq:history-mdp-reward-transition}
\end{align}
where $\Ind\{\cdot\}$ is the indicator function, and
\begin{align}
\Pr(o \mid h, a) &= \Exp_{s\mid h} \left[ \sum_{s'} T(s, a, s') O(a, s', o) \right] \,, \label{eq:pr:o:ha} \\
\Pr(s'\mid h') &\propto \Exp_{s\mid h} \left[ T(s, a, s') \right] O(a, s', o) \,. \label{eq:pr:s:h}
\end{align}
\subsection{Supporting Lemma}
\setcounter{lemma}{0}
\begin{lemma}
    The belief function of a group-invariant POMDP (as defined by \cref{def:group-invariant-pomdp}) is group-invariant,
    \begin{equation}
        \Pr(g s\mid g h) = \Pr(s\mid h) \,.
    \end{equation}
\end{lemma}

\begin{proof}[Proof By Induction]
    $ $\newline
    \textbf{Base Case}.
    We first prove that the belief after the first observation is invariant. We note here that the observation function for the first timestep takes the form $O(s, o)$, with no preceding action.
    \begin{align}
        \Pr(g s_0\mid g o_0) &\propto b_0(g s_0) O(g s_0, g o_0) 
        = b_0(s_0) O(s_0, o_0) 
        \propto \Pr(s_0\mid o_0) \,.
    \end{align}
    Since $\Pr(g s_0\mid g o_0)$ and $\Pr(s_0\mid o_0)$ are both proportional to the same quantity, and they are both normalized to be distributions over states, then they are themselves equal.
    
    \textbf{Inductive Step.}
    We then prove that if $\Pr(s_t\mid h_t)$ is invariant, then $\Pr(s_{t+1}\mid h_{t+1})$ is also invariant.  Per \cref{eq:pr:s:h},
    \begin{align}
        \phantom{\propto} \Pr(g s_{t+1}\mid g h_{t+1}) 
        &\propto \Pr(g s_t\mid g h_t) T(g s_t, g a_t, g s_{t+1}) O(g a_t, g s_{t+1}, g o_{t+1}) \nonumber \\
        &= \Pr(s_t\mid h_t) T(s_t, a_t, s_{t+1}) O(a_t, s_{t+1}, o_{t+1}) 
        \propto \Pr(s_{t+1}\mid h_{t+1}) \,.
    \end{align}
    Since $\Pr(g s_{t+1}\mid g h_{t+1})$ and $\Pr(s_{t+1}\mid h_{t+1})$ are both proportional to the same quantity, and they are both normalized to be distributions over states, then they are themselves equal.
    By induction, given the base case and the inductive step, the belief function $\Pr(s_t\mid h_t)$ is invariant for any $t$.
\end{proof}

\subsection{Proof}
\begin{proof}
We begin by constructing the History-MDP associated with a group-invariant POMDP and showing that it is itself a group-invariant MDP.
The transition and reward functions of the History-MDP are shown in \cref{eq:history-mdp-reward-transition} and satisfy the group invariance properties.

For this proof, it is simpler to express the history transition function as $\bar T(h, a, h') = \Pr(o\mid h, a)$, where $o$ is the observation (if any exists) s.t. $h' = hao$.  If no such observation exists, then $\bar T(h, a, h') = 0$ is trivially invariant.  If it does exist, then it is necessarily the last observation of $h'$,
\begin{align}
    \bar T(gh, ga, gh') &= \Pr(go\mid gh, ga) 
    = \sum_{s,s'} \Pr(s\mid gh) \Pr(s'\mid s, ga) \Pr(go\mid ga, s') \nonumber \\
    \intertext{since $g$ permutes the elements of $\sset$, we can re-index using $s = g \bar s$ and $s' = g \bar s'$,}
    &= \sum_{\bar s, \bar s'} \Pr(g\bar s\mid gh) T(g\bar s'\mid g\bar s, ga) O(go\mid ga, g\bar s') \nonumber \\
    &= \sum_{s, s'} \Pr(s\mid h) T(s'\mid s, a) O(o\mid a, s') 
    = \bar T(h, a, h') \,.
\end{align}

By using $s=g\bar s$, we proceed similarly for history rewards,
\begin{align}
    \bar R(gh, ga) &= \sum_s \Pr(s\mid gh) R(s, ga) 
    %
    %
    = \sum_{\bar s} \Pr(g \bar s\mid gh) R(g\bar s, ga) \nonumber \\
    &= \sum_s \Pr(s\mid h) R(s, a) 
    = \bar R(h, a) \,.
\end{align}

Therefore, $\bar T(h, a, h')$ and $\bar R(h, a)$ are invariant, and History-MDPs are group-invariant MDPs.
By the theory developed in~\citep{wang2022so}, this implies that the optimal Q-value function $Q^*(h, a)$ is invariant and that there exists at least one equivariant deterministic optimal policy $\pi^*(h)$. Moreover, 
%
\begin{align}
    V^*(gh) &= Q^*(gh, \pi^*(gh)) = Q^*(gh, g\pi^*(h)) \nonumber \\
    &= Q^*(h, \pi^*(h)) = V^*(h) \,,
\end{align}
this ends our proof by showing that $V^*(h)$ is invariant.
\end{proof}

\clearpage
\section{Environment Details}
\label{app:domain_details}

\subsection{Grid-world Domains}
\subsubsection{\texttt{CarFlag-1D}}
\begin{itemize}
    \item Action: Go-Left or Go-Right
    \item Observation (Discrete): The position of the car, the side of the green flag (-1 or 1 if the car is at the blue flag, and 0 otherwise)
    \item Reward: step reward: -0.01, reaching the green flag: 1.0, and reaching the red flag: -1.0
    \item Episode Initialization: The car is randomized such that it is not at the information location (blue flag). The goal (green flag) is always either at the leftmost or rightmost end. The red flag is on the opposite end
    \item Episode Termination: Reaching either flags or an episode lasts more than 50 timesteps
    \item World size: The distance between the red and the green flag is 50
\end{itemize}
\begin{figure}[htbp]
    \centering
    \includegraphics[width=0.8\linewidth]{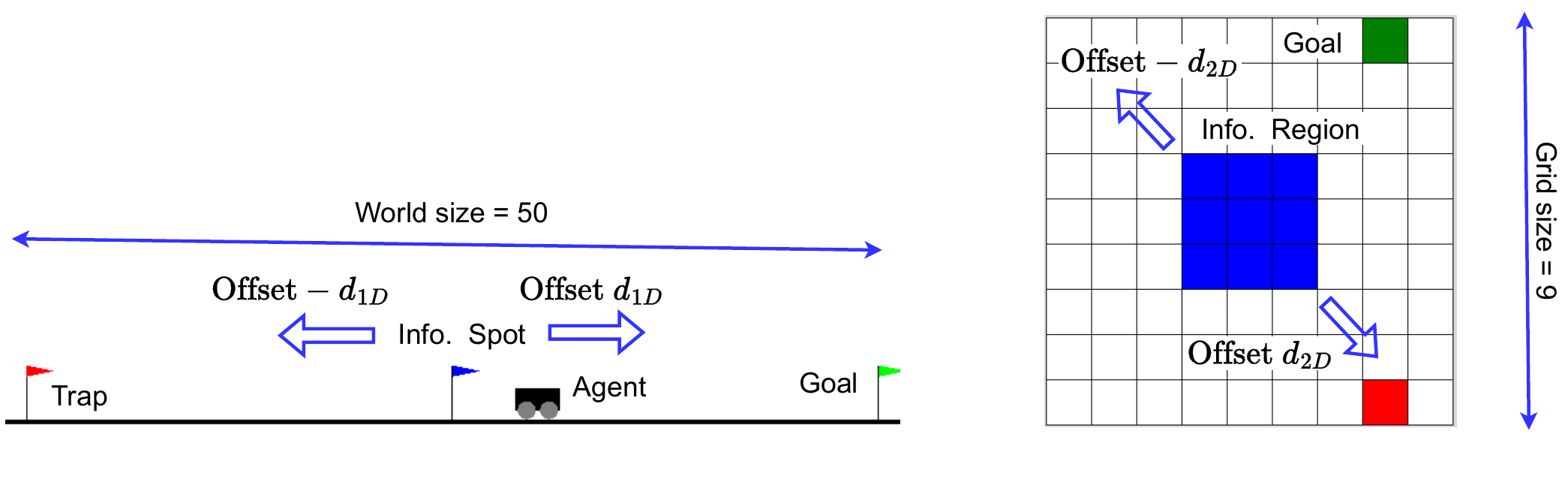}
    \caption{\texttt{CarFlag-1D} and \texttt{CarFlag-2D} domains. The information regions are not visible to the agent. These domains become asymmetric when the offsets from the information region to the world center, i.e., $d_{1D}$ and $d_{2D}$, are non-zero.}
    \label{fig:app_carflag_1d}
\end{figure}
\subsubsection{\texttt{CarFlag-2D}}

\begin{itemize}
    \item Action: Right/Left/Up/Down
    \item Observation: The observation is encoded as an $N\times N \times 2$ image, where $N$ is the grid size, the first channel encodes the car's position, and the second channel encodes the position of the green cell. The second channel is only informative when the agent is inside the information region (blue)
    \item Reward: Reaching the green cell: 1.0, otherwise 0.0
    \item Episode Initialization: The agent and the goal cell are randomized such that the minimum distance between them is at least two steps. Moreover, both the agent and the goal are not initialized inside the information region (blue)
    \item Episode Termination: Reached the goal or an episode lasts more than 50 timesteps
\end{itemize}

\subsection{Robot Manipulation Domains}
\label{app:robot_manip}
An episode is terminated for these domains when it lasts over 50 timesteps or the task is achieved. Because all robot domains share the same observation and action, we only describe them below.
\begin{figure}[htbp]
    \centering
    \includegraphics[width=0.5\linewidth]{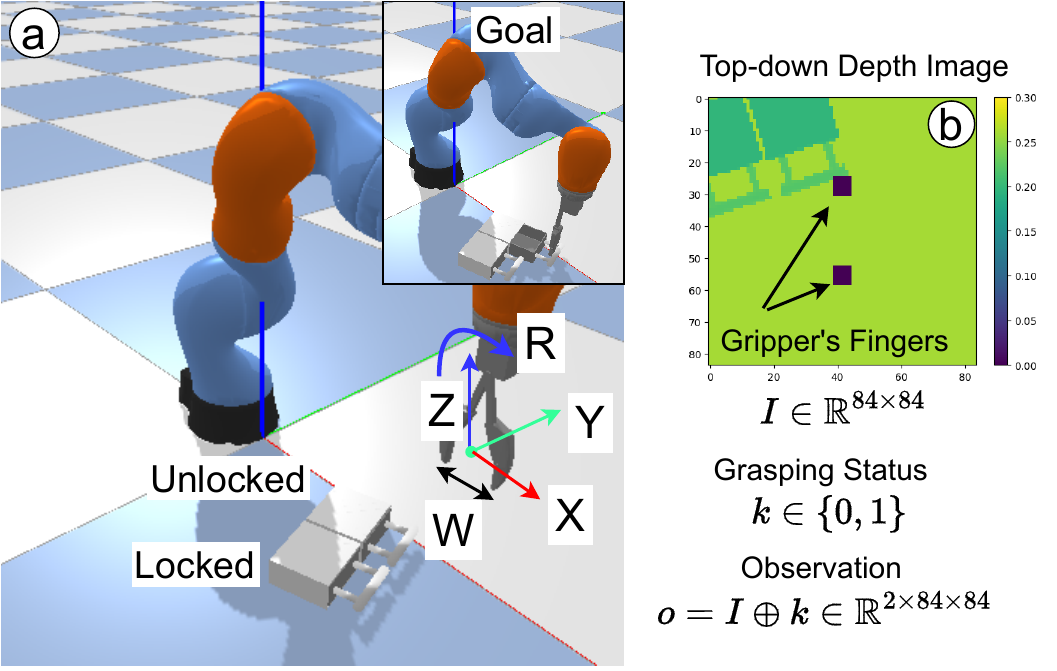}
    \caption{Visual description of \texttt{Drawer-Opening}.}
    \label{fig:intro_visual}
\end{figure}

\noindent \textbf{Action.} An action $a = (\delta_w, \delta_x, \delta_y, \delta_z, \delta_r)$, where $\delta_w \in [0, 1]$ is the absolute openness of the gripper (0: fully open, 1: fully closed), $\delta_{x,y,z} \in [-0.05, 0.05]$ are the displacements of the gripper in the X, Y, and Z axis, and $\delta_r \in [-\pi/8, \pi/8]$ is the angular rotation around the Z axis (see~\cref{fig:intro_visual}a)

\noindent \textbf{Observation.} An observation is a top-down depth image taken from a camera located at the end-effector. Specifically, an observation $o = (I, k)$, where $ I \in \mathbb{R}^{84 \times 84}$ is the depth image and $k \in \{1, 0 \}$ indicates the current holding status of the gripper. $I$ and $k$ are combined to create a unified depth observation $o\in \mathbb{R}^{2\times84\times84}$. Moreover, two fingers of the gripper are also projected on $I$ (black squares in~\cref{fig:intro_visual}b)

\noindent \textbf{Partial Observability. }These domains characterize the natural partial observability when certain physical properties of objects, e.g., whether a drawer in~\cref{fig:intro_visual}a is unlocked or not, are often unobservable using pixel observations alone

\subsubsection{\texttt{Block-Picking}}
\begin{itemize}
    \item Reward: A reward of 1.0 only when the movable block is picked and brought higher than 8cm
    \item Episode Initialization: The poses of the two blocks are randomized. The arm is initialized at a fixed pose
    \item Expert Generation: An expert (a planner with access to all object poses) randomly chooses one block to pick. If the expert picks the movable block, it will bring the block up to achieve the task. Otherwise, the expert keeps trying for several timesteps before switching to pick the movable block to achieve the task
\end{itemize}

\begin{figure}[htbp]
  \centering
  \begin{subfigure}[t]{0.2\linewidth}
    \includegraphics[width=\linewidth]{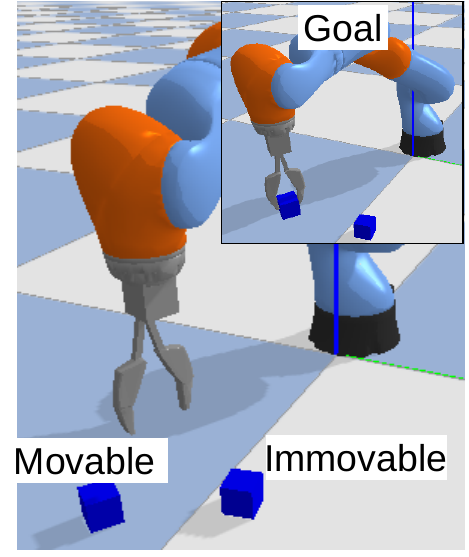}
    \caption*{\texttt{Block-Picking}} 
  \end{subfigure}
  \begin{subfigure}[t]{0.2\linewidth}
    \includegraphics[width=\linewidth]{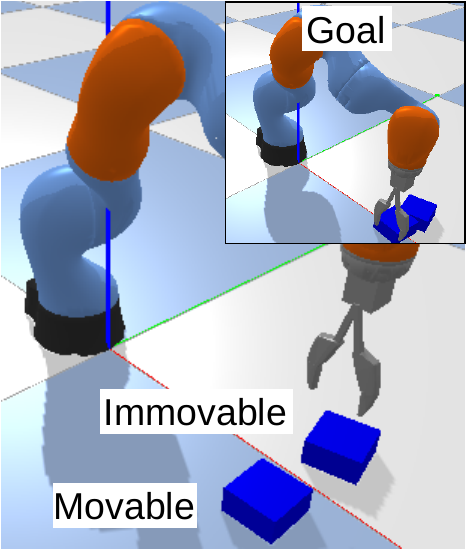}
    \caption*{\texttt{Block-Pulling}}
  \end{subfigure}
  \begin{subfigure}[t]{0.2\linewidth}
    \includegraphics[width=\linewidth]{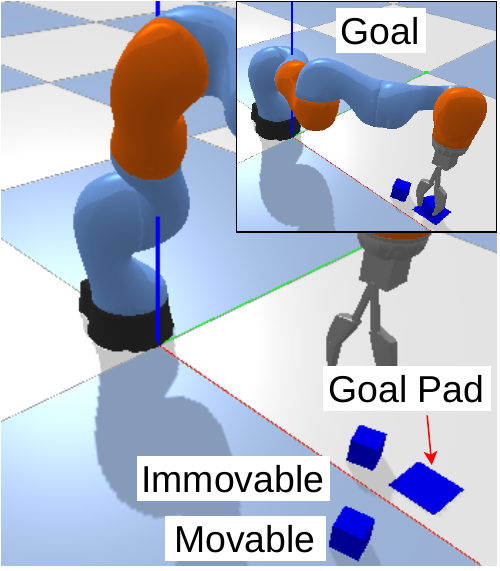}
    \caption*{\texttt{Block-Pushing}} 
  \end{subfigure}
  \begin{subfigure}[t]{0.18\linewidth}
    \includegraphics[width=\linewidth]{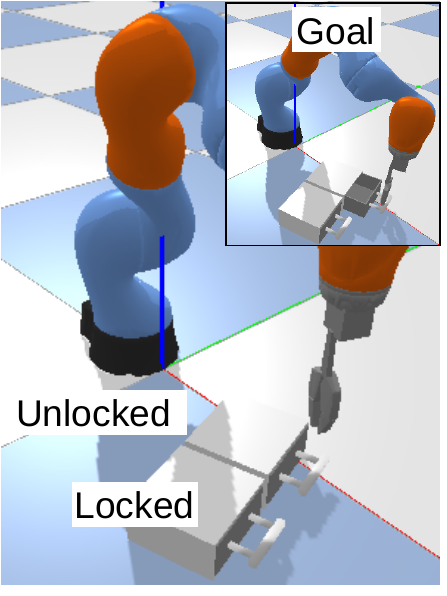}
    \caption*{\texttt{Drawer-Opening}}
  \end{subfigure}
  \caption{Robot manipulation domains.}
    \label{fig:app_blockpicking}
\end{figure}
\subsubsection{\texttt{Block-Pulling}}
\begin{itemize}
    \item Reward: A reward of 1.0 only when two blocks are in contact
    \item Expert Generation: An expert randomly chooses one block to pull towards the other block. If the block is pullable, it will be pulled towards the other block to achieve the task. Otherwise, the expert keeps trying for a while before pulling the other block
\end{itemize}

\subsubsection{\texttt{Block-Pushing}}

\begin{itemize}
    \item Reward: A reward of 1.0 only when the pushable block is within 5cm from the center of the goal pad. The agent additionally receives a penalty of 0.1 per timestep if it changes the height of the movable block by 5mm to prevent picking the block instead of pushing it
    \item Episode Initialization: The poses of the two blocks and the goal pad are randomly initialized
    \item Expert Generation: An expert randomly chose one block to push towards the goal pad. If the block is pushable, it will continue pushing until it reaches the goal pad. Otherwise, the expert keeps trying for several timesteps before doing the same thing with the other (pushable) block
\end{itemize}

\subsubsection{\texttt{Drawer-Opening}}
\begin{itemize}
    \item Reward: A reward of 1.0 only when the unlocked drawer is opened more than 5cm
    \item Episode Initialization: Two drawers are randomly placed next to each other with the same heading angle
    \item Expert Generation: An expert randomly chooses one drawer to open. If it chooses the unlocked drawer, it will then open the drawer to achieve the task. Otherwise, the expert keeps opening the unlocked drawer several timesteps before opening the other drawer
\end{itemize}

\clearpage
\section{Implementation Details}
\label{app:implementation_details}

\subsection{Network Structure of Equivariant Recurrent A2C (Equi-RA2C)}
\cref{fig:app_ra2c} shows the specific architecture of Equi-RA2C used in \texttt{CarFlag} domains. Because the actions can be inferred from the observations in these domains, we do not include the feature extractor for the previous actions. We also omit the skip-connections. The input representation is some representation of the observation $\rho_o$, depending on the domains (see below).
\begin{figure}[htbp]
    \centering
    \includegraphics[width=0.7\linewidth]{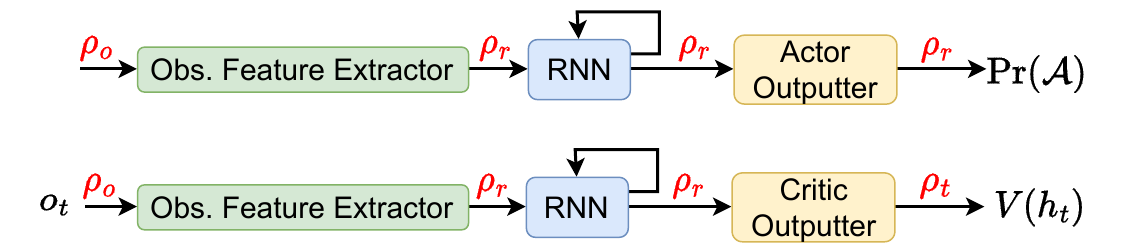}
    \caption{The architecture of Equi-RA2C used in \texttt{CarFlag-1D} and \texttt{CarFlag-2D}.}
    \label{fig:app_ra2c}
\end{figure}

\cref{fig:app_ra2c_1d} shows the details of Equi-RA2C used in \texttt{CarFlag-1D} for the \texttt{flip2dOnR2} group in the \texttt{escnn}~\footnote{https://github.com/QUVA-Lab/escnn}~\citep{e2cnn, cesa2022a} library. Notice that the input $x_t$ for the LSTM cell using the \emph{irreducible} representation of the \texttt{flip2dOnR2} group denoted as $\rho_{\text{irr}}$. For \texttt{CarFlag-1D}, using this representation in the input would negate the signs of every component in $x_t$, i.e., flipping the positions of the car, the sides of the green flag, and the previous actions in the history. Because the observation in this domain is feature-based, we remove the observation feature extractor and directly feed the observation to the equivariant LSTM.
\begin{figure}[htbp]
    \centering
    \includegraphics[width=1.0\linewidth]{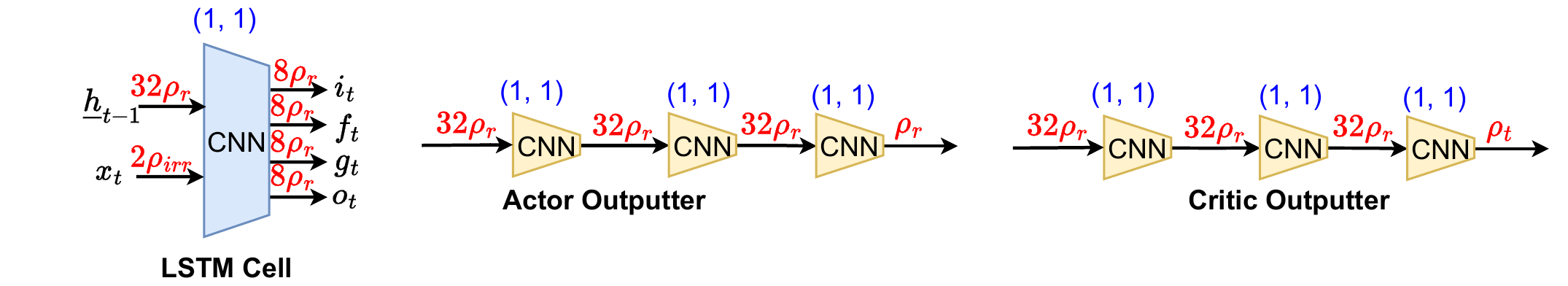}
    \caption{Details of Equi-RA2C used in \texttt{CarFlag-1D} for the \texttt{flip2dOnR2} group. Numbers inside brackets (blue - on top) denote the value of kernel sizes and strides used for the CNN modules on the bottom. The numbers next to the representations, e.g., 32$\rho_r$, denote the number of feature fields.}
    \label{fig:app_ra2c_1d}
\end{figure}

\cref{fig:app_ra2c_2d} shows the details of Equi-RA2C used in \texttt{CarFlag-2D} for the $C_4$ group.
\begin{figure}[htbp]
    \centering
    \includegraphics[width=0.9\linewidth]{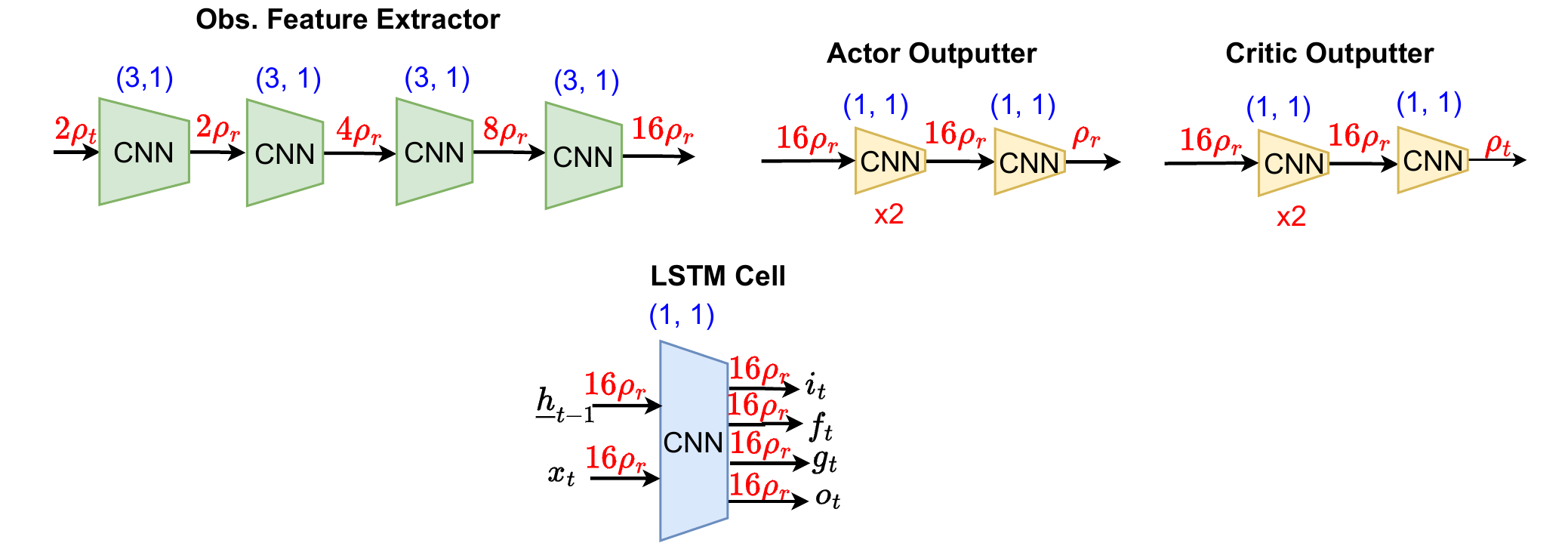}
    \caption{The details of Equi-RA2C used in \texttt{CarFlag-2D} for the $C_4$ group.}
    \label{fig:app_ra2c_2d}
\end{figure}
\subsection{Network Structure of Equivariant Recurrent SAC (Equi-RSAC) with $C_4$ Group}
\label{sect:explain_mixed}

\cref{fig:equi_rsac_details} shows the details of Equi-RSAC used in the robot manipulation domains with the $C_4$ group. The input representation is \emph{mixed} for the action feature extractor because the action input has components that transform differently under a rotation. Specifically, given an action $a = (\delta_w, \delta_x, \delta_y, \delta_z, \delta_r)$, the trivial representation $\rho_t$ is chosen for the $\delta_w, \delta_z, \delta_r$ components (which should be unchanged under the rotation). In contrast, the standard representation $\rho_s$ is chosen for the lateral components $(\delta_x, \delta_y)$, which should rotate. For the same reason, for the actor outputter, $\rho_\mu$ is mixed, i.e., the trivial representations $\rho_t$ are used for the $w, z, r$ components, and $\rho_s$ is used for the $x, y$ components.
\begin{figure}[htbp]
    \centering
    \includegraphics[width=0.9\linewidth]{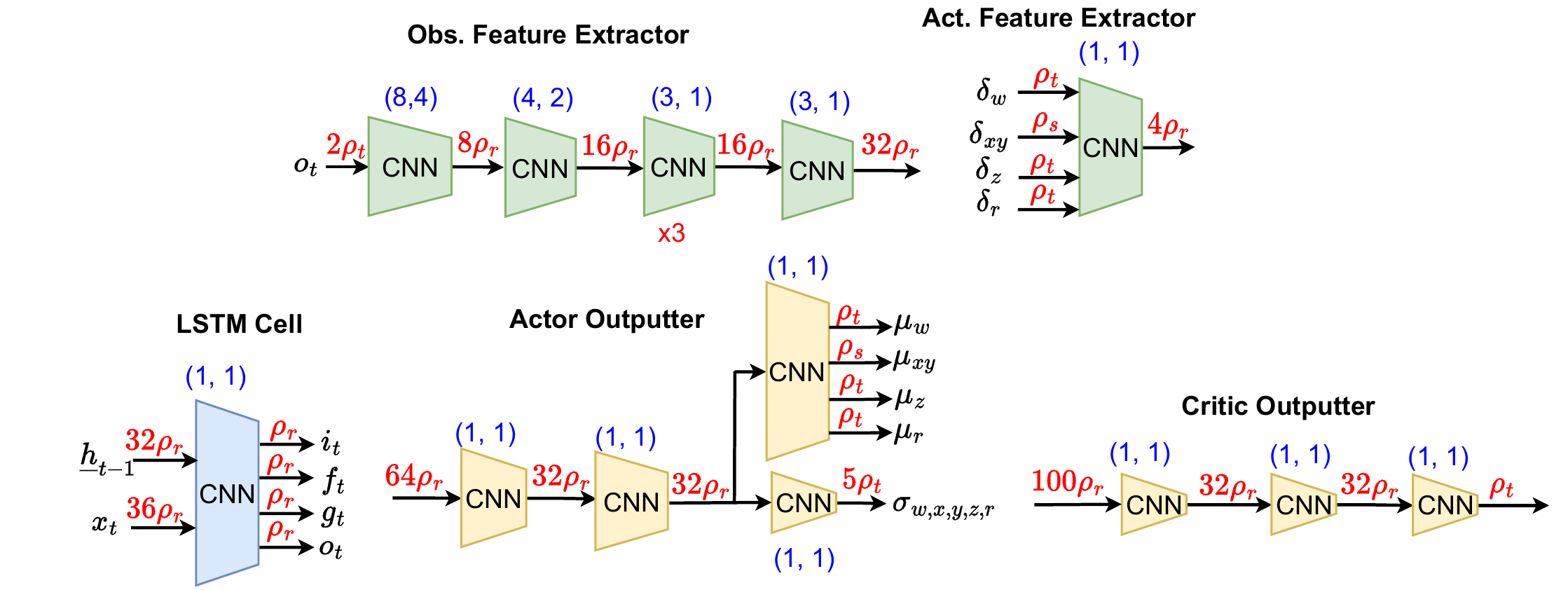}
    \caption{Details of Equi-RSAC with the robot manipulation domains and the $C_4$ group.}
    \label{fig:equi_rsac_details}
\end{figure}

\subsection{Implementation Using The ESCNN Library}
Given the definition of each equivariant component above, we can easily implement it with \texttt{escnn}. For instance, the following PyTorch~\citep{paszke2017automatic} code defines the observation feature extractor in~\cref{fig:app_ra2c_2d}a with ReLU as a non-linearity component:

\begin{python}
import escnn.nn as enn

# Define group C4
s = escnn.gspaces.rot2dOnR2(4)

# Define in/out representations
repr_i  = enn.FieldType(s,  2*[s.trivial_repr])
repr_m0 = enn.FieldType(s,  2*[s.regular_repr])
repr_m1 = enn.FieldType(s,  4*[s.regular_repr])
repr_m2 = enn.FieldType(s,  8*[s.regular_repr])
repr_o  = enn.FieldType(s, 16*[s.regular_repr])

obs_feature_extractor = enn.SequentialModule(
          enn.R2Conv(repr_i, repr_m0, 3, 1),
          enn.ReLU(repr_m0),
          enn.R2Conv(repr_m0, repr_m1, 3, 1),
          enn.ReLU(repr_m1),
          enn.R2Conv(repr_m1, repr_m2, 3, 1),
          enn.ReLU(repr_m2),
          enn.R2Conv(repr_m2, repr_o, 3, 1),
          enn.ReLU(repr_o),
)
\end{python}
Implementing the mixed representation is also straightforward by summing different field types. In order to create the actor and the critic, we simply chain components by using the \texttt{SequentialModule} as in native PyTorch.

\subsection{Training Details}
We implement using PyTorch. The batch size for all agents is 32 (episodes). The replay buffer has a capacity of 100,000 transitions. We use the Adam optimizer~\citep{kingma2014adam} with a learning rate of 3e-4 for actors and critics and 1e-3 for optimizing $\alpha$ for SAC-based agents. The target entropy $\bar{H}$ for SAC-based agents is -dim$(\mathcal{A})$ followed the common practice, and $\alpha$ is initialized at 0.1. After prepopulating the replay buffer with 80 expert episodes, the buffer is filled with 20 episodes with random actions. We use the same 1:1 environment/gradient step ratio for all agents.

\subsection{Implementing Equivariant LSTM}
We implement the equivariant LSTM~\citep{hochreiter1997long} based on a public code of ConvLSTM~\citep{shi2015convolutional} at \url{https://github.com/Hzzone/Precipitation-Nowcasting} as the authors did not release the official code.

\clearpage
\section{Baseline Details}
\label{app:baseline_details}

\paragraph{RA2C~\citep{pytorchrl}} We modified the code at \url{https://github.com/ikostrikov/pytorch-a2c-ppo-acktr-gail}. We used 16 environments in parallel and used recurrent policies. Other hyper-parameters are kept at default.
\paragraph{DPFRL~\citep{ma2020discriminative}} We used the authors' code at \url{https://github.com/Yusufma03/DPFRL}. We used 30 particles, MGF particle aggregation type, and the hidden dimension is 128.
\paragraph{RAD~\citep{laskin2020reinforcement}} We collected depth images of size 90x90 to perform random cropping to reduce the size to 84x84. We perform the same type of random cropping for every depth image within an episode.
\paragraph{DrQ~\citep{kostrikov2020image}} We used random shift of $\pm4$ pixels as suggested by the original work. The same type of shifting is used for every depth image within a sequence. We also followed the authors' suggestions when using the numbers of augmentations for calculating the Q-targets, and the Q-values are $K=2$ and $M=2$, respectively.
\paragraph{SLAC~\citep{lee2020stochastic}} We used a Pytorch implementation at \url{https://github.com/toshikwa/slac.pytorch}, which has been benchmarked against the performance reported in the original paper. We pre-train the latent variable model for 2k steps before iterating between data collection, model update, and evaluation. We also pre-fill the replay buffer with the same number of expert and random episodes before training and use four extra augmented episodes for each episode during training to ensure a fair comparison. The sequence length is extended from 8 (originally) to 50 (maximum episode length). We varied the sequence length for better performance, but the performance did not improve much. For any episode shorter than 50 steps, we zero-pad dummy transitions \emph{in front}.
\paragraph{DreamerV2~\citep{hafner2020mastering}} We used the official code at \url{https://github.com/danijar/dreamerv2}. For \texttt{CarFlag} domains, we mainly keep the default hyper-parameters (suggested by the authors). In \texttt{CarFlag-2D}, the observation image is extended to have the size of 64$\times$64$\times$3 by zero-padding around the original image and is added with a dummy channel (all zero).

\paragraph{DreamerV3~\citep{hafner2023mastering}} We used the official code at \url{https://github.com/danijar/dreamerv3} and performed similar steps like in the case of DreamerV2. We used the \emph{small} world models with about 18M trainable parameters (predefined in the repo's configuration file) for our \texttt{CarFlag} domains.

\clearpage

\section{Visualization of SLAC Reconstructed Images}
\label{app:slac_viz}

\cref{fig:app_slac_reconstruct} shows the comparison between the depth images produced by the trained latent model of SLAC~\citep{lee2020stochastic} (top row) and the ground-truth ones (bottom row) in \texttt{Block-Pulling} after 40k training steps. It can be seen that small squares representing the gripper have been reconstructed quite well, but the model fails to reconstruct the two blocks representing the gripper's position in the scene.

\begin{figure}[htbp]
    \centering
    \includegraphics[width=0.9\linewidth]{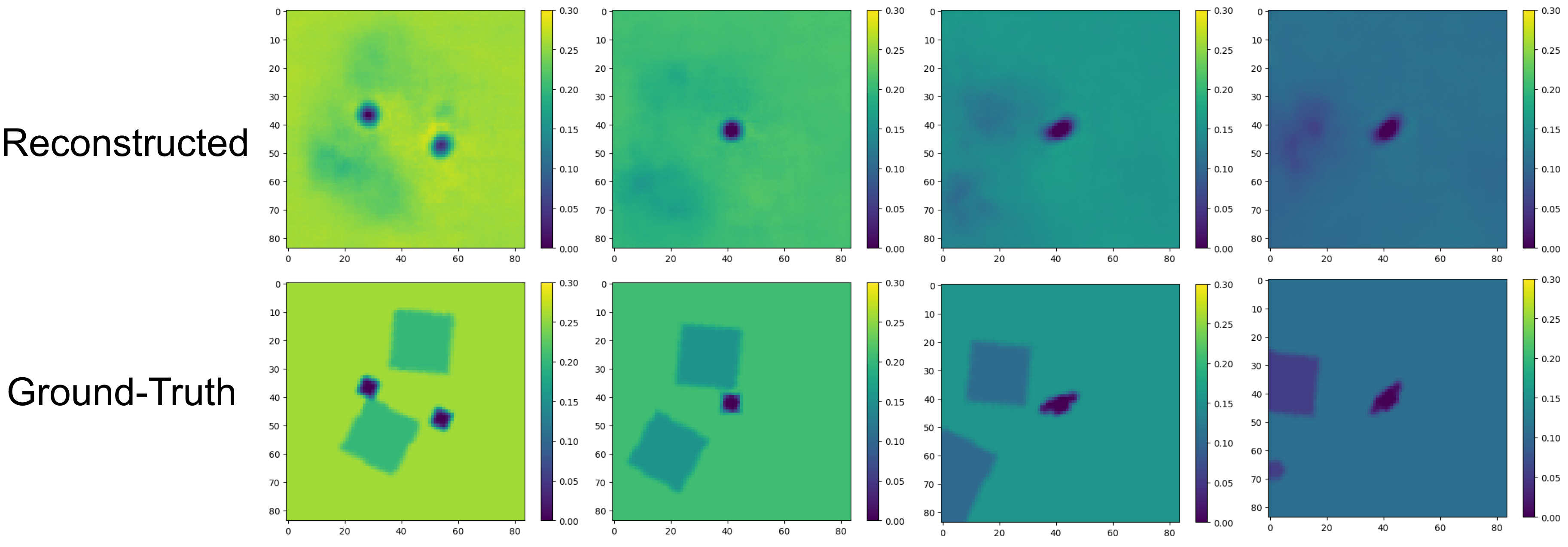}
    \caption{Images reconstructed by the latent model of SLAC~\citep{lee2020stochastic} in \texttt{Block-Pulling}: reconstructed (top row), ground-truth (bottom row).}
    \label{fig:app_slac_reconstruct}
\end{figure}

\clearpage
\section{Visualization of Data Augmentations}
\label{app:viz_data_aug}
We show visualizations of different ways for augmenting the observations within a training sequence in \texttt{Drawer-Opening}: random rotation (\cref{fig:app_rot_aug}), random crop (\cref{fig:app_rad_aug}), and random shift (\cref{fig:app_drq_aug}). Note that the same operation (rotation/crop/shift) is applied similarly to every observation in an episode. For each training episode, we perform this augmentation four times to generate four auxiliary episodes. 
\begin{figure}[htbp]
    \centering
    \includegraphics[width=0.9\linewidth]{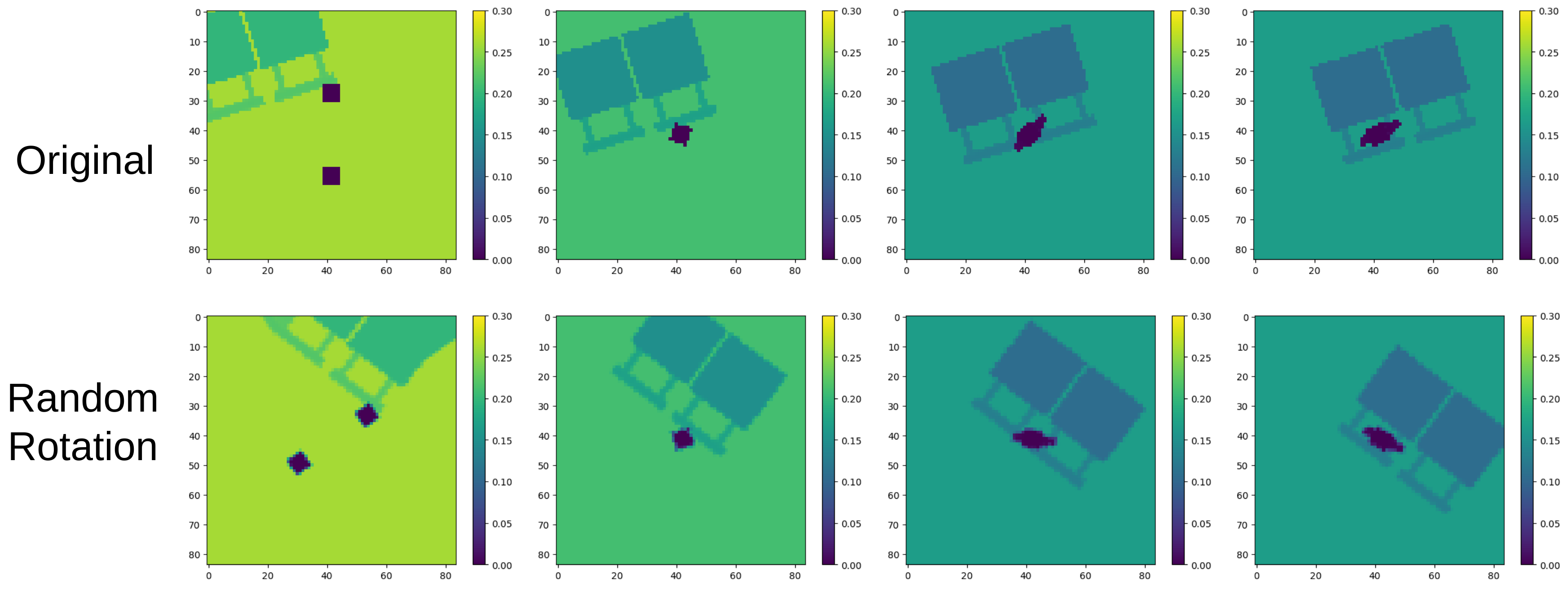}
    \caption{Visualization of randomly rotated augmentations in \texttt{Drawer-Opening}: original observations (top row), randomly rotated observations (bottom row).}
    \label{fig:app_rot_aug}
\end{figure}

\begin{figure}[htbp]
    \centering
    \includegraphics[width=0.9\linewidth]{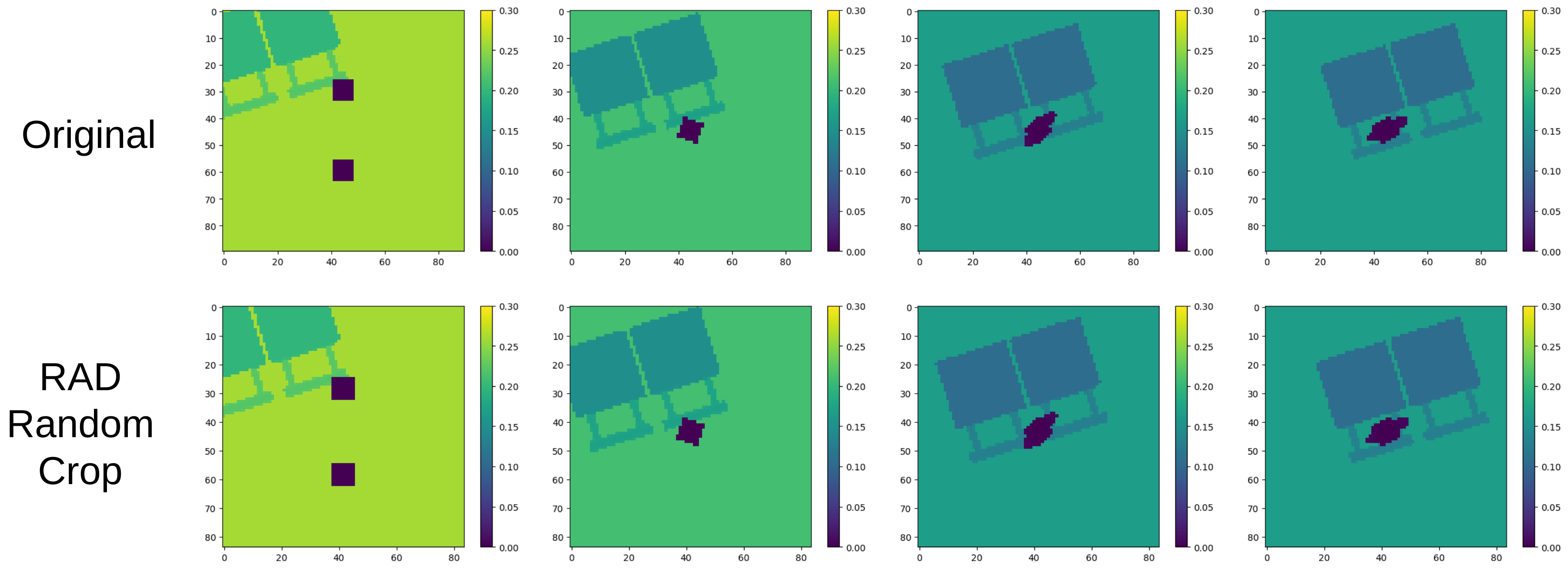}
    \caption{Visualization of randomly cropped augmentation for RAD~\citep{laskin2020reinforcement} in \texttt{Drawer-Opening}: original observations (top row), randomly cropped observations (bottom row).}
    \label{fig:app_rad_aug}
\end{figure}

\begin{figure}[htbp]
    \centering
    \includegraphics[width=0.9\linewidth]{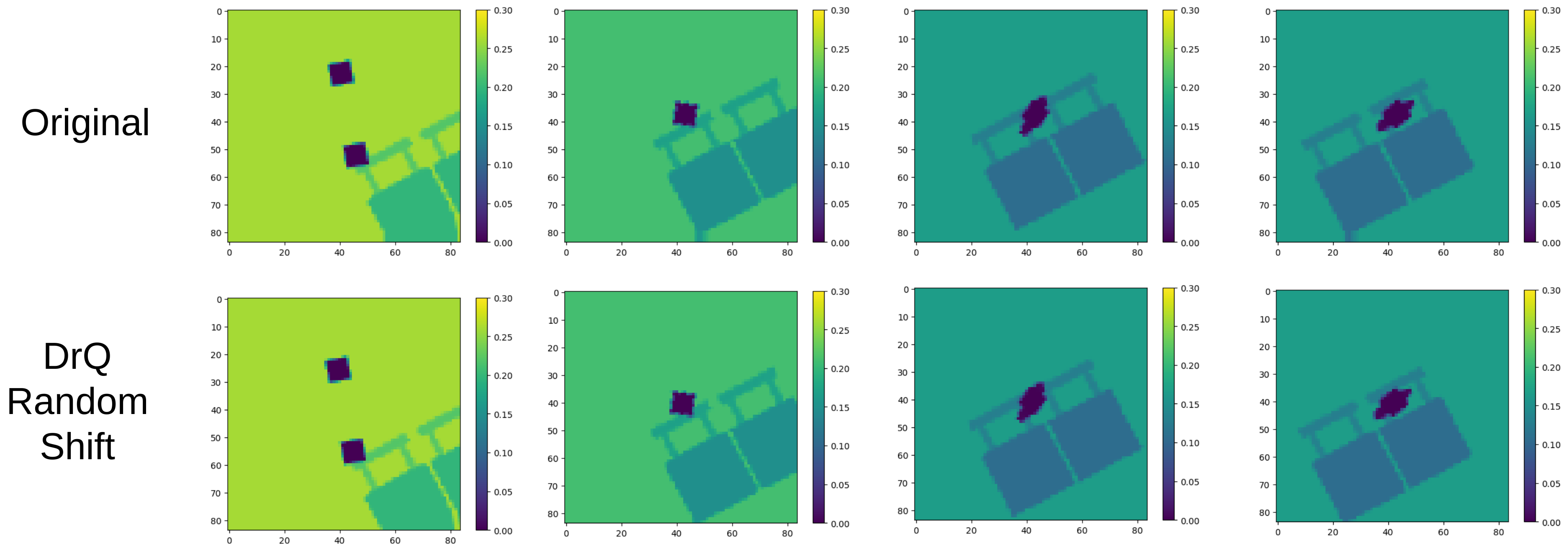}
    \caption{Visualization of randomly shifted augmentation for DrQ~\citep{kostrikov2020image} in \texttt{Drawer-Opening}: original observations (top row), randomly shifted observations (bottom row).}
    \label{fig:app_drq_aug}
\end{figure}

\clearpage
\section{Ablation Studies}
\label{app:ablation_studies}


\subsection{Equivariant Actor or Critic Only}
In~\cref{fig:app_actor_critic_type}, we additionally show the learning performance when only either actor or critic is equivariant in \texttt{Block-Pushing} and \texttt{Drawer-Opening}. From the figure, having an equivariant critic (purple) is more beneficial than having an equivariant actor (blue). However, having both being equivariant (green) yields the best performance.

\begin{figure}[htbp]
 \centering
  \begin{subfigure}[t]{0.4\linewidth}
    \includegraphics[width=\linewidth]{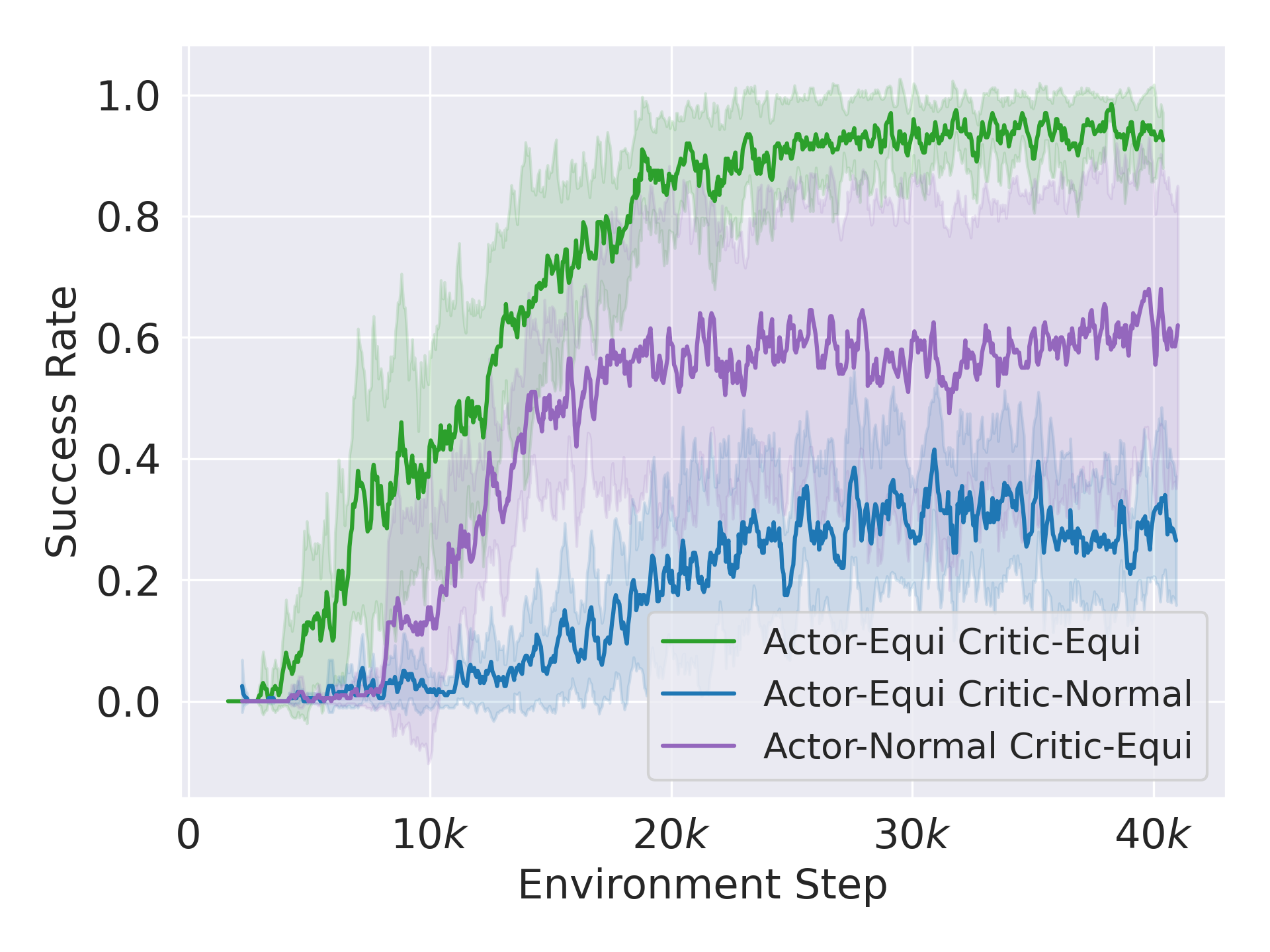}
    \caption{\texttt{Block-Picking}}
  \end{subfigure}
  \begin{subfigure}[t]{0.4\linewidth}
    \includegraphics[width=\linewidth]{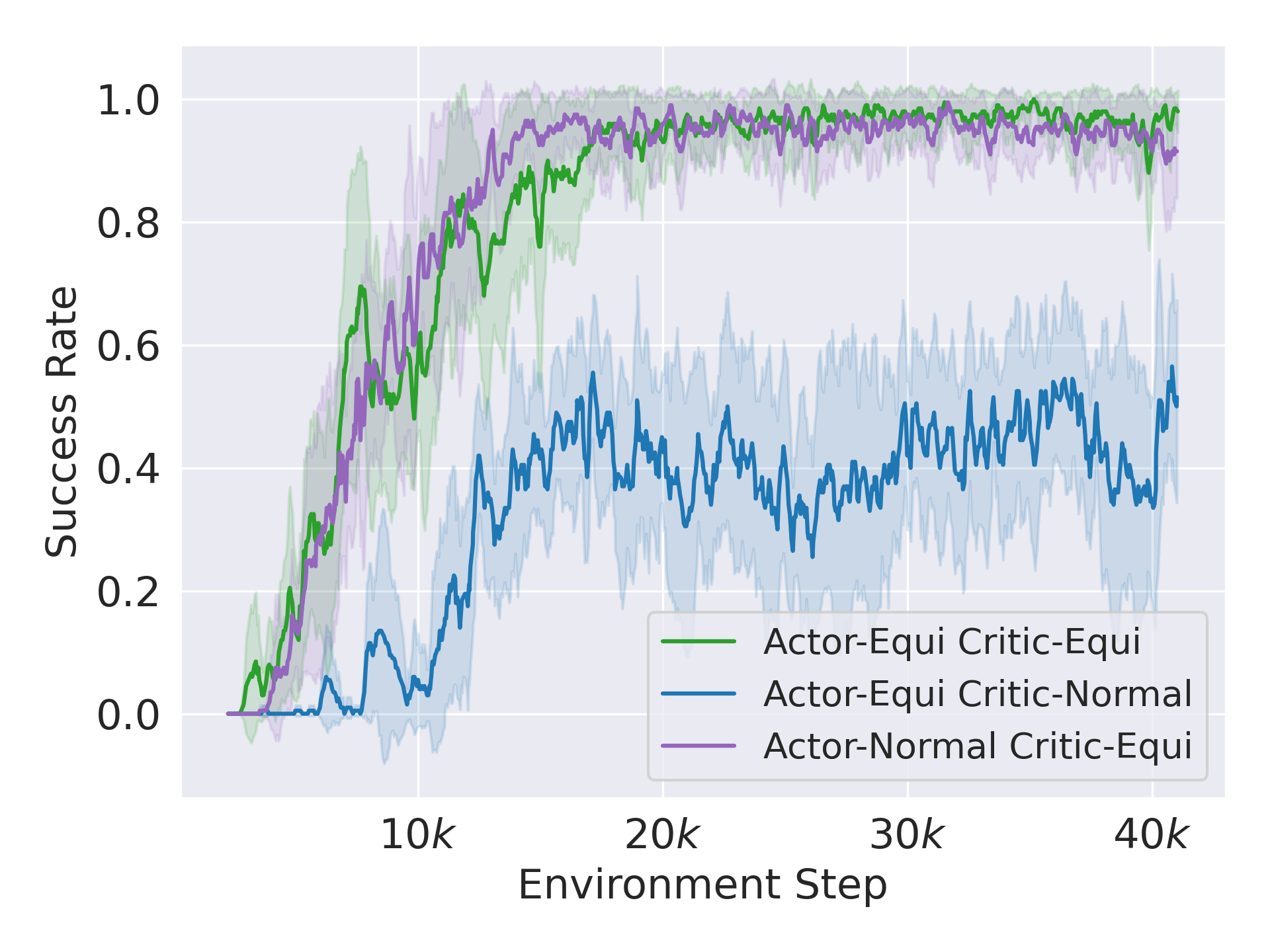}
    \caption{\texttt{Block-Pulling}} 
  \end{subfigure}
  \begin{subfigure}[t]{0.4\linewidth}
    \includegraphics[width=\linewidth]{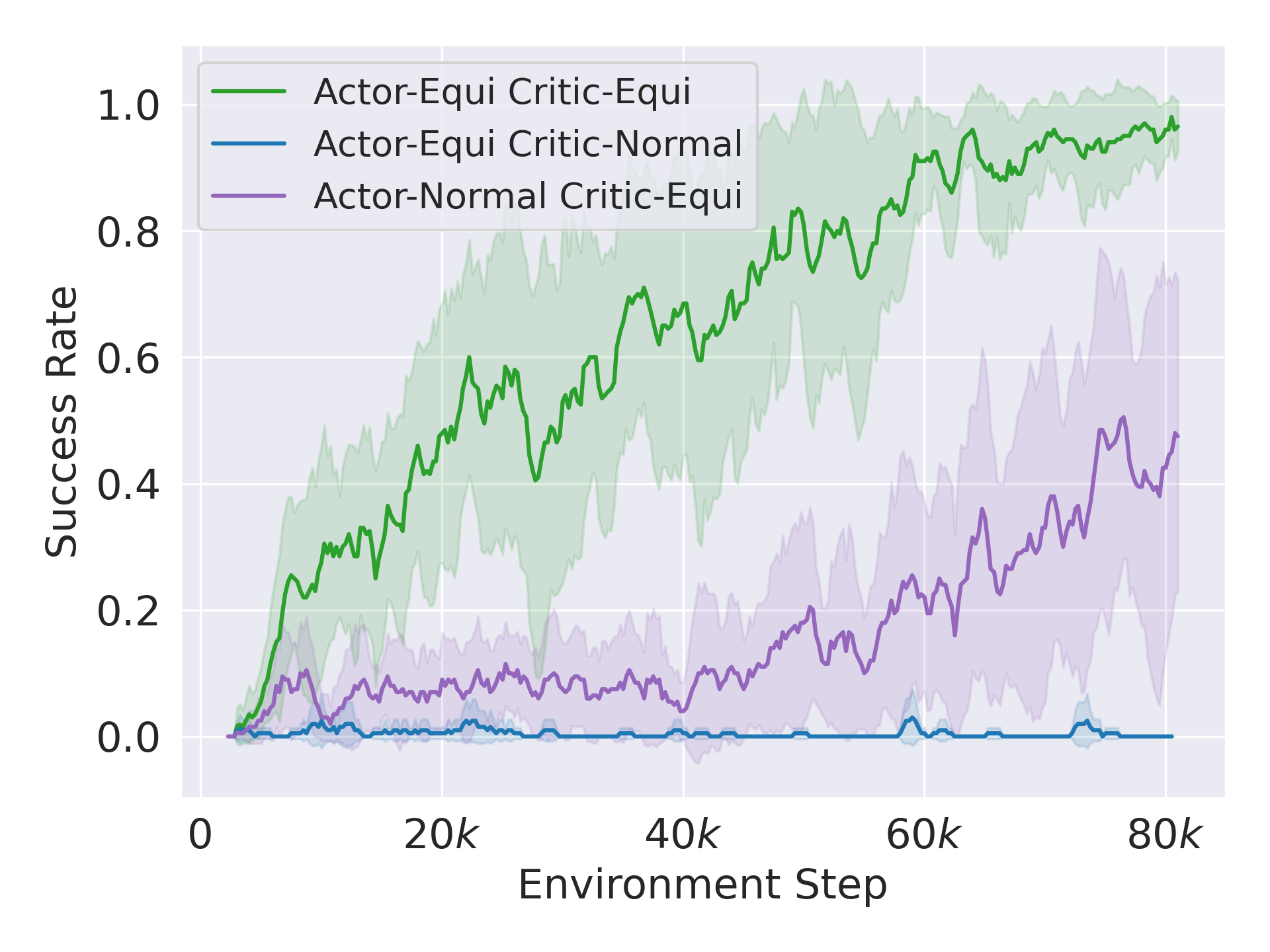}
    \caption{\texttt{Block-Pushing}}
  \end{subfigure}
  \begin{subfigure}[t]{0.4\linewidth}
    \includegraphics[width=\linewidth]{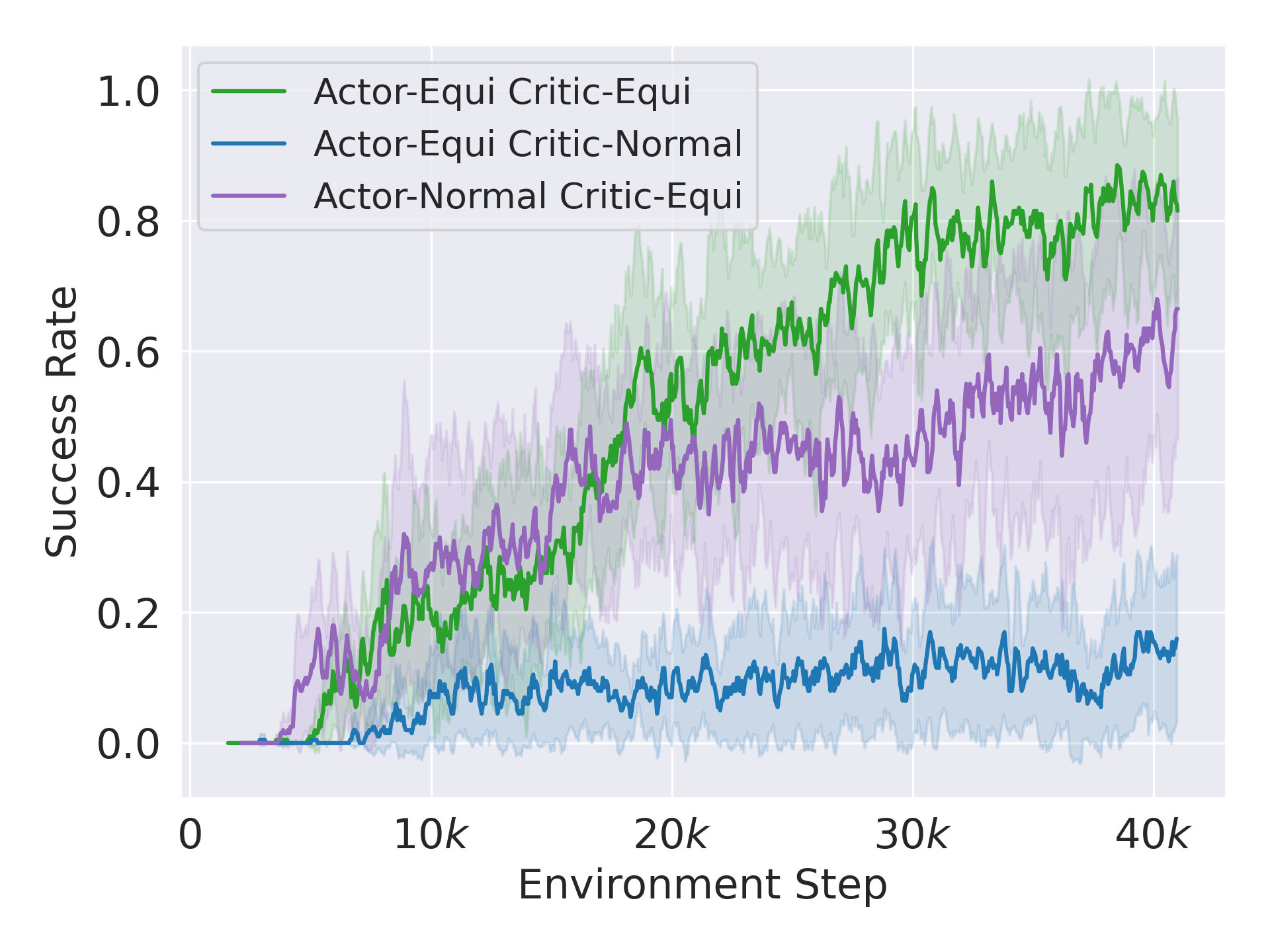}
    \caption{\texttt{Drawer-Opening}}
  \end{subfigure}
  \caption{Comparing the effect of only using equivariant actor or critic.}
  \label{fig:app_actor_critic_type}
\end{figure}

\clearpage
\subsection{Different Symmetry Groups}
\cref{fig:app_abl_groups} shows the performance when the $C_4$ and $C_8$ symmetry groups in the robot manipulation domains. Using $C_4$ is much better than using $C_8$ in \texttt{Block-Pushing}, but the two groups perform similarly in the remaining domains. Furthermore, it is possible to use other group symmetries that extend $C_n$ with reflection, such as the dihedral groups $D_4$ or $D_8$.

\begin{figure}[htbp]
 \centering
   \begin{subfigure}[t]{0.4\linewidth}
    \includegraphics[width=\linewidth]{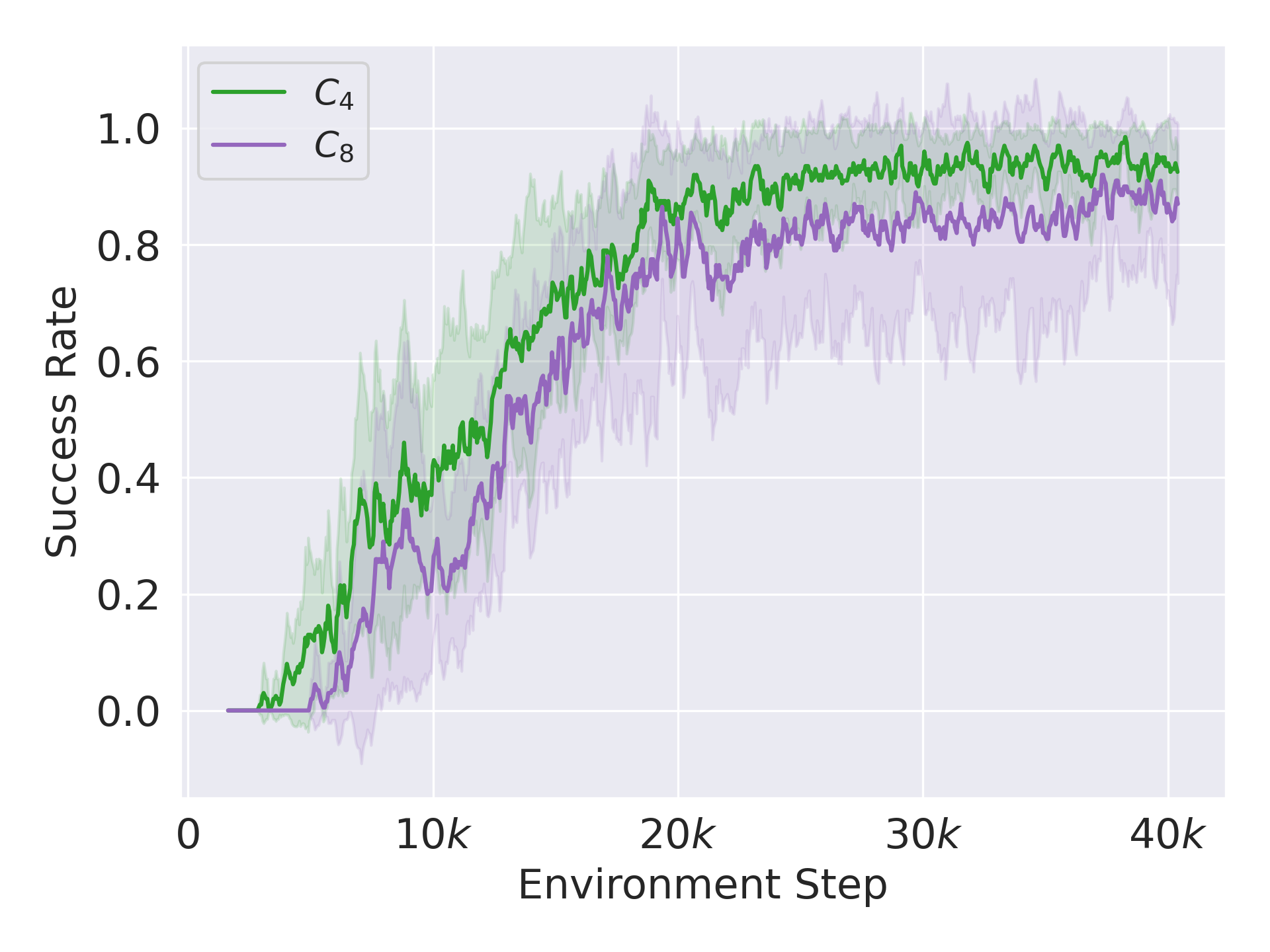}
    \caption{\texttt{Block-Picking}}
  \end{subfigure}
  \begin{subfigure}[t]{0.4\linewidth}
    \includegraphics[width=\linewidth]{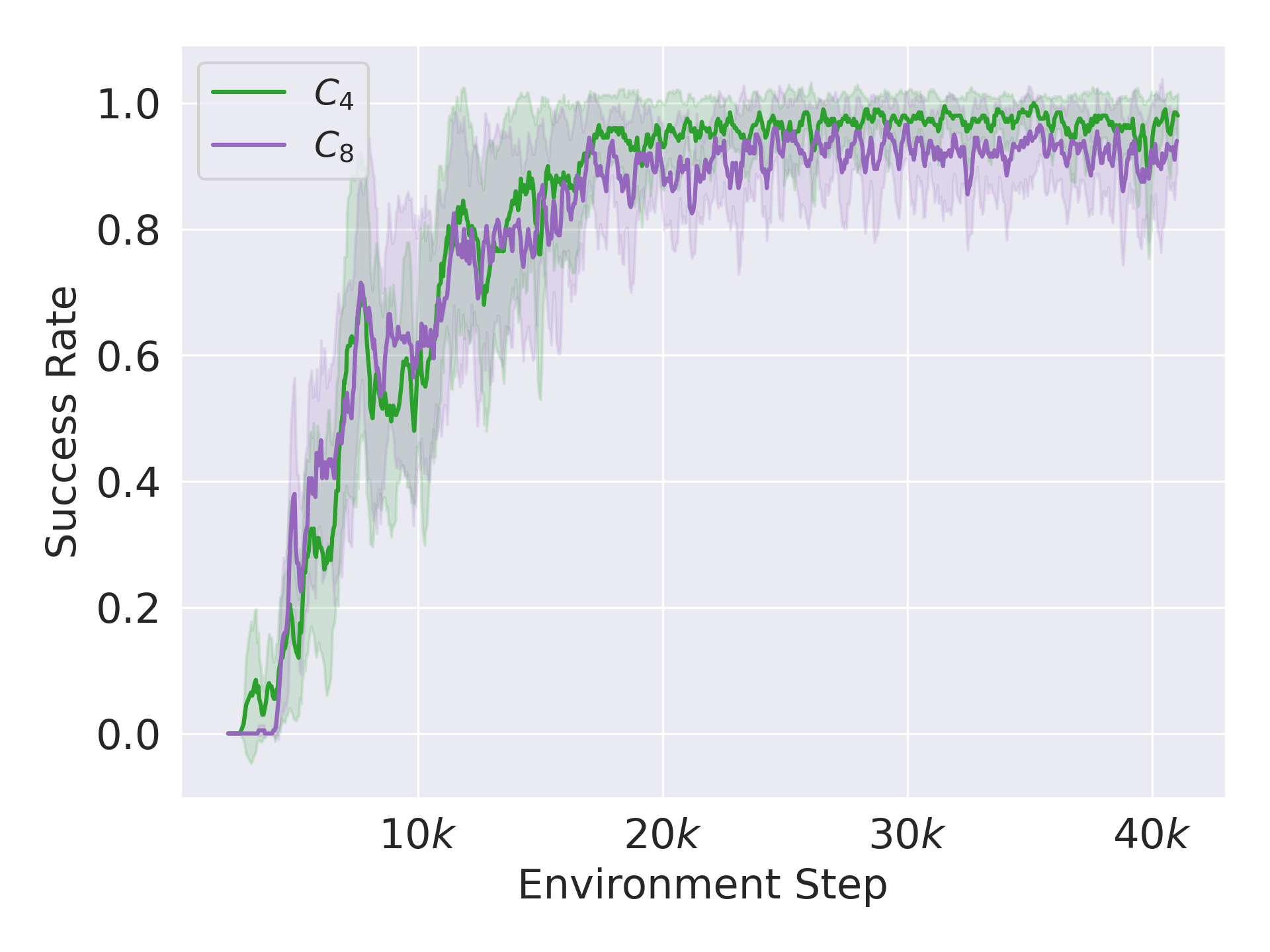}
    \caption{\texttt{Block-Pulling}}
  \end{subfigure}
  \begin{subfigure}[t]{0.4\linewidth}
    \includegraphics[width=\linewidth]{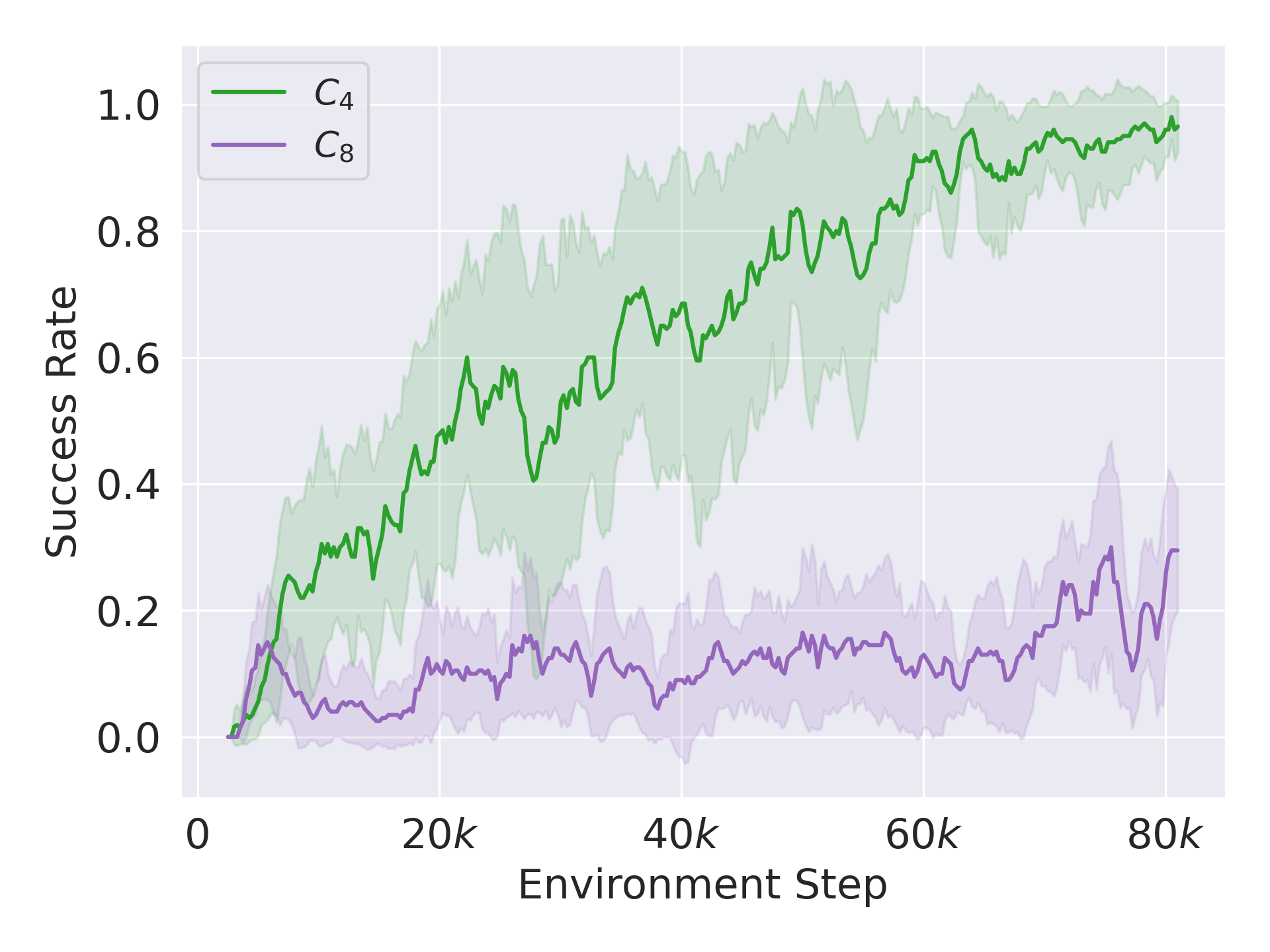}
    \caption{\texttt{Block-Pushing}}
  \end{subfigure}
  \begin{subfigure}[t]{0.4\linewidth}
    \includegraphics[width=\linewidth]{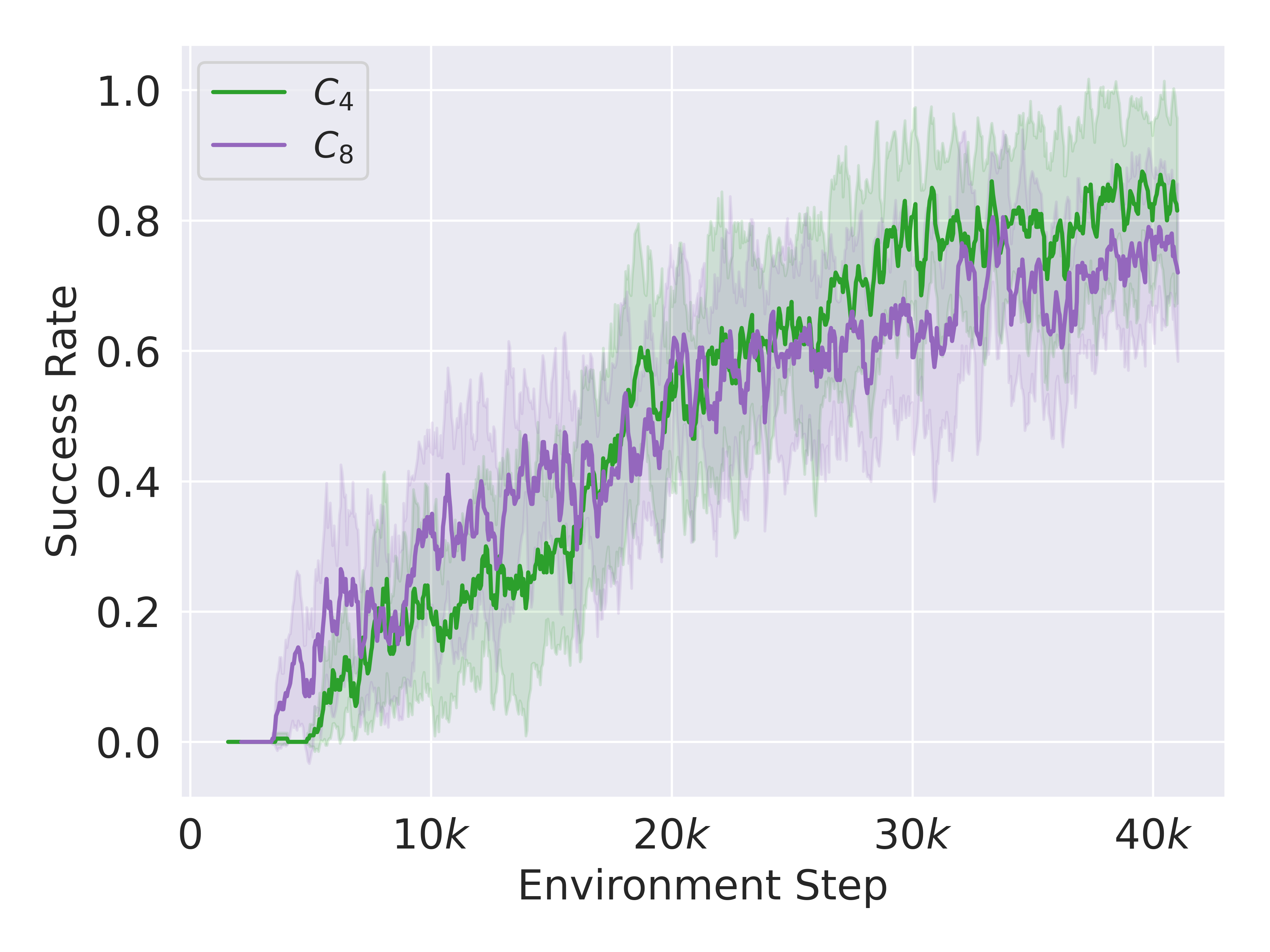}
    \caption{\texttt{Drawer-Opening}}
  \end{subfigure}
  \caption{Comparing the effect of using symmetry groups $C_4$ and $C_8$.}
  \label{fig:app_abl_groups}
\end{figure}

\clearpage
\subsection{Randomly Initialized Cell and Hidden States of Equivariant LSTM}
\cref{fig:app_abl_random_init_lstm} shows the performance when the equivariant LSTM is initialized with random instead of zero cell and hidden states. Random initialization results in a worse performance because the equivariance of the actor and the critic is broken. However, our method is generally robust to this change when the performance is still better than the baselines.

\begin{figure}[htbp]
 \centering
   \begin{subfigure}[t]{0.4\linewidth}
    \includegraphics[width=\linewidth]{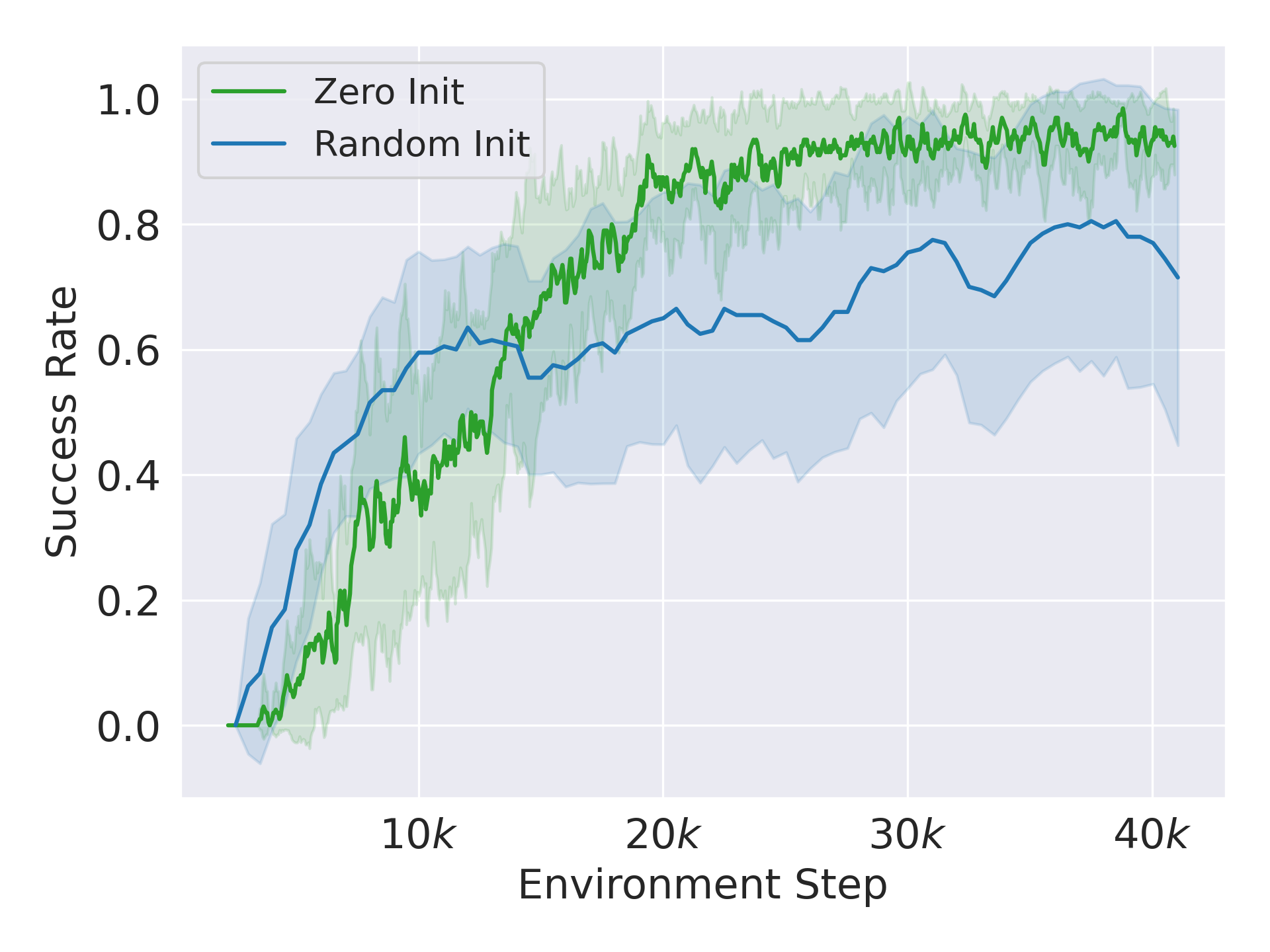}
    \caption{\texttt{Block-Picking}}
  \end{subfigure}
  \begin{subfigure}[t]{0.4\linewidth}
    \includegraphics[width=\linewidth]{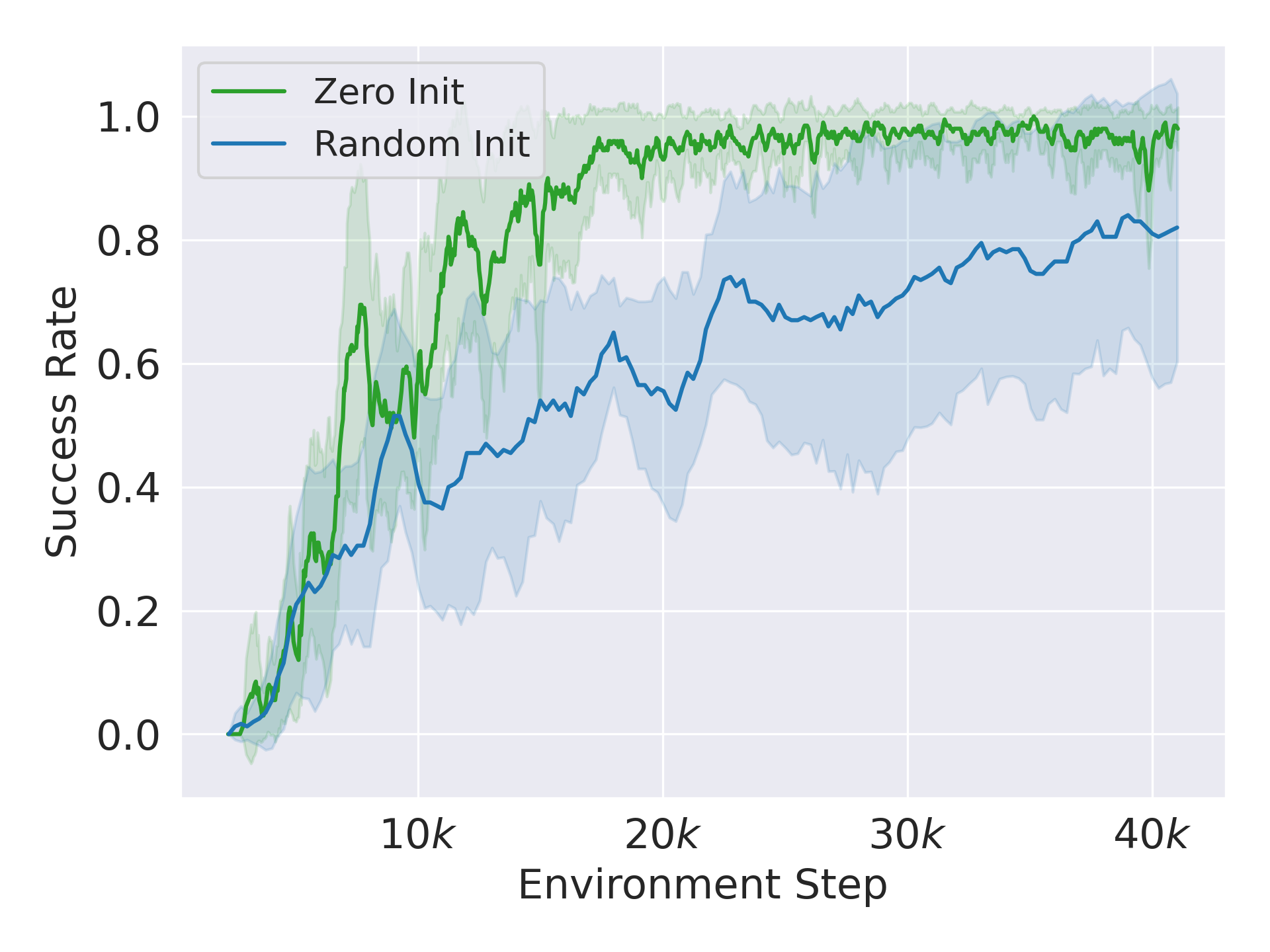}
    \caption{\texttt{Block-Pulling}}
  \end{subfigure}
  \begin{subfigure}[t]{0.4\linewidth}
    \includegraphics[width=\linewidth]{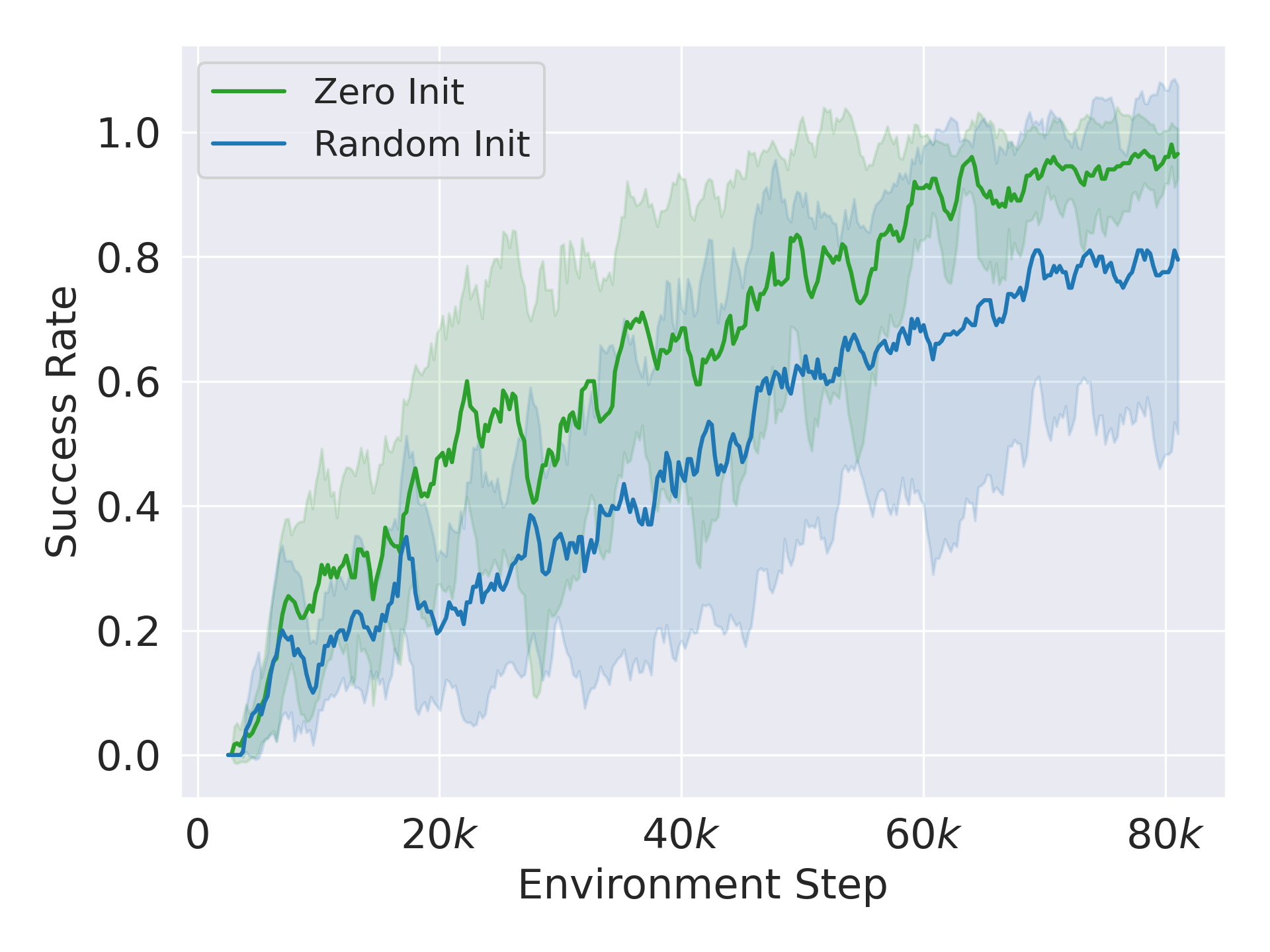}
    \caption{\texttt{Block-Pushing}}
  \end{subfigure}
  \begin{subfigure}[t]{0.4\linewidth}
    \includegraphics[width=\linewidth]{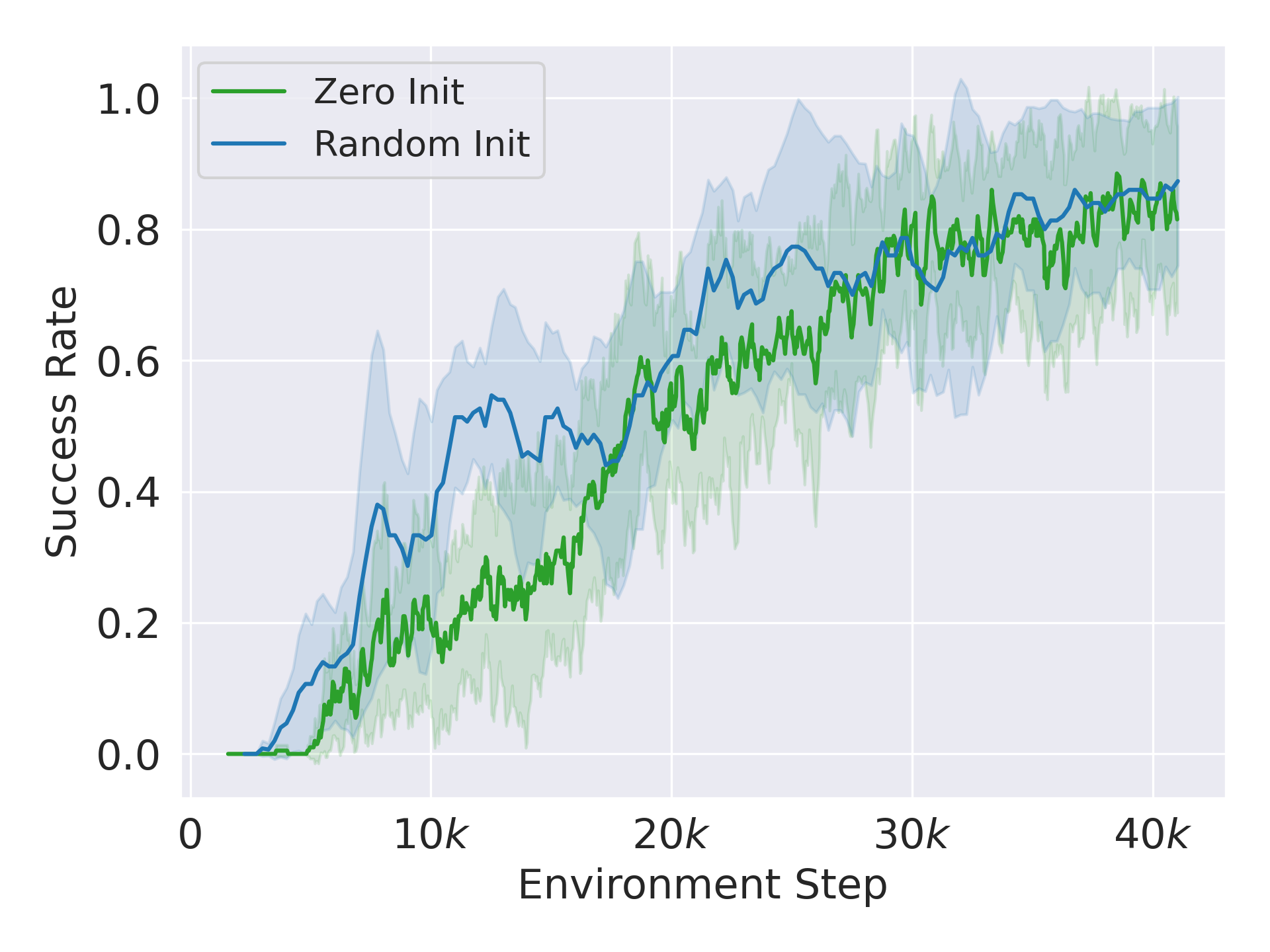}
    \caption{\texttt{Drawer-Opening}}
  \end{subfigure}
  \caption{Comparing the performance when initializing the cell and hidden states of the equivariant LSTM with zero and random values. Random initialization results in a worse performance because the actor's and the critic's equivariance is broken.}
  \label{fig:app_abl_random_init_lstm}
\end{figure}

\clearpage
\section{Additional Experimental Results}
\label{app:extra_results}

\subsection{Performance in Asymmetric \texttt{CarFlag} Domains}
\label{app:extra_asym_results}

\cref{fig:app_asym_carflag} shows the evaluation success rates in asymmetric variants of \texttt{CarFlag} domains with different offsets.
\begin{figure}[htbp]
 \centering
  \begin{subfigure}[t]{0.4\linewidth}
    \includegraphics[width=\linewidth]{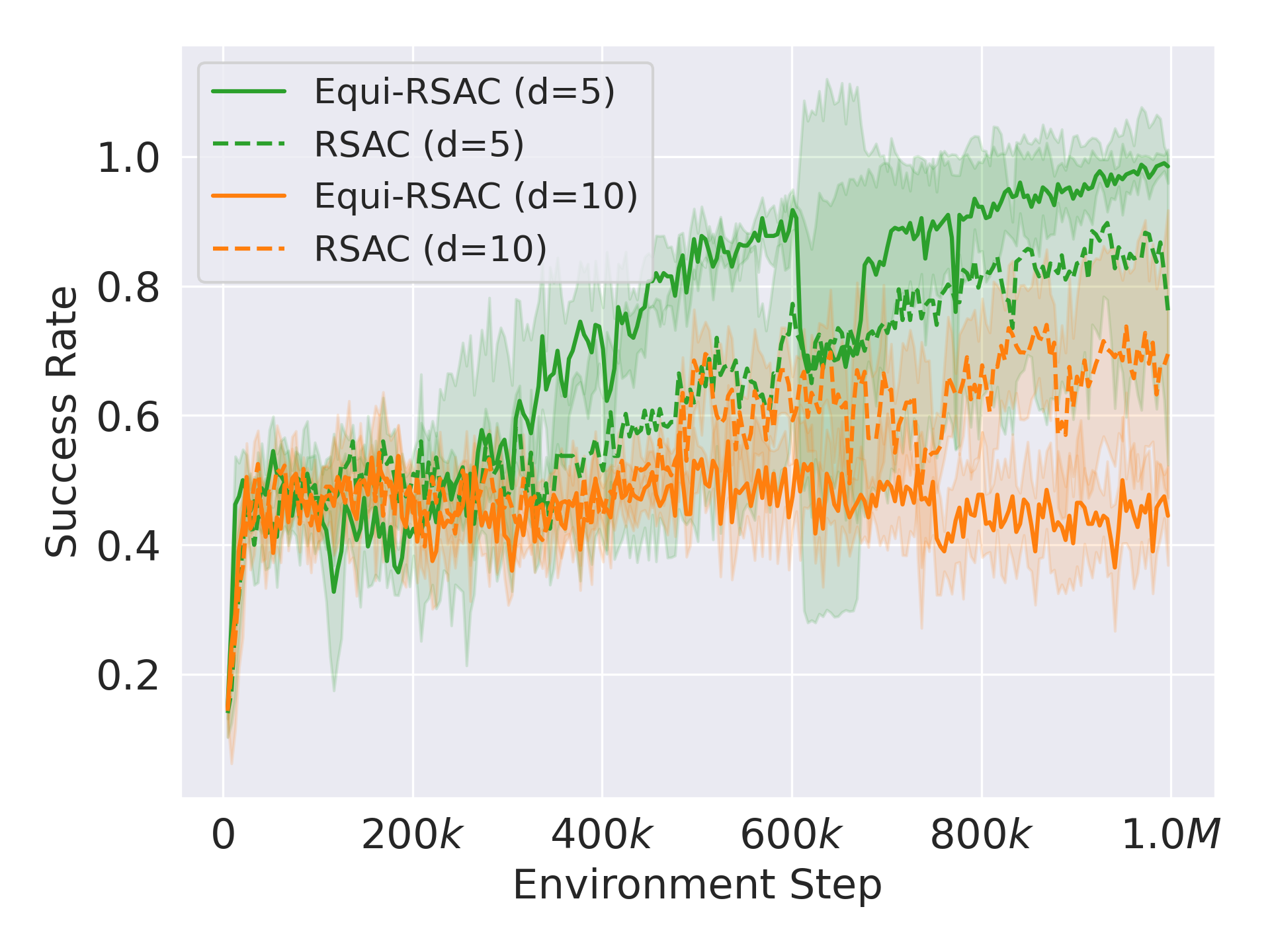}
    \caption{\texttt{Asym-CarFlag-1D} w/ positive offsets}
  \end{subfigure}
  \begin{subfigure}[t]{0.4\linewidth}
    \includegraphics[width=\linewidth]{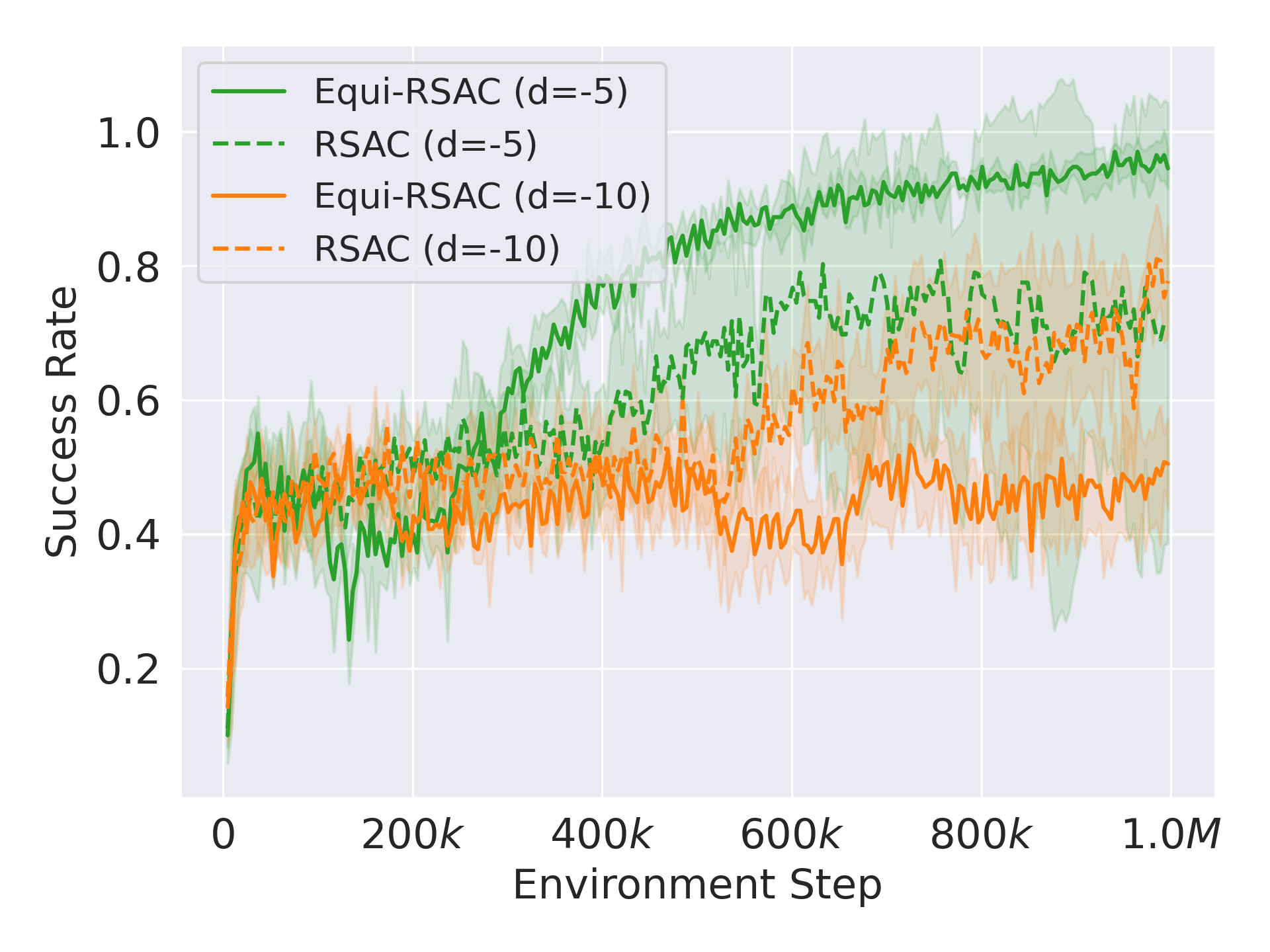}
    \caption{\texttt{Asym-CarFlag-1D} w/ negative offsets}
  \end{subfigure}
  \begin{subfigure}[t]{0.4\linewidth}
    \includegraphics[width=\linewidth]{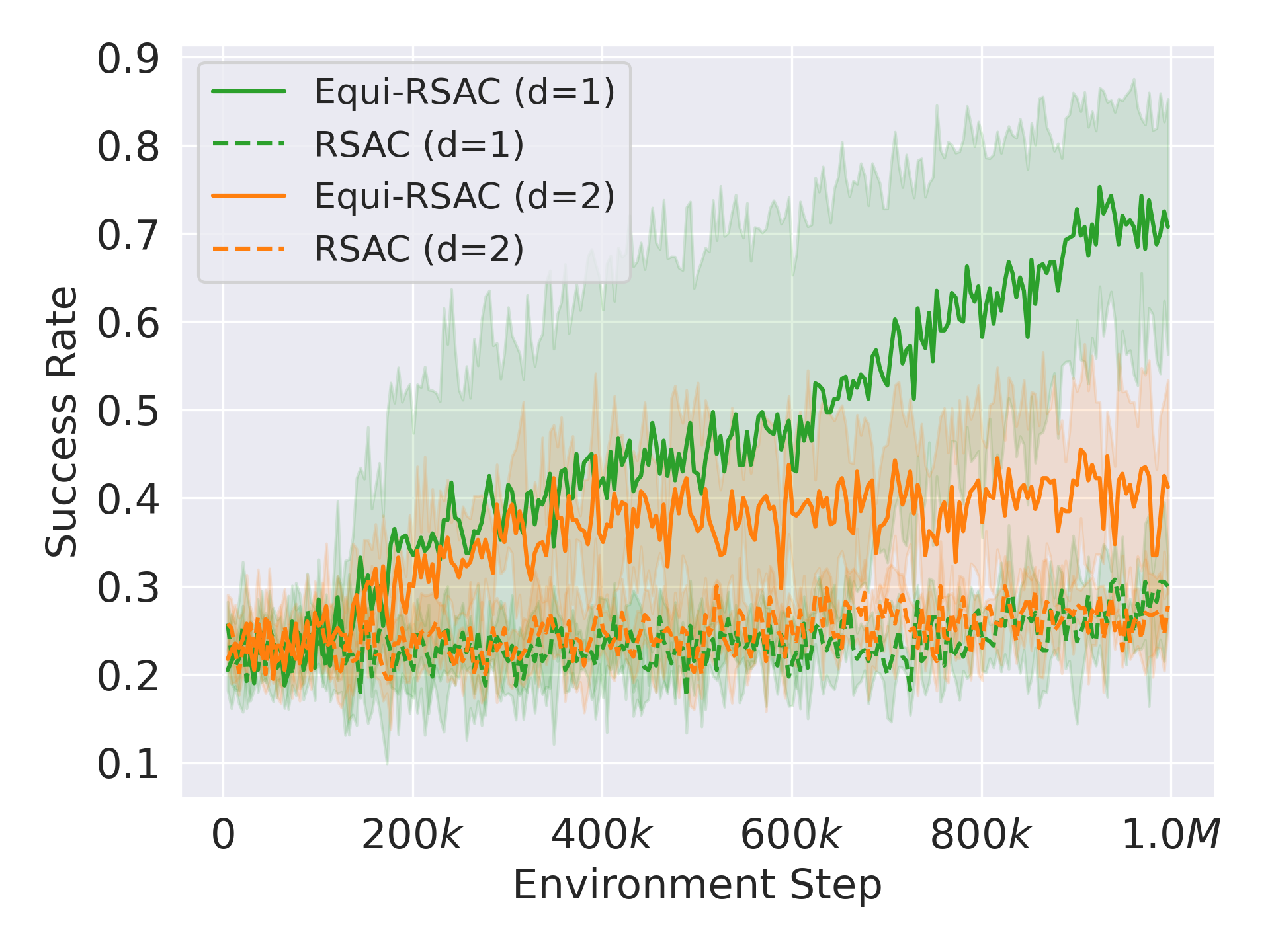}
    \caption{\texttt{Asym-CarFlag-2D} w/ positive offsets}
  \end{subfigure}
  \begin{subfigure}[t]{0.4\linewidth}
    \includegraphics[width=\linewidth]{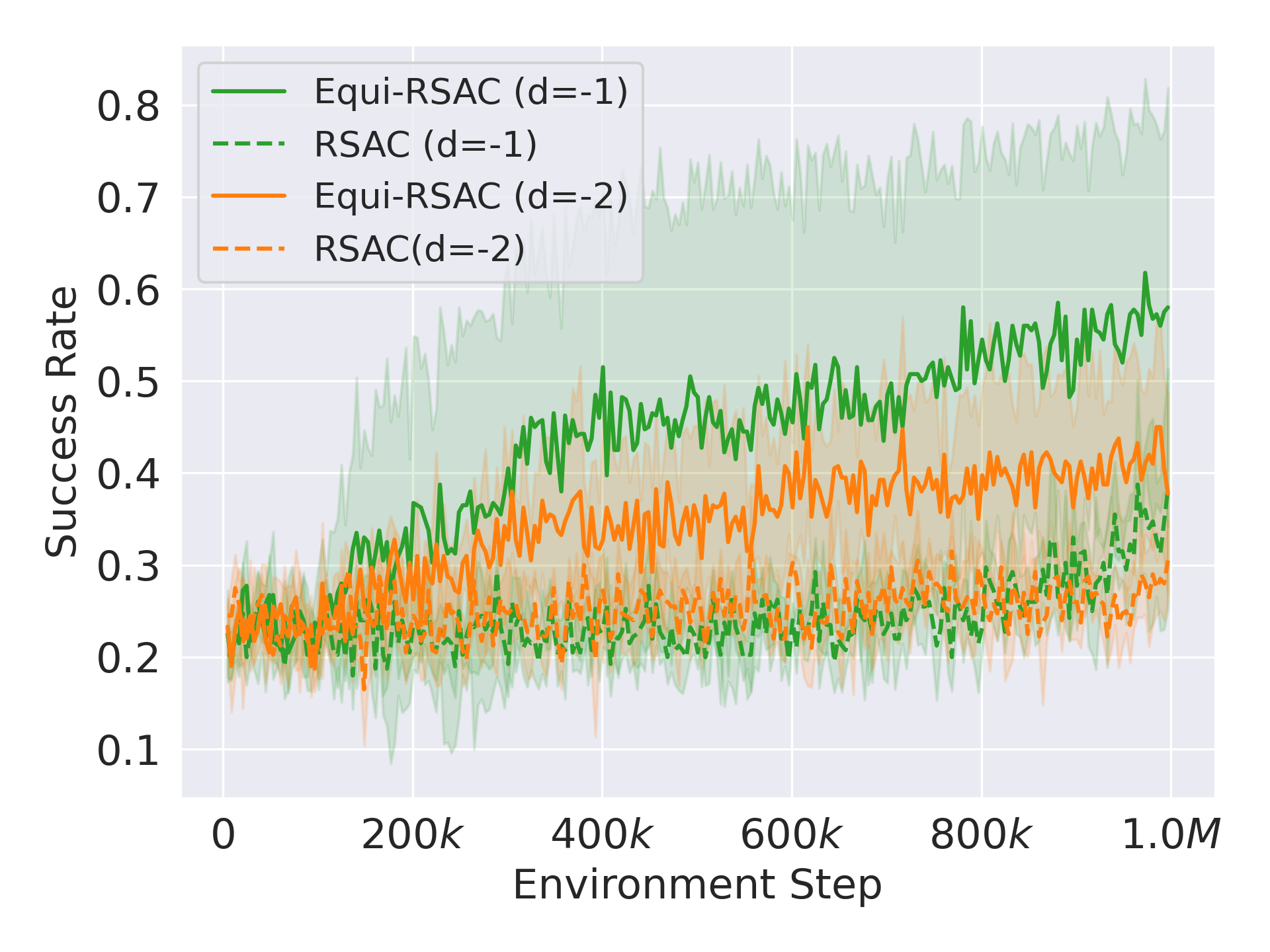}
    \caption{\texttt{Asym-CarFlag-2D} w/ negative offsets}
  \end{subfigure}
  \caption{Learning performance with asymmetric version of \texttt{CarFlag} domains.}
  \label{fig:app_asym_carflag}
\end{figure}

\subsection{Performance in Variants of \texttt{CarFlag} Domains}
\cref{fig:app_extra_carflag} show the evaluation success rates in different variants of \texttt{CarFlag} domains with a different world size and grid size. Our equivariant agent still outperforms other baselines.
\begin{figure}[htbp]
 \centering
  \begin{subfigure}[t]{0.4\linewidth}
    \includegraphics[width=\linewidth]{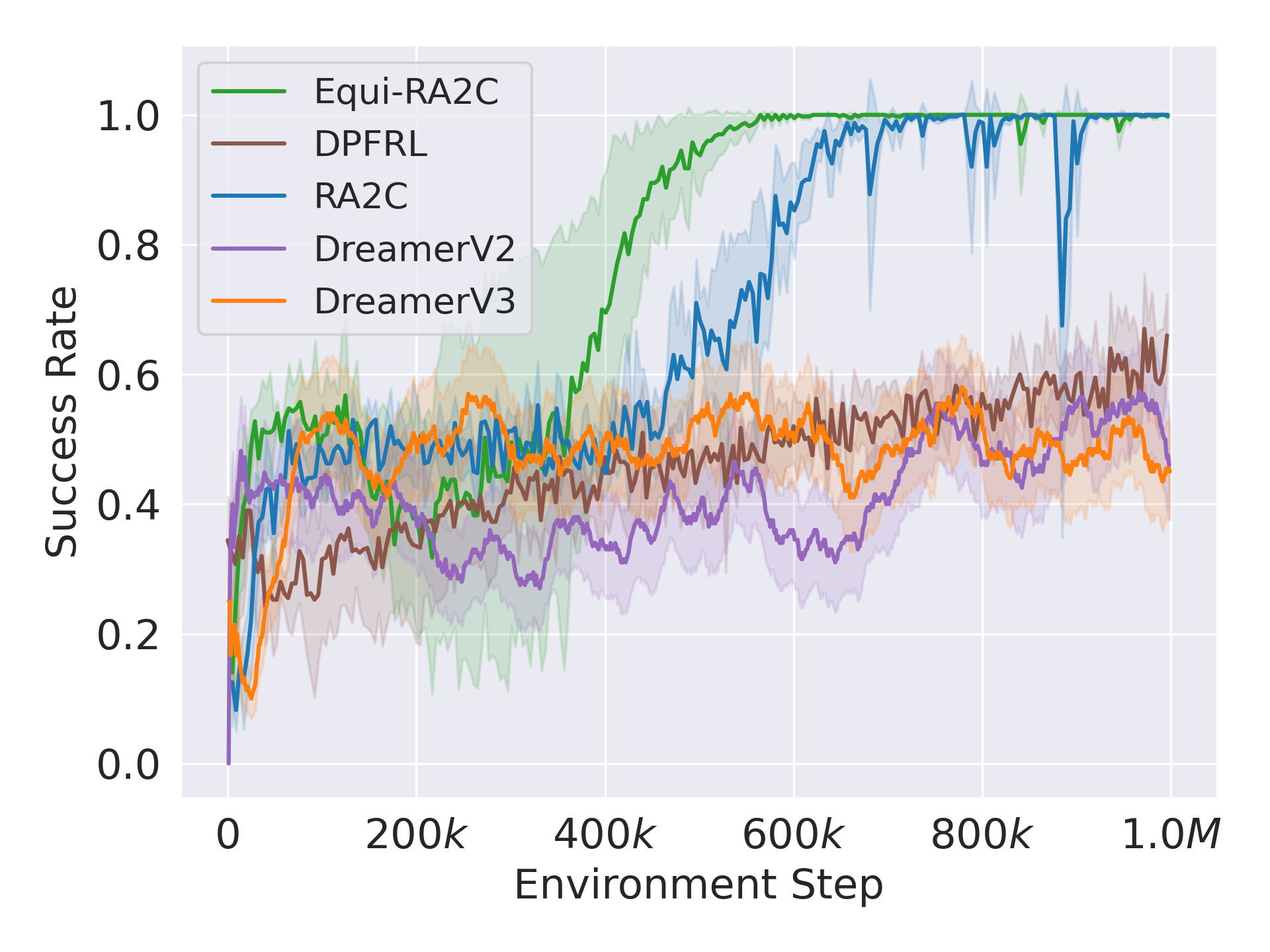}
    \caption{\texttt{CarFlag-1D} with world size 40}
  \end{subfigure}
  \begin{subfigure}[t]{0.4\linewidth}
    \includegraphics[width=\linewidth]{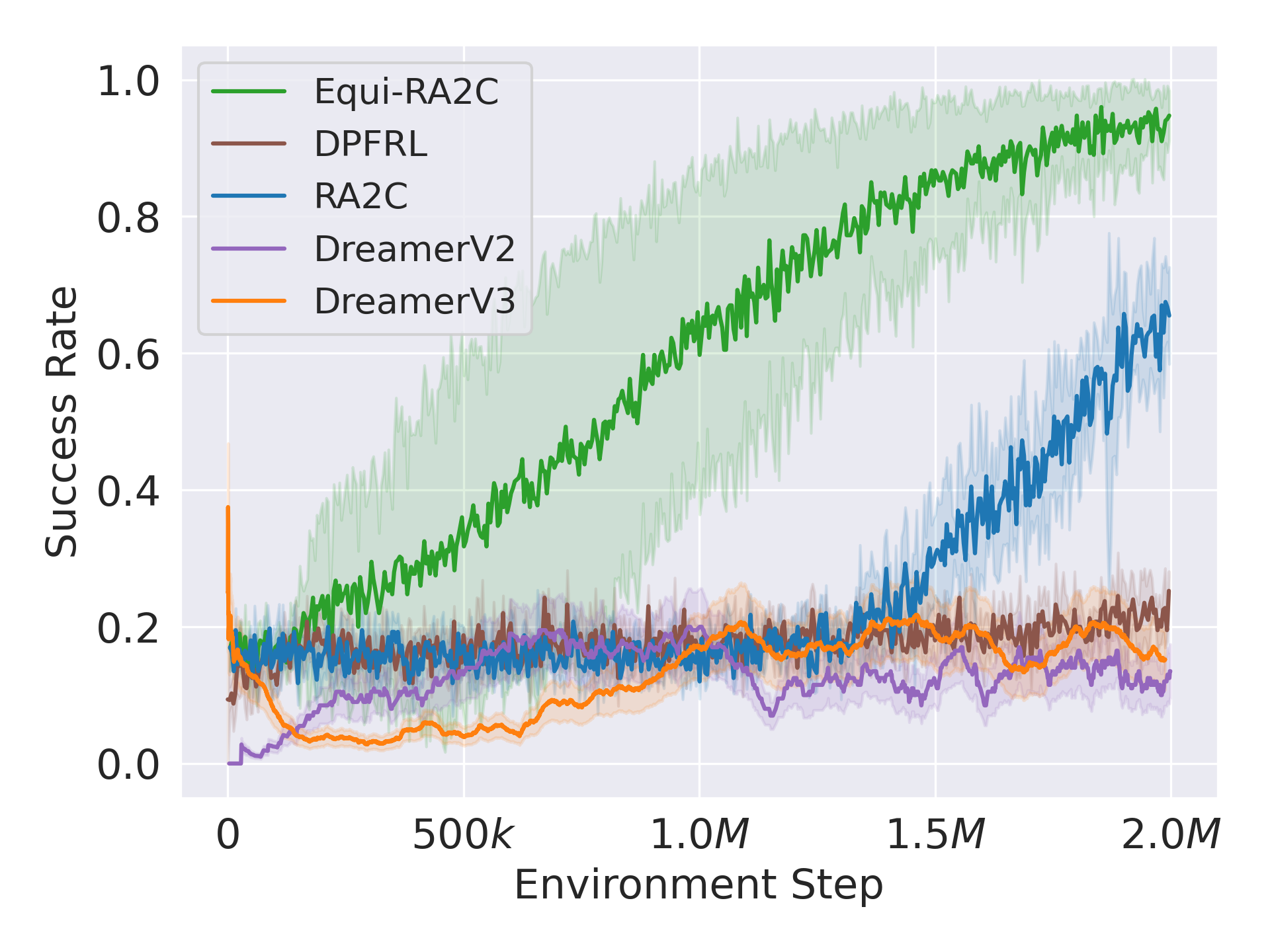}
    \caption{\texttt{CarFlag-2D} with grid size 11}
  \end{subfigure}
  \caption{Learning performance in \texttt{CarFlag} domains with different sizes.}
  \label{fig:app_extra_carflag}
\end{figure}

\subsection{Effect of Rotational Augmentation}
\cref{fig:app_no_rot} shows that including rotational augmented episodes significantly improves the learning performance of equivariant agents. These rotational augmented episodes possibly help equivariant agents distinguish different discrete rotations within a group, thus boosting performance.
\begin{figure}[htbp]
 \centering
  \begin{subfigure}[t]{0.4\linewidth}
    \includegraphics[width=\linewidth]{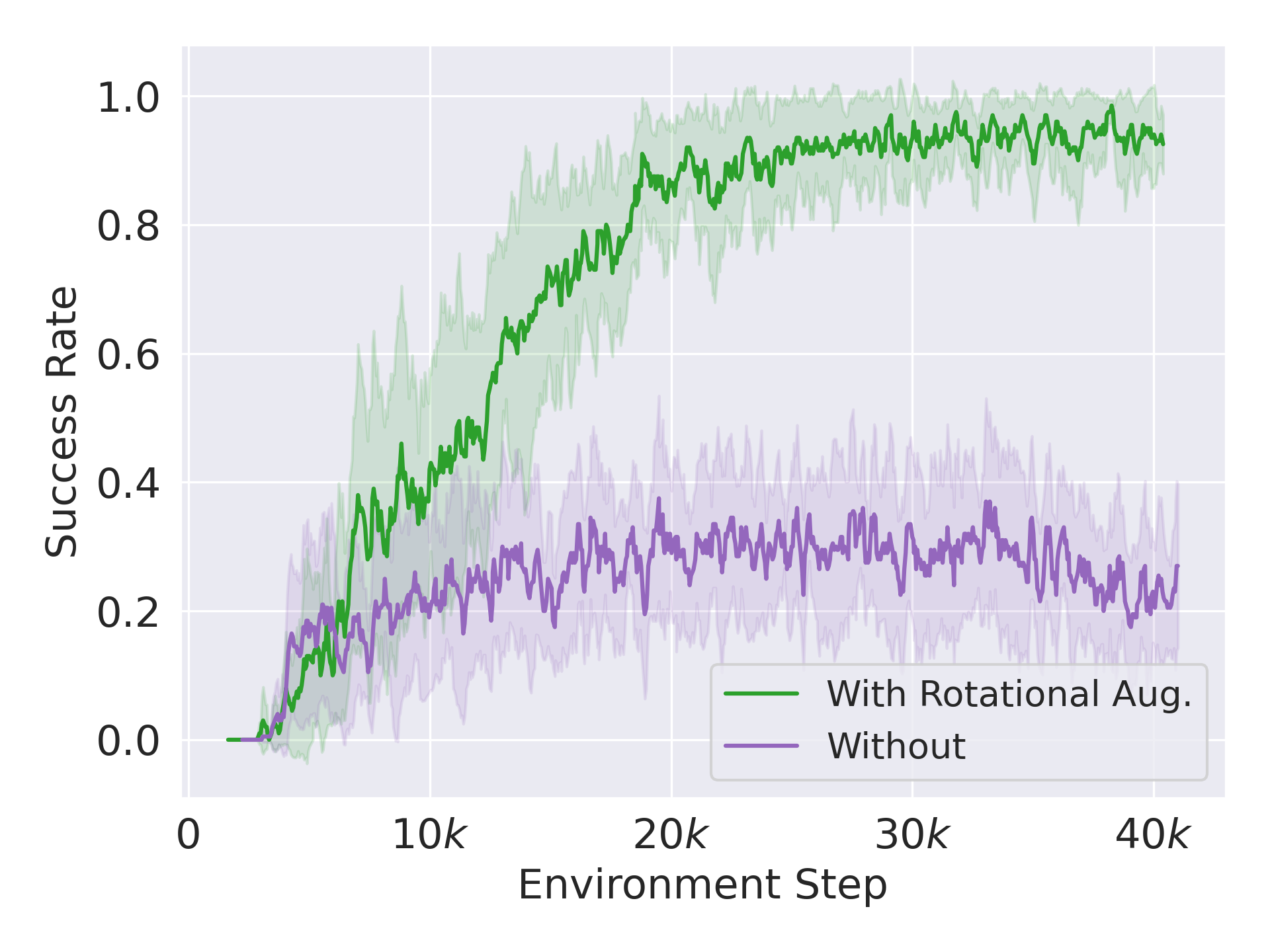}
    \caption{\texttt{Block-Picking}}
  \end{subfigure}
  \begin{subfigure}[t]{0.4\linewidth}
    \includegraphics[width=\linewidth]{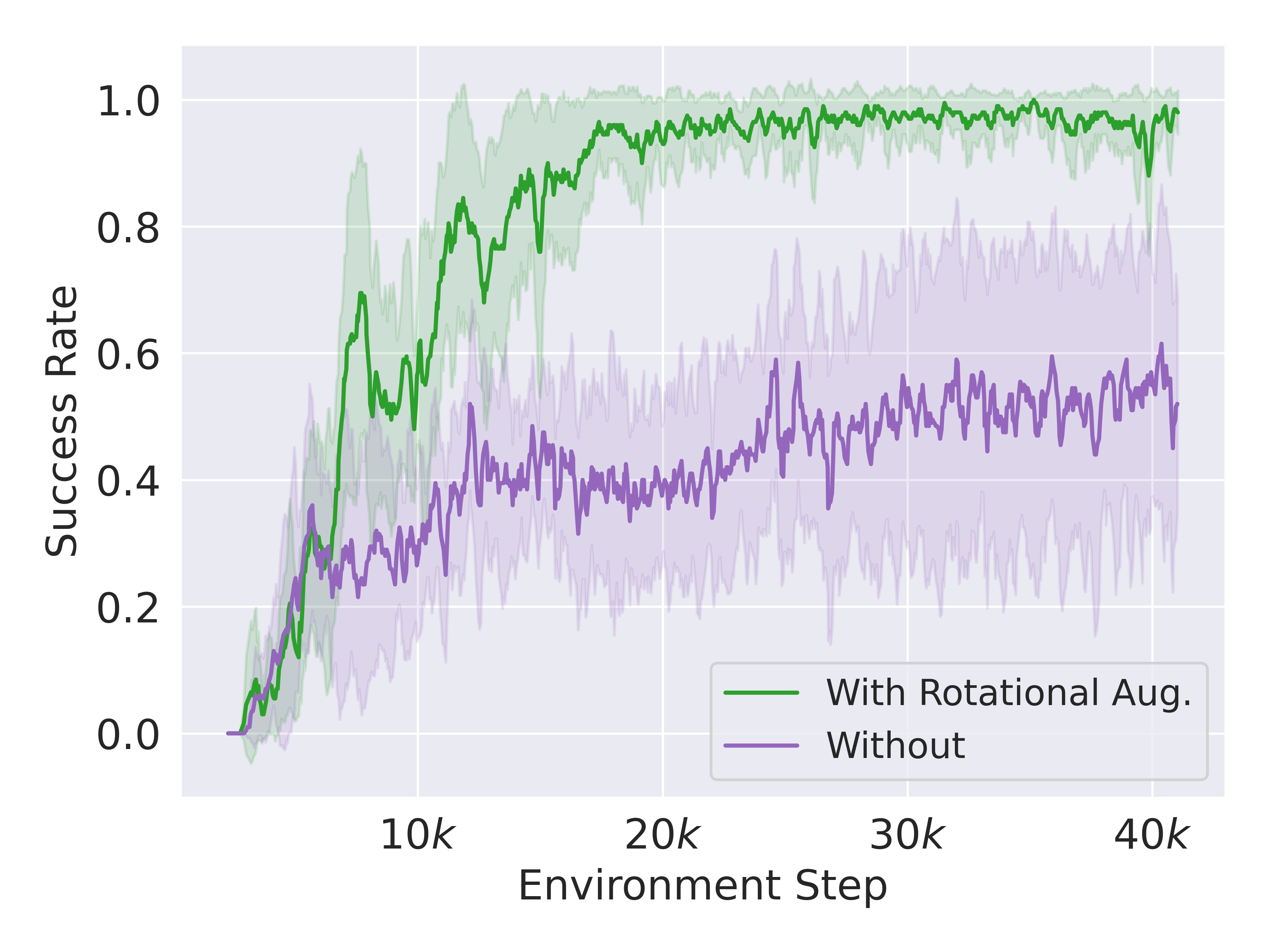}
    \caption{\texttt{Block-Pulling}}
  \end{subfigure}
  \begin{subfigure}[t]{0.4\linewidth}
    \includegraphics[width=\linewidth]{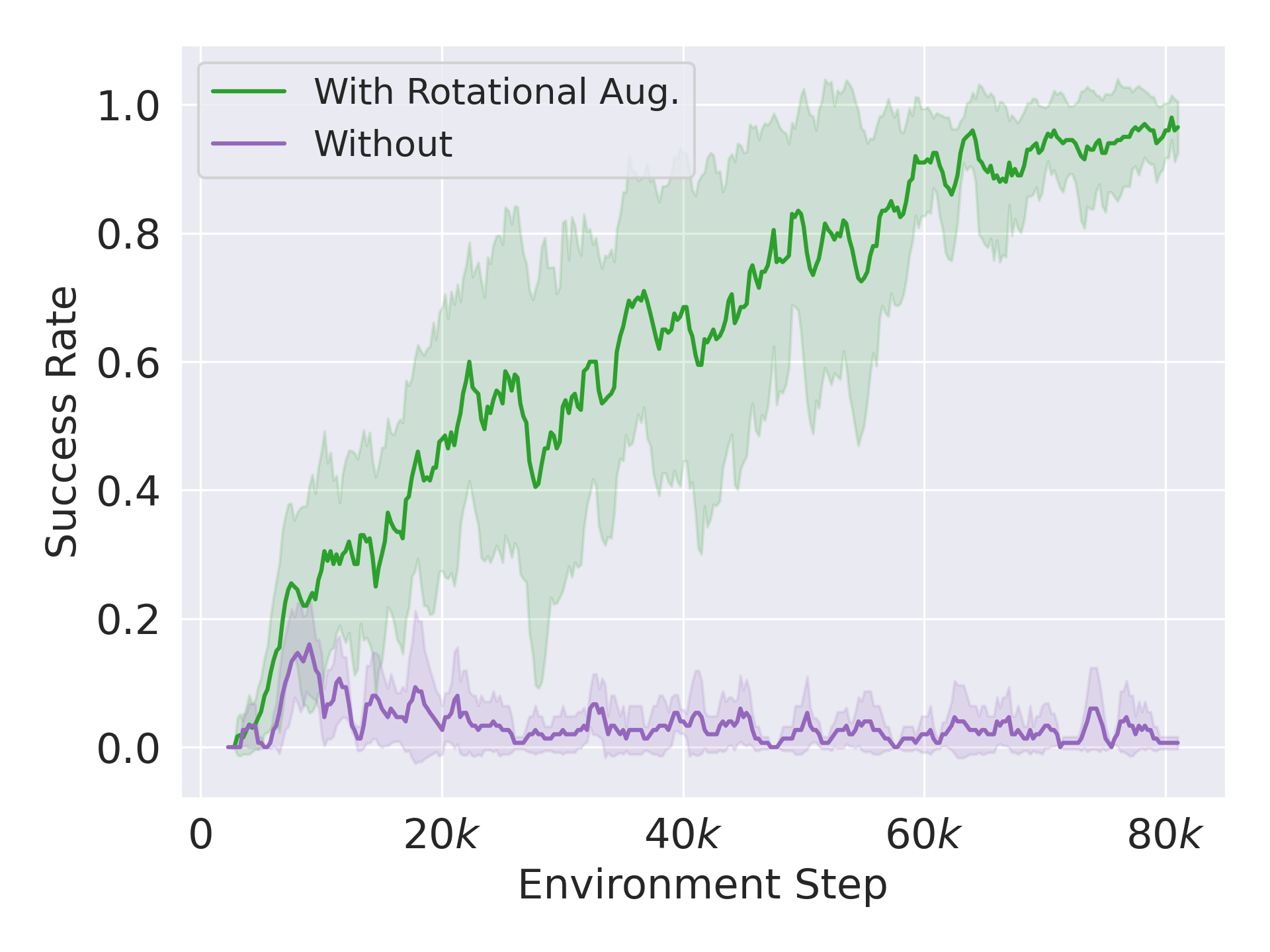}
    \caption{\texttt{Block-Pushing}}
  \end{subfigure}
  \begin{subfigure}[t]{0.4\linewidth}
    \includegraphics[width=\linewidth]{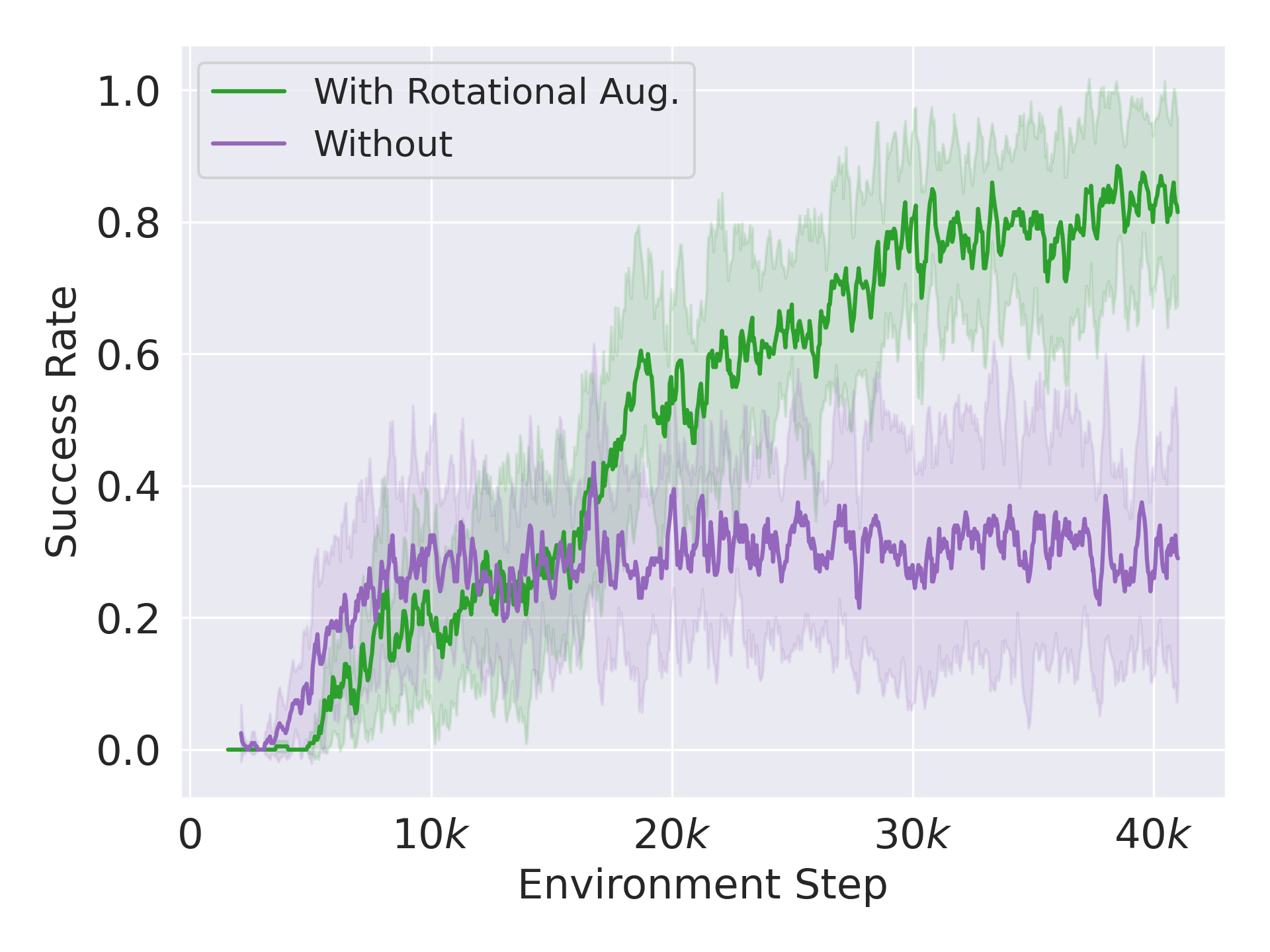}
    \caption{\texttt{Drawer-Opening}}
  \end{subfigure}
  \caption{Comparing the performance of our equivariant agents when using/not using rotational augmentation episodes.}
  \label{fig:app_no_rot}
\end{figure}

\clearpage
\subsection{Effect of Number of Demonstration Episodes}
\cref{fig:app_abl_demonstrations} shows that the performance improves when using more demonstrations in all domains, as expected.
\begin{figure}[htbp]
 \centering
  \begin{subfigure}[t]{0.4\linewidth}
    \includegraphics[width=\linewidth]{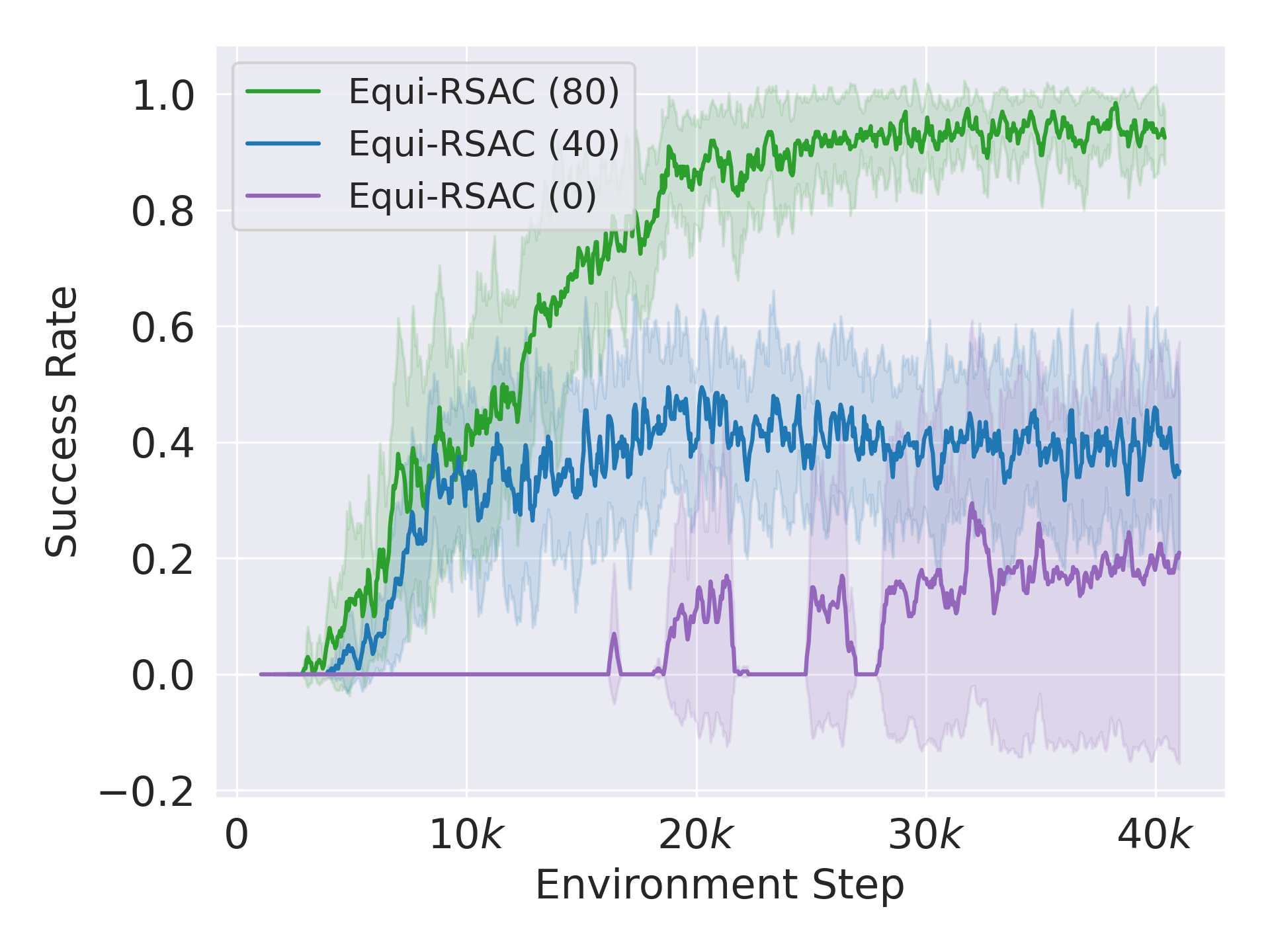}
    \caption{\texttt{Block-Picking}}
  \end{subfigure}
  \begin{subfigure}[t]{0.4\linewidth}
    \includegraphics[width=\linewidth]{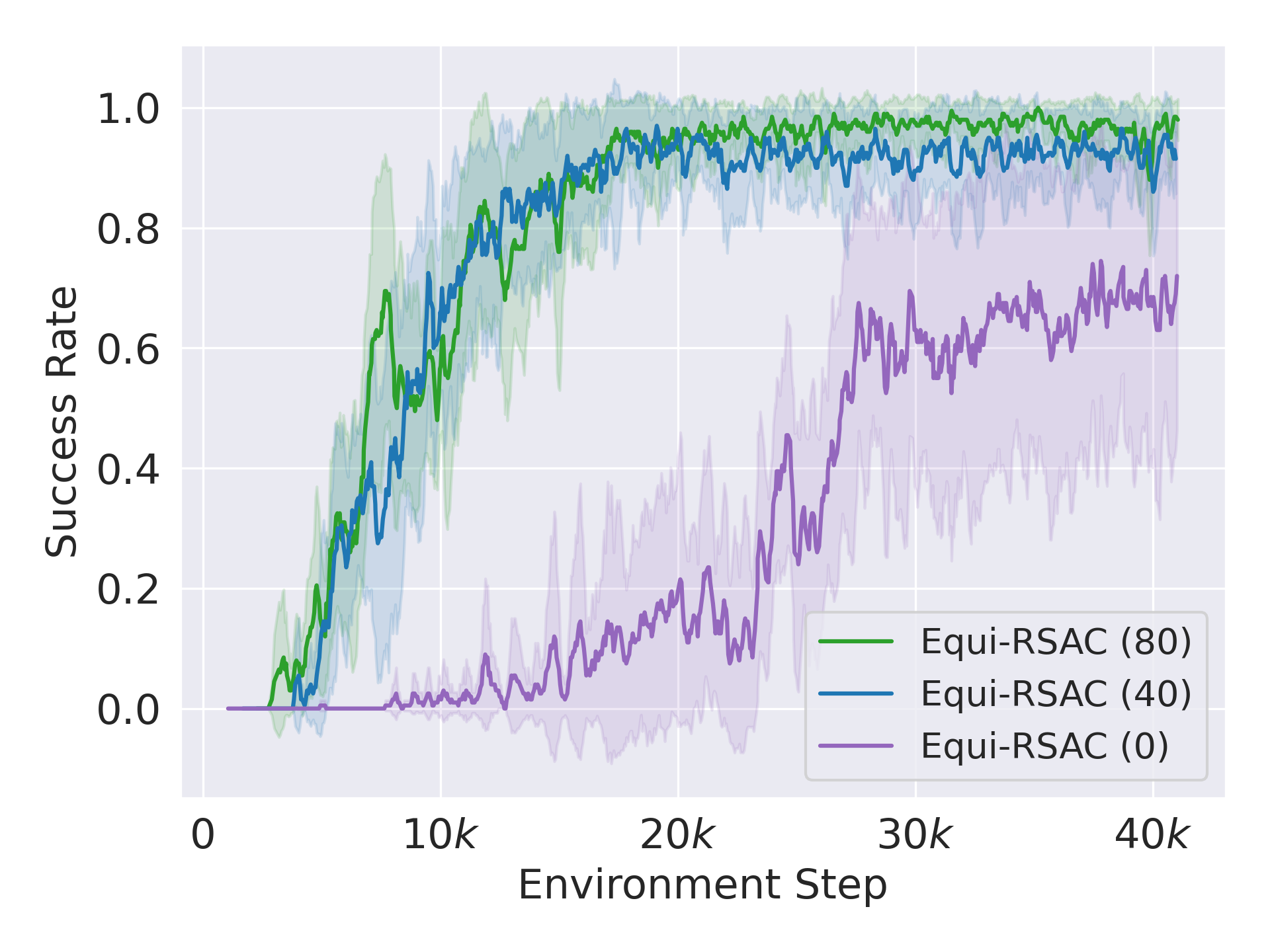}
    \caption{\texttt{Block-Pulling}}
  \end{subfigure}
  \begin{subfigure}[t]{0.4\linewidth}
    \includegraphics[width=\linewidth]{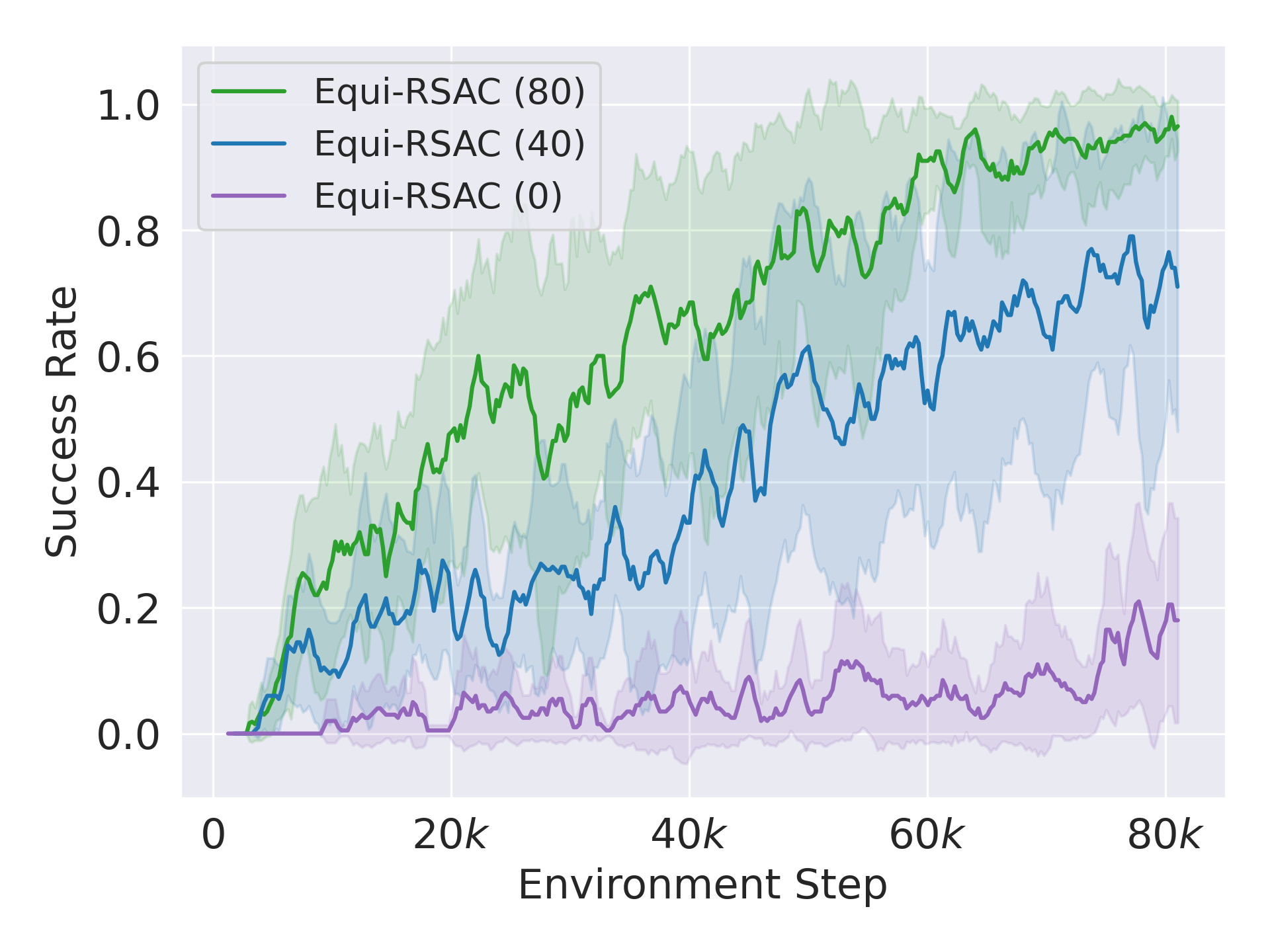}
    \caption{\texttt{Block-Pushing}}
  \end{subfigure}
  \begin{subfigure}[t]{0.4\linewidth}
    \includegraphics[width=\linewidth]{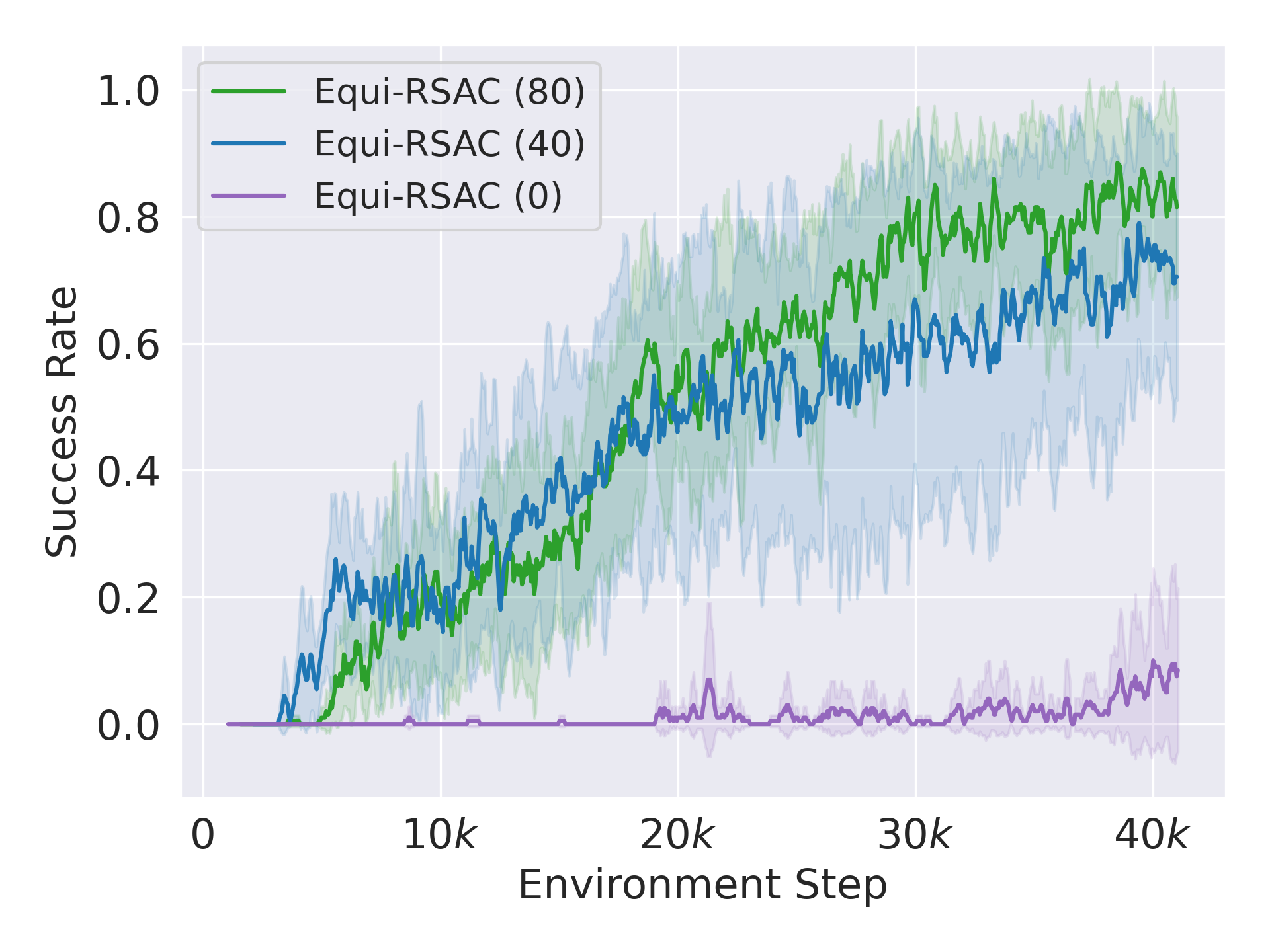}
    \caption{\texttt{Drawer-Opening}}
  \end{subfigure}
  \caption{Using different numbers of demonstration episodes.}
  \label{fig:app_abl_demonstrations}
\end{figure}

\end{document}